\documentclass{article}
\usepackage{iclr2024_conference,times}
\usepackage{booktabs}
\usepackage{graphicx} %
\usepackage{hyperref}       %
\usepackage{url}
\hypersetup{
  colorlinks,
  linkcolor={red!50!black},
  citecolor={blue!50!black},
  urlcolor={blue!80!black}
  }
  \definecolor{orange}{HTML}{ff7f0e}
  \definecolor{blue}{HTML}{1f77b4}
  
\usepackage{amsmath}
\usepackage{amsthm}
\usepackage{amsfonts}
\usepackage{amssymb}
\usepackage{xcolor}
\usepackage{nicefrac}
\usepackage{dsfont}
\usepackage{algorithm}
\usepackage{algpseudocode}
\usepackage{caption}
\usepackage{subcaption}
\usepackage{multirow}
\usepackage{bm}
\usepackage{wrapfig}
\usepackage{tabularx}
\usepackage{mathtools}

\let\textcite\relax
\newcommand\textcite\citet
\let\parencite\relax
\newcommand\parencite\citep
\let\cite\relax
\newcommand\cite\citep

\usepackage[capitalise]{cleveref}
\crefname{equation}{}{}
\numberwithin{equation}{section} %

\newcommand\labelthis{\addtocounter{equation}{1}\tag{\theequation}}

\newtheorem{theorem}{Theorem}[section]
\newtheorem{corollary}[theorem]{Corollary}
\newtheorem{lemma}[theorem]{Lemma}
\newtheorem{proposition}[theorem]{Proposition}

\newtheorem{algorithm2}{Algorithm}[section]
\newtheorem{problem}{Problem}

\theoremstyle{definition}
\newtheorem{definition}{Definition}[section]

\theoremstyle{remark}
\newtheorem*{remark}{Remark}

\DeclareMathOperator{\prox}{prox}
\DeclareMathOperator{\sign}{sign}
\DeclareMathOperator*{\argmax}{argmax}
\DeclareMathOperator*{\argmin}{argmin}
\newcommand{\R}{\mathbb{R}}
\newcommand{\E}{\mathbb{E}}
\newcommand{\s}{\mathbf{s}}
\newcommand{\x}{\mathbf{x}}
\newcommand{\U}{\mathbf{U}}
\newcommand{\y}{\mathbf{y}}
\newcommand{\z}{\mathbf{z}}
\newcommand{\A}{\mathbf{A}}
\newcommand{\MAP}{\mathrm{MAP}}

\newcommand{\Y}{\mathcal{Y}}
\newcommand{\N}{\mathcal{N}}
\newcommand{\W}{\mathbf{W}}
\newcommand{\HHH}{\mathbf{H}}
\newcommand{\w}{\mathbf{w}}
\newcommand{\D}{\mathbf{D}}

\newcommand{\uu}{\mathbf{u}}
\newcommand{\vv}{\mathbf{v}}
\newcommand*{\conv}{\mathbin{\ast}}
\newcommand{\beps}{\bm{\varepsilon}} %
\newcommand{\bb}{\mathbf{b}}
\newcommand{\bc}{\mathbf{c}}
\newcommand{\cc}{\mathbf{c}}
\newcommand{\adj}{^T}
\newcommand{\Lip}{\mathrm{Lip}}
\newcommand{\diag}{\mathop{\mathrm{diag}}}
\newcommand{\frob}{\mathrm{F}}
\newcommand{\cF}{\mathcal{F}}
\newcommand{\Id}{\mathop{\mathrm{Id}}}
\newcommand{\ualpha}{\underline{\alpha}}
\newcommand{\oalpha}{\overline{\alpha}}

\DeclareMathOperator{\crit}{crit}

\DeclarePairedDelimiterX\abs[1]\lvert\rvert{
  \ifblank{#1}{\:\cdot\:}{#1}
}
\DeclarePairedDelimiterX\norm[1]\lVert\rVert{
  \ifblank{#1}{\:\cdot\:}{#1}
}

\definecolor{forestgreenweb}{rgb}{0.13, 0.55, 0.13}

\newcommand{\rebuttal}[1]{{\color{black} #1}}

\title{What's in a Prior? \\Learned Proximal Networks for Inverse \\ Problems}
\date{}

\author{Zhenghan Fang\thanks{Equal contribution.} \\
Mathematical Institute for Data Science \\ Johns Hopkins University  \\
\texttt{zfang23@jhu.edu} \\
\And
Sam Buchanan\footnotemark[1] \\
Toyota Technological Institute at Chicago  \\
\texttt{sam@ttic.edu} \\
\And
Jeremias Sulam \\
Mathematical Institute for Data Science \\ Johns Hopkins University  \\
\texttt{jsulam1@jhu.edu} \\
}

\iclrfinalcopy
\begin{document}

\maketitle

\begin{abstract}
Proximal operators are ubiquitous in inverse problems, commonly appearing as part of algorithmic strategies to regularize problems that are otherwise ill-posed. Modern deep learning models have been brought to bear for these tasks too, as in the framework of plug-and-play or deep unrolling, where they loosely resemble proximal operators. 
Yet, something essential is lost in employing these purely data-driven approaches: there is no guarantee that a general deep network represents the proximal operator of any function, nor is there any characterization of the function for which the network might provide some approximate proximal. This not only makes guaranteeing convergence of iterative schemes challenging but, more fundamentally, complicates the analysis of what has been learned by these networks about their training data. 
Herein we provide a framework to develop \textit{learned proximal networks} (LPN), prove that they provide exact proximal operators for a data-driven nonconvex regularizer, and show how a new training strategy, dubbed \textit{proximal matching}, provably promotes the recovery of the log-prior of the true data distribution. Such LPN provide general, unsupervised, expressive proximal operators that can be used for general inverse problems with convergence guarantees. We illustrate our results in a series of cases of increasing complexity, demonstrating that these models not only result in state-of-the-art performance, but provide a window into the resulting priors learned from data.
\end{abstract}

\section{Introduction}
\vspace{-5pt}
Inverse problems concern the task of estimating underlying variables that have undergone a degradation process, such as in denoising, deblurring, inpainting, or compressed sensing \cite{bertero2021introduction,ongie2020deep}. Since these problems are naturally ill-posed, solutions to any of these problems involve, either implicitly or explicitly, the utilization of \textit{priors}, or models, about what type of solutions are preferable \cite{engl1996regularization,benning2018modern,arridge2019solving}.
Traditional methods model this prior directly, by constructing regularization functions that promote specific properties in the estimate, such as for it to be smooth \cite{tikhonov1977solutions}, piece-wise smooth \cite{rudin1992nonlinear,bredies2010total}, or for it to have a sparse decomposition under a given basis or even a potentially overcomplete dictionary \cite{bruckstein2009sparse, sulam2014image}. 
On the other hand, from a machine learning perspective, the complete restoration mapping can also be modeled by a regression function and 
 by providing a large collection of input-output (or clean-corrupted) pairs of samples \cite{mccann2017convolutional,ongie2020deep,zhu2018image}.

An interesting third alternative has combined these two approaches by making the
insightful observation that \rebuttal{many iterative solvers for inverse
problems incorporate the application of the proximal operator for the
regularizer}. Such a proximal step can be loosely interpreted as a denoising
step and, as a result, off-the-shelf strong-performing denoising
algorithms (as those given by modern deep learning methods) can be employed as
a subroutine. \rebuttal{The Plug-and-Play (PnP) framework is a notable example
where proximal operators are replaced with such denoisers
\cite{venkatakrishnan2013plug,zhang2017learning,meinhardt2017learning,zhang2021plug,kamilov2023plug,tachella2019real},
but these can be applied more broadly to solve inverse problems, as
well \cite{romano2017little,romano2015boosting}.} 
While this strategy works
very well in practice, little is known about the approximation properties of
these methods. For instance, \emph{do these denoising networks actually (i.e.,
provably) provide a proximal operator for some regularization function?}
Moreover, and from a variational perspective, \emph{would this regularization
function recover the \textit{correct} regularizer, such as the (log) prior of
the data distribution?} 
Partial answers to some of these questions exist, but how to address all of them in a single framework 
remains unclear
\cite{hurault2022proximal,lunz2018adversarial,cohen2021has,zou2023deep,Goujon2023-wa}
(see a thorough discussion of related works in \cref{sec:related-works}). More broadly, the ability to characterize a data-driven (potentially nonconvex) regularizer that enables good restoration is paramount in applications that demand notions of robustness and interpretability, and this remains an open challenge.

In this work, we address these questions by proposing a new class of deep neural networks, termed \emph{learned proximal networks} (LPN), that \textit{exactly implement the proximal operator} of a general learned function. Such a LPN implicitly learns a regularization function that can be characterized and evaluated, shedding light onto what has been learned from data. In turn, we present a new training problem, which we dub \textit{proximal matching}, 
that provably promotes the recovery of the {correct} regularization term (i.e., the log of the data distribution), which need not be convex.
Moreover, the ability of LPNs to implement exact proximal operators allows for guaranteed convergence to critical points of the variational problem, which we derive for PnP reconstruction algorithms under no additional assumptions on the trained LPN.
We demonstrate through experiments on that our LPNs can recover the correct underlying data distribution, and further show that LPNs lead to state-of-the-art reconstruction performance on image deblurring, CT reconstruction and compressed sensing, while enabling precise characterization of the data-dependent prior learned
by the model. 
Code for reproducing all experiments is made publicly available at \url{https://github.com/Sulam-Group/learned-proximal-networks}.

\vspace{-5pt}
\section{Background}
\vspace{-5pt}
Consider an unknown signal in an Euclidean space\footnote{The analyses in this paper can be generalized directly to more general Hilbert spaces.}, $\x \in \R^{n}$%
, and a known measurement operator, $A: \R^{n} \rightarrow \R^{m}$. The goal of inverse problems is to recover $\x$ from measurements $\y = A(\x) + \vv \in \R^{m}$, where $\vv$ is a noise or nuisance term. This problem is typically ill-posed: infinitely many solutions $\x$ may explain (i.e. approximate) the measurement $\y$ \cite{benning2018modern}. Hence, a prior is needed to regularize the problem, which can generally take the form
\begin{equation} \label{eq:main_inverse_problem}
    \min_\x \frac{1}{2}\|\y - A(\x)\|_2^2 + R(\x),
\end{equation}
for a function $R(\x):\R^n\to\R$ promoting a solution that is likely under the prior distribution of $\x$. We will make no assumptions on the convexity of $R(\x)$ in this work.%

\paragraph{Proximal operators}
Originally proposed by \textcite{moreau1965proximite} as a generalization of projection operators, proximal operators are  central in optimizing the problem \eqref{eq:main_inverse_problem} by means of proximal 
gradient descent (PGD) \cite{beck2017first}, alternating direction method of multipliers (ADMM) \cite{boyd2011distributed}, or
primal dual hybrid gradient (PDHG) \cite{chambolle2011first}. For a given functional $R$ as above, its proximal operator $\prox_{R}$ is defined by
\begin{equation} \label{eq:prox_def}
    \prox_{R}(\y) := \argmin_{\x} \frac{1}{2} \|\y - \x\|^2 + R(\x).
\end{equation}
When $R$ is non-convex, the solution to this problem may not be unique and the proximal mapping is set-valued. Following \cite{gribonval2020characterization}, we define the proximal operator of a function $R$ as a \emph{selection} of the set-valued mapping: $f(\y)$ is a proximal operator of $R$ if and only if $f(\y) \in \argmin_{\x} \frac{1}{2} \|\y - \x\|^2 + R(\x) $ for each $\y \in \R^{n}$.
A key result in \cite{gribonval2020characterization} is that the {continuous} proximal  of a (potentially nonconvex) function can be fully characterized as the gradient of a convex function, as the following result formalizes.

\begin{proposition}\label{prop:prox_equivalence}[Characterization of continuous proximal operators, \cite[Corollary 1]{gribonval2020characterization}]
\label{prop:characterization-continuous-prox}
    Let $\Y \subset \R^n$ be non-empty and open and $f : \Y \rightarrow \R^n$ be a continuous function. Then, $f$ is a proximal operator of a function $R :\R^n \rightarrow \R \cup \{+\infty\}$ if and only if there exists a convex differentiable function $\psi$ such that $f(\y) = \nabla \psi(\y)$ for each $\y \in \Y$.
\end{proposition}

\begin{wrapfigure}{r}{0.18\textwidth}
    \centering 
    \vspace{-20pt}
    \includegraphics[width=0.15\textwidth]{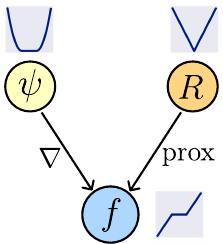}
    \caption{Sketch of Prop. \ref{prop:characterization-continuous-prox} for $R(\cdot) = \|\cdot\|_1$.}
    \vspace{-2em}
    \label{fig:method}
\end{wrapfigure}

It is worth stressing the differences between $R$ and $\psi$. While $f$ is the proximal operator of $R$, i.e. $\prox_R = f$, $f$ is also the gradient of a convex $\psi$, $\nabla \psi = f$ (see \Cref{fig:method}). Furthermore, $R$ may be non-convex, while $\psi$ must be convex. As can be expected, there exists a precise relation between $R$ and $\psi$, and we will elaborate further on this connection shortly.
The characterization of proximals of convex functions is similar but additionally requiring $f$ to be non-expansive \cite{moreau1965proximite}. Hence, by relaxing the nonexpansivity, we obtain a broader class of proximal operators. As we will show later, the ability to model proximal operators of non-convex functions will prove very useful in practice, as the log-priors\footnote{\rebuttal{In this paper, the ``log-prior'' of a data distribution $p_\x$ means its negative log-likelihood, $-\log p_\x$.}} of most real-world data are indeed non-convex.

\textbf{Plug-and-Play}
\label{sec:background-pnp}
This paper closely relates to the Plug-and-Play (PnP) framework \cite{venkatakrishnan2013plug}.
\rebuttal{PnP employs off-the-shelf denoising algorithms to solve general inverse problems within an iterative optimization solver, such as PGD \cite{beck2017first,hurault2022proximal}, ADMM \cite{boyd2011distributed,venkatakrishnan2013plug}, half quadratic splitting (HQS) \cite{geman1995nonlinear,zhang2021plug}, primal-dual hybrid gradient (PDHG) \cite{chambolle2011first}, and Douglas-Rachford splitting (DRS) \cite{douglas1956numerical,lions1979splitting,combettes2007douglas,hurault2022proximal}}. %
Inspired by the observation that $\prox_R(\y)$ resembles the \textit{maximum a posteriori} (MAP) denoiser at $\y$ with a log-prior $R$, PnP replaces the explicit solution of this step with generic denoising algorithms, such as BM3D \cite{dabov2007image,venkatakrishnan2013plug} or CNN-based denoisers \cite{meinhardt2017learning,zhang2017learning,zhang2021plug,kamilov2023plug}, bringing the benefits of advanced denoisers to general inverse problems. %
While useful in practice, such denoisers are \emph{not} in general proximal
operators. Indeed, modern denoisers need not be MAP estimators at all, but instead typically approximate a minimum mean squared error (MMSE) solution. Although deep learning denoisers have achieved impressive results when used with PnP, little is known
about the implicit prior---if any---encoded in these denoisers, thus diminishing the
interpretability of the reconstruction results. Some convergence
guarantees have been derived for PnP with MMSE denoisers
\cite{Xu2020-my}, chiefly relying on the assumption that the denoiser is non-expansive (which can be hard to verify or enforce in practice). Furthermore, when interpreted as proximal operators, the prior in MMSE denoisers can be drastically different from the original (true data) prior \textcite{Gribonval2011-pf}, raising concerns about correctness. 
There is a broad family of works that relate to the ideas in this work, and we expand on them in \Cref{sec:related-works}.

\vspace{-5pt}
\section{Learned Proximal Networks}
\label{sec:methodology}
\vspace{-5pt}

First, we seek a way to parameterize a neural network such that its mapping is the proximal operator of some (potentially nonconvex) scalar-valued functional.
Motivated by \Cref{prop:prox_equivalence}, we will seek network architectures that parameterize \emph{gradients of convex functions}. 
A simple way to achieve this is by differentiating a neural network that
implements a convex function: given a scalar-valued neural network,
$\psi_\theta: \R^n \rightarrow \R$, whose output is convex with respect to its
input, we can parameterize a LPN as $f_\theta = \nabla \psi_\theta$, which can be efficiently evaluated via back propagation. This makes LPN a gradient field---and a conservative vector field---of an explicit convex function. 
Fortunately, this is not an entirely new problem.
\textcite{Amos2017-ql} proposed input convex neural networks (ICNN) that guarantee to parameterize convex functions by constraining the network weights to be non-negative and the nonlinear activations convex and non-decreasing\footnote{Other ways to parameterize gradients of convex functions exist \cite{richter2021input}, but come with other constraints and limitations (see discussion in \Cref{sec:other-parameterization-of-lpn}).}. 
Consider a single-layer neural network characterized by the weights $\textbf{W} \in \R^{m\times n}$, bias $\bb\in \R^m$ and a scalar non-linearity $g:\R\to\R$. Such a network, at $\y$, is given by $\z = g(\mathbf W \y + \bb)$. With this notation, we now move to define our LPNs.

\begin{proposition}[Learned Proximal Networks]
\label{prop:lpn}
    Consider a scalar-valued $(K+1)$-layered neural network $\psi_\theta: \R^{n} \rightarrow \R$ defined by $\psi_\theta(\y) = \mathbf w^T \z_{K} + b$ and the recursion
    \begin{align*}
        \z_{1} = g( \mathbf H_1 \y + \bb_1), \qquad
        \z_{k} = g(\mathbf W_k  \z_{k-1} + \mathbf H_k \y + \bb_k), \,
        k \in [2, K] 
    \end{align*}
    where $\theta = \{\mathbf w, b, (\mathbf W_k)_{k=2}^K, (\mathbf H_k, \mathbf b_k)_{k=1}^K\}$ are learnable parameters, and $g$ is a convex, non-decreasing and $C^{2}$ scalar function, and $\mathbf W_k$ and $\mathbf w$ have non-negative entries.
    Let 
     $f_\theta$ be the gradient map of $\psi_\theta$ w.r.t. its input, i.e. $f_\theta = \nabla_{\y} \psi_\theta$. %
    Then, there exists a function $R_\theta: \R^{n} \rightarrow \R \cup
    \{+\infty\}$ such that  $f_\theta(\y) \in \prox_{R_\theta}(\y)$, $\forall~\y
    \in \R^{n}$. %
\end{proposition}

The simple proof of this result follows by combining properties of ICNN from \textcite{Amos2017-ql} and the characterization of proximal operators from \textcite{gribonval2020characterization} (see \Cref{sec:proof-lpn}). The $C^2$ condition for the nonlinearity\footnote{\Cref{prop:lpn} also holds if the nonlinearities are different, which we omit for simplicity of presentation.} $g$ is imposed to ensure differentiability of the ICNN $\psi_\theta$ and the LPN $f_\theta$, which will become useful in proving convergence for PnP algorithms with LPN in \cref{sec:solving-inverse}. Although this rules out popular choices like Rectifying Linear Units (ReLUs), there exist several alternatives satisfying these constraints. Following \cite{Huang2021-ds}, we adopt the \emph{softplus} function $g(x) = \frac{1}{\beta} \log (1 + \exp{(\beta x)}),$ a $\beta$-smooth approximation of ReLU. 
Importantly, LPN can be highly expressive (representing any continuous proximal operator) under reasonable settings, given the universality of ICNN \cite{Huang2021-ds}.

Networks defined by gradients of ICNN have been explored for inverse problems: \textcite{cohen2021has} used such networks to learn gradients of data-driven regularizers, thereby enforcing the learned regularizer to be convex.
While this is useful for the analysis of the optimization problem,  this cannot capture nonconvex log-priors that exist in
most cases of interest. On the other hand,
\textcite{hurault2022proximal} proposed parameterizing proximal operators as $f(\y) = \y - \nabla g(\y)$, where $\nabla g$ is $L$-Lipschitz with $L<1$. 
In practice, this is realized only approximately by regularizing its Lipschitz constant during training (%
see discussion in \Cref{sec:related-works}).
\rebuttal{Separately, gradients of ICNNs are also important in
data-driven optimal transport \cite{makkuva2020optimal,Huang2021-ds}}.

\paragraph{Recovering the prior from its proximal}
\label{sec:method-prior}
Once an LPN $f_\theta$ is obtained, we would like to recover its prox-primitive\footnote{Note our use of \emph{prox-primitive} to refer to the function $R$ with respect to the operator $\text{prox}_R$.}, $R_\theta$. This is important, as this function is precisely the regularizer in the variational objective,
$\min_{\x} \frac{1}{2} \|\y - A(\x)\|_2^2 + R_\theta(\x)$.
Thus, being able to evaluate $R_\theta$ at arbitrary points provides explicit information about the prior, enhancing interpretability of the learned regularizer. 
We start with the relation between $f$, $R_\theta$ and $\psi_\theta$
from \textcite{gribonval2020characterization} given by %
\begin{equation}
    R_\theta(f_\theta(\y)) = \langle \y, f_\theta(\y) \rangle - \frac{1}{2} \|f_\theta(\y)\|_2^2 - \psi_\theta(\y).
\end{equation}
Given our parameterization for $f_\theta$, all quantities are easily computable (via a forward pass of the LPN in \Cref{prop:lpn}). 
However, the above equation only allows to evaluate the regularizer $R_\theta$ at
points in the image of $f_\theta$, $f_\theta(\y)$, and not at an arbitrary point
$\x$. Thus, we must invert $f_\theta$, i.e. find $\y$ such that $f_\theta(\y)
= \x$. This inverse is nontrivial, since in general an LPN may not be invertible
or even surjective. Thus, as in \textcite{Huang2021-ds}, we add
a quadratic term to $\psi_\theta$, $\psi_{\theta}(\y; \alpha) = \psi_\theta(\y)
+ \frac{\alpha}{2}\|\y\|_2^2$, with $\alpha > 0$, turning $\psi_\theta$
strongly convex, and its gradient map, $f_\theta = \nabla \psi_\theta$, invertible and bijective. %
To compute this inverse, it suffices to minimize the strongly convex objective
\begin{equation}
\label{eq:invert-cvx}
    \min_\y \psi_{\theta}(\y\rebuttal{; \alpha}) - \langle \x, \y \rangle,
\end{equation}
which has a unique global minimizer $\hat{\y}$ satisfying the first-order optimality
condition $f_\theta(\hat{\y}) = \nabla \psi_\theta(\hat{\y}\rebuttal{; \alpha})
= \x$: the inverse we seek. Hence, computing the inverse amounts to solving a convex optimization problem---efficiently addressed by a variety of solvers, e.g. conjugate gradients.

Another feasible approach to invert $f_\theta$ is to simply optimize $\min_\y \|f_\theta(\y) - \x\|_2^2$, using, e.g., first-order methods. This problem is nonconvex in general, however, and thus does not allow for global convergence guarantees. Yet, we empirically find this approach work well on multiple datasets, yielding a solution $\hat{\y}$ with small mean squared error $\|f_\theta(\hat{\y})-\x\|_2^2$. %
We summarize the procedures for estimating the regularizer from an LPN in  \Cref{alg:prior} and \Cref{sec:alg-prior}.

\subsection{Training learned proximal networks via proximal matching}
\label{sec:method-train}
To solve inverse problems correctly, it is crucial that LPNs capture the true proximal operator of the underlying data distribution.
Given an unknown distribution $p_\x$, the goal of training an LPN is to learn the proximal operator of its log, $\prox_{-\log p_\x}:=f^*$.
Unfortunately, paired ground-truth samples $\{\x_i, f^*(\x_i)\}$ do not exist in
common settings---the prior distributions of many types of real-world data are unknown, making supervised training infeasible.
Instead, we seek to train an LPN using \emph{ only i.i.d.\ samples from the unknown data distribution} in an unsupervised way.

To this end, we introduce a novel loss function that we call \emph{proximal matching}. Based on the observation that the proximal operator is the \emph{maximum a posteriori} (MAP) denoiser for additive Gaussian noise, i.e. for samples $\y = \x + \sigma \vv$ with $\x \sim p_\x, \vv \sim \N(0, \mathbf{I})$, we train LPN to perform \emph{denoising} by minimizing a loss of the form 
\begin{equation}
\label{eq:base-loss}
    \underset{\x,\y}{\E} \left[ d(f_\theta(\y), \x)  \right],
\end{equation}
where $d$ is a suitable metric. Popular choices for $d$ include the squared
$\ell_2$ distance $\|f_\theta(\y) - \x\|_2^2$, the $\ell_1$ distance $\|f_\theta(\y) - \x\|_1$, or the Learned Perceptual Image Patch Similarity (LPIPS,
\cite{zhang2018unreasonable}), all of which have been used to train deep learning based denoisers \cite{zhang2017beyond,yu2019deep,tian2020deep}. However, denoisers trained with
these losses do not approximate the MAP denoiser, nor the proximal operator of the log-prior, $\prox_{-\log p_\x}$.
The squared $\ell_2$ distance, for instance,
leads to the minimum mean square error (MMSE) estimate%
given by the mean of the posterior, $\E[\x \mid \y]$. 
Similarly, the $\ell_1$ distance leads to the conditional marginal median of the posterior -- and not its maximum. As a concrete example, \Cref{fig:lap} illustrates the limitations of these metrics for learning the proximal operator of \rebuttal{the log-prior of a Laplacian distribution}.

We thus propose a new loss function that promotes the recovery of the correct proximal, dubbed \textbf{proximal matching loss}:
\begin{equation}
\label{eq:loss}
    \mathcal{L}_{PM}(\theta;\gamma) = \underset{\x,\y}{\E} \left[ m_\gamma(\|f_\theta(\y)- \x\|_2)  \right], \quad 
    m_{\gamma}(x) =
    1 - \frac{1}{(\pi\gamma^2)^{n/2}}\exp\left(-\frac{x^2}{\gamma^2}\right), \gamma > 0.
\end{equation}
Crucially, $\mathcal{L}_{PM}$ only depends on %
$p_\x$ (and Gaussian noise), allowing (approximate) proximal learning given only finite i.i.d.\ samples. 
Intuitively, $m_\gamma$ can be interpreted as an approximation to the Dirac function controlled by $\gamma$. Hence, minimizing the proximal matching loss $\mathcal{L}_{PM}$ amounts to maximizing the posterior probability $p_{\x \mid \y} (f_\theta(\y))$, and therefore results in the MAP denoiser (and equivalently, the proximal of log-prior). We now make this precise and show that minimizing $\mathcal{L}_{PM}$ yields the proximal operator of the log-prior almost surely as $\gamma \searrow 0$.

\begin{theorem}[Learning via Proximal Matching]
\label{thm:continuous}
Consider a signal $\x \sim p_\x$, where $\x$ is bounded and $p_\x$ is a continuous density,\footnote{That is, $\x$ admits a continuous probability density $p$ with respect to the Lebesgue measure on $\mathbb{R}^n$.}
and a noisy observation $\y = \x + \sigma \vv$,
where $\vv \sim \mathcal{N}(0, \mathbf{I})$ and $\sigma > 0$.
Let $m_\gamma(x) : \R \to \R$ be defined as in \Cref{eq:loss}.
Consider the optimization problem
\begin{equation}
\label{eq:cont-fstar}
f^* = \argmin_{f\ \mathrm{measurable}} 
\lim_{\gamma \searrow 0} \E_{\x,\y} \left[ m_\gamma \left( \|f(\y) - \x\|_2
\right)  \right].
\end{equation}
Then, almost surely (i.e., for almost all $\y$), $f^*(\y) = \argmax_{\bc} p_{\x \mid \y}(\bc) \triangleq \prox_{-\sigma^2\log p_\x}(\y)$.
\end{theorem}

We defer the proof to \Cref{sec:proof-continuous} and instead make a few remarks. First,
while the result above was presented for the loss defined in \Cref{eq:loss} for simplicity, this holds in greater generality for loss functions satisfying specific technical conditions (see \Cref{remark:other-prox-matching-loss}).
Second, an analogous result for discrete distributions can also be derived, and we include this companion result in \Cref{thm:discrete}, \Cref{sec:thm-discrete}.
\rebuttal{Third, the Gaussian noise level $\sigma$ acts as a scaling factor on
the learned regularizer, as indicated by $f^*(\y) = \prox_{-\sigma^2\log p_x}(\y)$. 
Thus varying the noise level effectively varies the strength of the regularizer.}
Lastly, to bring this theoretical guarantee to practice, we progressively
decrease $\gamma$ until a small positive amount during training according to
a schedule function $\gamma(\cdot)$ for an empirical sample (instead of the
expectation)\rebuttal{, and pretrain LPN with $\ell_1$ loss before proximal
matching}. We include an algorithmic description of training via proximal
matching in \Cref{sec:alg-train}, \Cref{alg:train}. %
Connections between the proximal matching loss \Cref{eq:loss} and prior work on impulse denoising and modal regression are discussed in \Cref{sec:related-works}.

Before moving on, we summarize the results of this section: the parameterization in \Cref{prop:lpn} guarantees that LPN implement a proximal operator for some regularizer function; the optimization problem in \Cref{eq:invert-cvx} then provides a way to evaluate this regularizer function at arbitrary points; and lastly, \Cref{thm:continuous} shows that if we want the LPN to recover the correct proximal (of the log-prior of data distribution), then \emph{proximal matching} is the correct learning strategy for these networks.

\vspace{-5pt}
\section{Solving Inverse Problems with LPN}
\label{sec:solving-inverse}
\vspace{-5pt}
Once an LPN is trained, it can be used to solve inverse problems within the PnP
framework \cite{venkatakrishnan2013plug} by substituting any occurrence of the proximal step
$\prox_R$ with the learned proximal network $f_\theta$. 
As with any PnP method, our LPN can be flexibly plugged into a wide range of iterative algorithms, such as PGD, ADMM, or HQS.
Chiefly, and in contrast to previous PnP approaches, our LPN-PnP approach provides the guarantee that the employed denoiser is indeed a proximal operator. %
As we will now show, this enables convergence guarantees absent any additional
assumptions on the learned network. %
We provide an instance of solving inverse problems using LPN with PnP-ADMM in \Cref{alg:admm}, and another example with PnP-PGD in
\Cref{alg:pgd}.

\begin{wrapfigure}{r}{0.55\textwidth}
\vspace{-30pt}
\begin{minipage}{0.54\textwidth}
\begin{algorithm}[H]
\caption{Solving inverse problem with LPN and PnP-ADMM}
\label{alg:admm}

\begin{algorithmic}[1]

\Require Trained LPN $f_\theta$, operator $A$, measurement $\y$, initial $\x_0$, number of iterations $K$, penalty parameter $\rho$

\State $\uu_0 \gets 0$, $\z_0 \gets \x_0$

\For{$k = 0$ \textbf{to} $K-1$}
    \State $\x_{k+1} \gets \argmin_{\x} \{ \frac{1}{2} \| \y - A(\x) \|_2^2
    + \frac{\rho}{2} \|\z_k - \uu_k - \x\|_2^2 \}$

    \State $\uu_{k+1} \gets \uu_k + \x_{k+1} - \z_{k}$

    \State $\z_{k+1} \gets f_\theta \left( \uu_{k+1} + \x_{k+1} \right)$ 

\EndFor

\Ensure $\x_K$ 
\end{algorithmic}
\end{algorithm}

\end{minipage}
\vspace{-10pt}
\end{wrapfigure}

\paragraph{Convergence Guarantees in Plug-and-Play Frameworks}

Because LPNs are by construction proximal operators, %
PnP schemes with plug-in LPNs correspond to iterative algorithms for minimizing
the variational objective \Cref{eq:main_inverse_problem}, with the
implicitly-defined regularizer $R_{\theta}$ associated to the LPN.
As a result, convergence guarantees for PnP schemes with LPNs follow readily
from convergence analyses of the corresponding optimization procedure, under
suitably general assumptions.
We state and discuss such a guarantee for using an LPN with PnP-ADMM (\cref{alg:admm}) in \cref{thm:pnp-admm-lpn-stationary}---our proof appeals to the nonconvex ADMM analysis of \citet{Themelis2020-jj}.

\begin{theorem}[Convergence guarantee for running PnP-ADMM with LPNs]\label{thm:pnp-admm-lpn-stationary}
    Consider the sequence of iterates $(\x_{k}, \uu_k, \z_k)$, $k \in \{0, 1, \dots\}$,
    defined by \Cref{alg:admm} run with a linear measurement operator $\A$ and an
    LPN $f_{\theta}$ with softplus activations, trained with $0 < \alpha < 1$.
    Assume further that the penalty parameter $\rho$ satisfies
    $\rho > \| \A\adj \A \|$.
    Then the sequence of iterates $(\x_{k}, \uu_k, \z_k)$ converges to a limit
    point $(\x^*, \uu^*, \z^*)$ which %
    is a fixed point of the PnP-ADMM iteration (\Cref{alg:admm}).
\end{theorem}

We defer the proof of \Cref{thm:pnp-admm-lpn-stationary}
to \Cref{sec:proof-pnp-admm-lpn-stationary}.
There, we moreover show that $\x_k$ converges to a critical point of the regularized reconstruction cost \cref{eq:main_inverse_problem}
with regularization function $R = \lambda R_{\theta}$, %
where $R_{\theta}$ is the implicitly-defined regularizer associated to $f_{\theta}$
(i.e.\, $f_{\theta} = \prox_{R_{\theta}}$) and the regularization strength
$\lambda$ depends on parameters of the PnP algorithm (%
$\lambda = \rho$ for PnP-ADMM).
In addition, we emphasize that 
\Cref{thm:pnp-admm-lpn-stationary} %
requires the bare minimum of assumptions
on the trained LPN: it holds for any LPNs by construction, under assumptions that are all actionable and achievable in practice (on network weights, activation, and strongly convex parameter). This should be contrasted to PnP schemes that utilize a
black-box denoiser -- convergence guarantees in this
setting require restrictive assumptions on the denoiser, such as contractivity
\cite{Ryu2019-ca}, (firm) nonexpansivity
\cite{Sun2019-zc,Sun2021-ll,cohen2021has,Cohen2021-mp,Tan2023-gd},
or other Lipschitz constraints \cite{hurault2022gradient,hurault2022proximal}, which are difficult to verify or enforce in practice without
sacrificing denoising performance. Alternatively, other PnP schemes sacrifice expressivity
for a principled approach by enforcing that the denoiser takes a restrictive
form, such as being a (Gaussian) MMSE denoiser \cite{Xu2020-my}, %
a linear denoiser \cite{Hauptmann2023-nu}, or the proximal operator of an
implicit convex function \cite{Sreehari2016-in,Teodoro2018-ly}.

The analysis of LPNs we use to prove \Cref{thm:pnp-admm-lpn-stationary} is general enough to be extended straightforwardly to other PnP optimization schemes.
Under a similarly-minimal level of assumptions to \Cref{thm:pnp-admm-lpn-stationary}, we give in \cref{thm:pnp-pgd-lpn-stationary} (\cref{sec:thm-pnp-pgd-lpn}) a convergence analysis for PnP-PGD (\cref{alg:pgd}), which tends to perform slightly worse than PnP-ADMM in practice.

\section{Experiments}
\label{sec:experiments}

\begin{wrapfigure}{r}{.7\textwidth}
    \centering
    \includegraphics[trim = 10 30 0 70, width=0.7\textwidth]{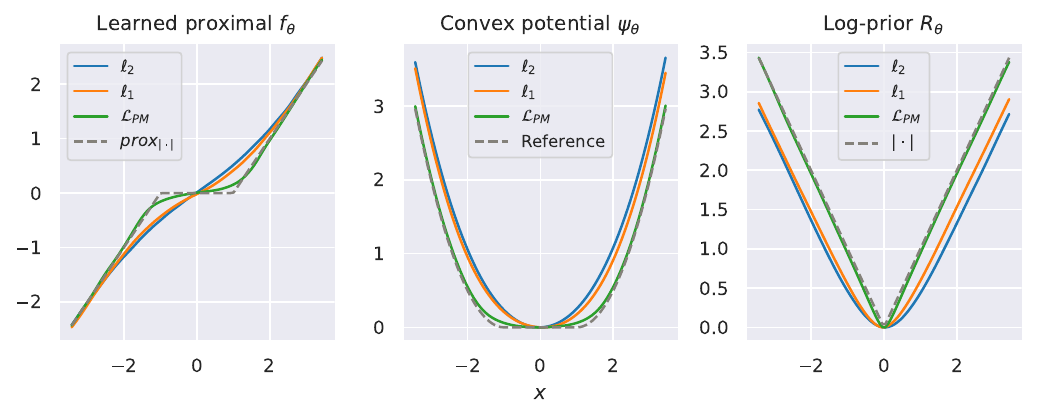}
    \caption{The proximal $f_\theta$, convex potential $\psi_\theta$, and log-prior $R_\theta$ learned by LPN via the squared $\ell_2$ loss, $\ell_1$ loss, and proximal matching loss $\mathcal{L}_{PM}$ for a Laplacian distribution (ground truth in gray). %
    }
    \label{fig:lap}
\end{wrapfigure}
We evaluate LPN on datasets of increasing complexity, from an analytical one-dimensional example of a Laplacian distribution to image datasets of increasing dimensions: MNIST ($28\times 28$) \cite{lecun1998mnist}, CelebA ($128\times 128$) \cite{liu2018large}, and Mayo-CT ($512\times 512$) \cite{mccollough2016tu}. We demonstrate how the ability of LPN to learn an exact proximal for the correct prior reflects on natural values for the obtained log-likelihoods. Importantly, we showcase the performance of LPN for real-world inverse problems on CelebA and Mayo-CT, for deblurring, sparse-view tomographic reconstruction, and compressed sensing, comparing it with other state-of-the-art unsupervised approaches for (unsupervised) image restoration. See full experimental details in \Cref{sec:details}.

\subsection{What is your prior?}
\paragraph{Learning soft-thresholding from Laplacian distribution}
We first experiment with a distribution whose log-prior has a known proximal operator, the 1-D Laplacian distribution
$p(x \mid \mu,b) = \frac{1}{2b} \exp \left( -\frac{|x - \mu|}{b} \right)$. Letting $\mu=0,\ b=1$ for simplicity, the negative log likelihood (NLL) is the $\ell_1$ norm,
$-\log p (x) = |x| - \log(\frac{1}{2})$, and its proximal can be written is the soft-thresholding function
$\prox_{-\log p}(x) = \sign(x) \max(|x| - 1, 0).$ 
We train a LPN on i.i.d.\ samples from the Laplacian and Gaussian noise, as in \Cref{eq:base-loss}, and compare different loss functions, including the proximal matching loss $\mathcal{L_{PM}}$, for which we consider different $\gamma \in \{0.5,0.3,0.1\}$ in $\mathcal{L_{PM}}$ (see \cref{eq:loss}).

As seen in \Cref{fig:lap}, when using either the $\ell_2$ or $\ell_1$ loss, the learned prox differs from the correct soft-thresholding function. Indeed, verifying our analysis in \Cref{sec:method-train}, these yield the posterior mean and median, respectively, rather than the posterior mode. With the matching loss $\mathcal{L}_{PM}$ ($\gamma = 0.1$ in \cref{eq:loss}), the learned proximal matches much more closely the ground-truth prox. The third panel in \Cref{fig:lap} further depicts the learned log-prior $R_\theta$ associated with each LPN $f_\theta$, computed using the algorithm in \Cref{sec:method-prior}. Note that $R_\theta$ does not match the ground-truth log-prior $|\cdot|$ for $\ell_2$ and $\ell_1$ losses, but converges to the correct prior with $\mathcal{L}_{PM}$ (see more results for different $\gamma$ in \Cref{sec:experiments-lap-appx}).
Note that we normalize the offset of learned priors by setting the minimum value to $0$ for visualization:
the learned log-prior $R_\theta$ has an arbitrary offset (since we only estimate the log-prior). In other words, LPN is only able to learn the relative density of the distribution due to the intrinsic scaling symmetry of the proximal operator.

\paragraph{Learning a prior for MNIST}
\label{sec:experiments-mnist}

\begin{figure}
    \centering

    \subcaptionbox{
    \label{fig:mnist-noise}}
    {
        \includegraphics[trim = 0 30 0 30, width=0.2\textwidth]{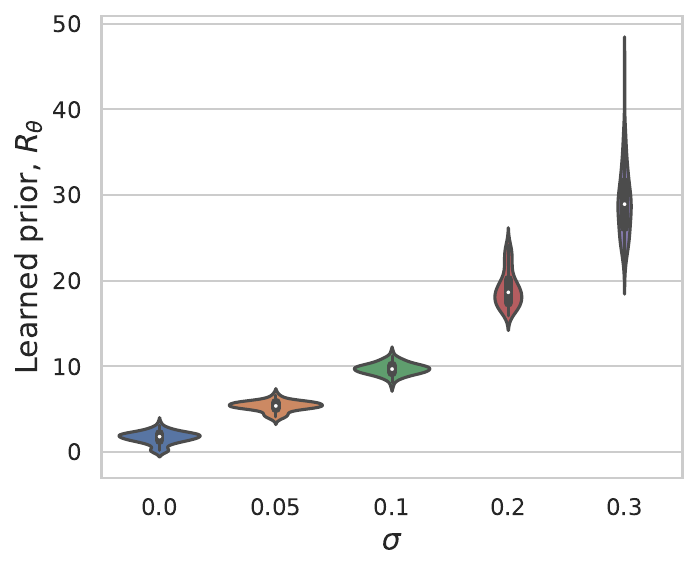}
        \includegraphics[width=0.27\textwidth,trim=0 0 0 0,clip]{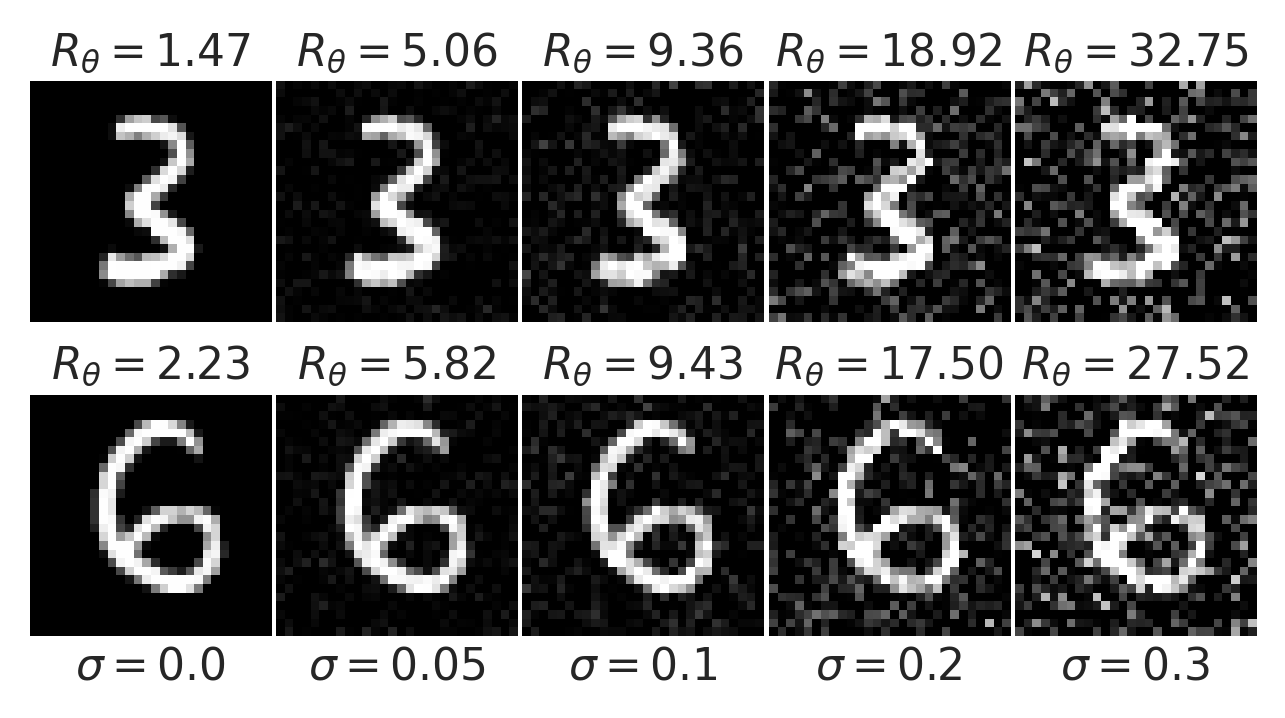}
    }
    \subcaptionbox{
    \label{fig:mnist-cvx}}
    {
        \includegraphics[trim = 0 30 0 30, width=0.2\textwidth]{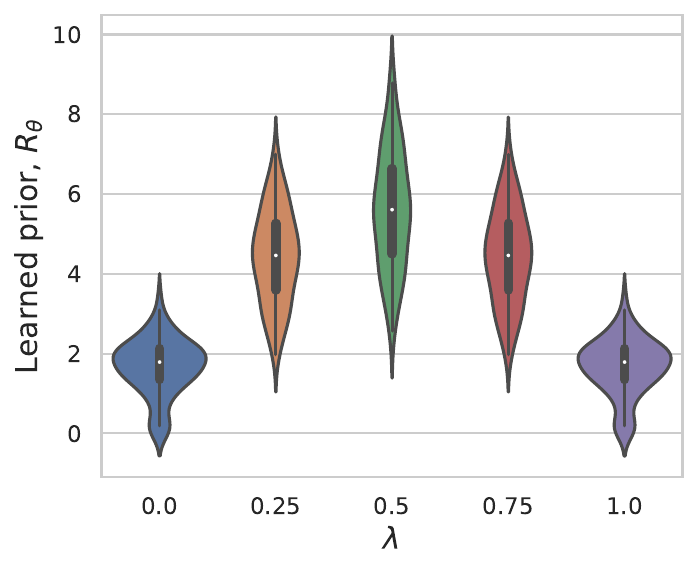}
        \includegraphics[ width=0.27\textwidth,trim=0 0 0 0,clip]{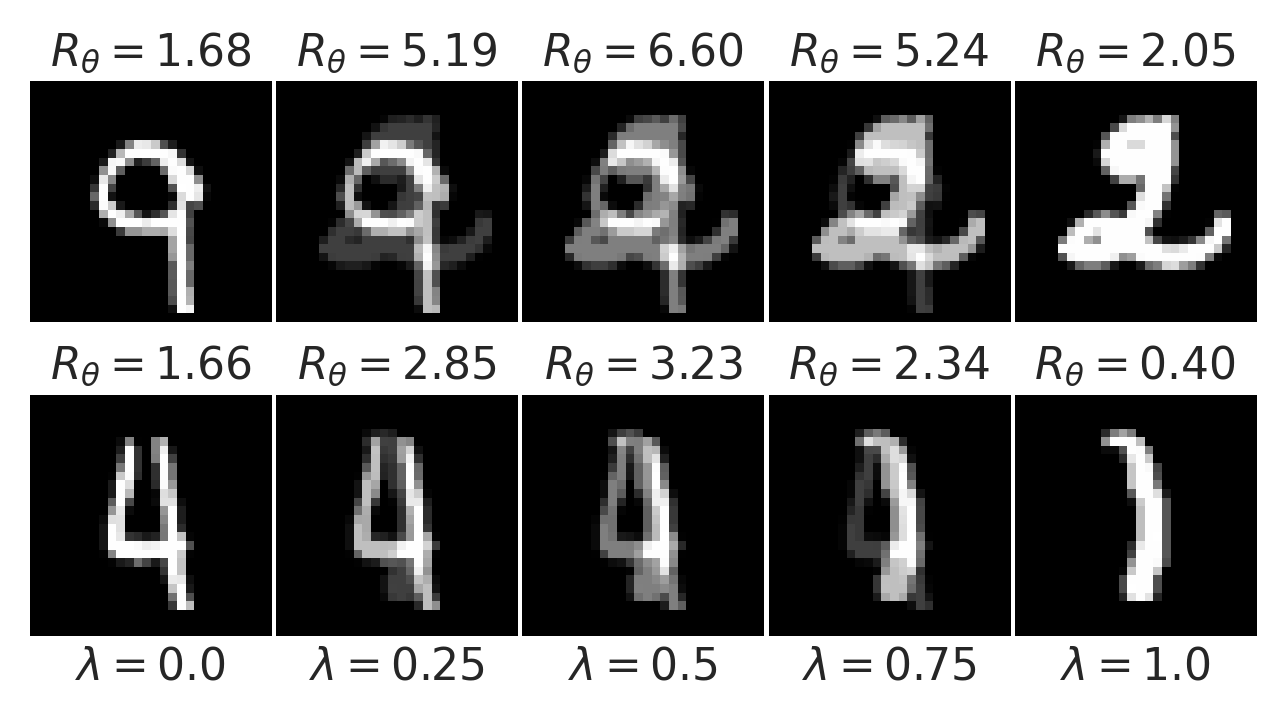}
    }
    \vspace{-5pt}
    \caption{Left: log-prior $R_\theta$ learned by LPN on MNIST (computed over 100 test images), evaluated at images corrupted by (a) additive Gaussian noise, and (b) convex combination of two images $(1-\lambda) \x + \lambda \x'$. Right: the prior evaluated at individual examples.}
    \label{fig:mnist-prior}
\end{figure}

Next, we train an LPN on MNIST, attempting to learn a general restoration method for hand-written digits---and through it, a prior of the data. For images, we implement the LPN with convolution layers; see \Cref{sec:details-mnist} for more details.  Once the model is learned, we evaluate the obtained prior on a series of inputs with different types and degrees of perturbations in order to gauge how such modifications to the data are reflected by the learned prior. 
\Cref{fig:mnist-noise} visualizes the change of prior $R_\theta$ after adding increasing levels of Gaussian noise. As expected,  as the noise level increases, the values reported by the log-prior also increases, reflecting that noisier images are less likely according to the data distribution of real images.

The lower likelihood upon perturbations of the samples is general. We depict examples with image blur in \Cref{sec:experiments-mnist-blur}, and also present a study that depicts the non-convexity 
\begin{wrapfigure}{r}{0.6\textwidth}
    \centering
    \vspace{-6pt}
    \subcaptionbox{
        Sparse-view tomographic reconstruction.
        \label{fig:ct-tomo}
    }
    {
        \includegraphics[width=0.58\textwidth]{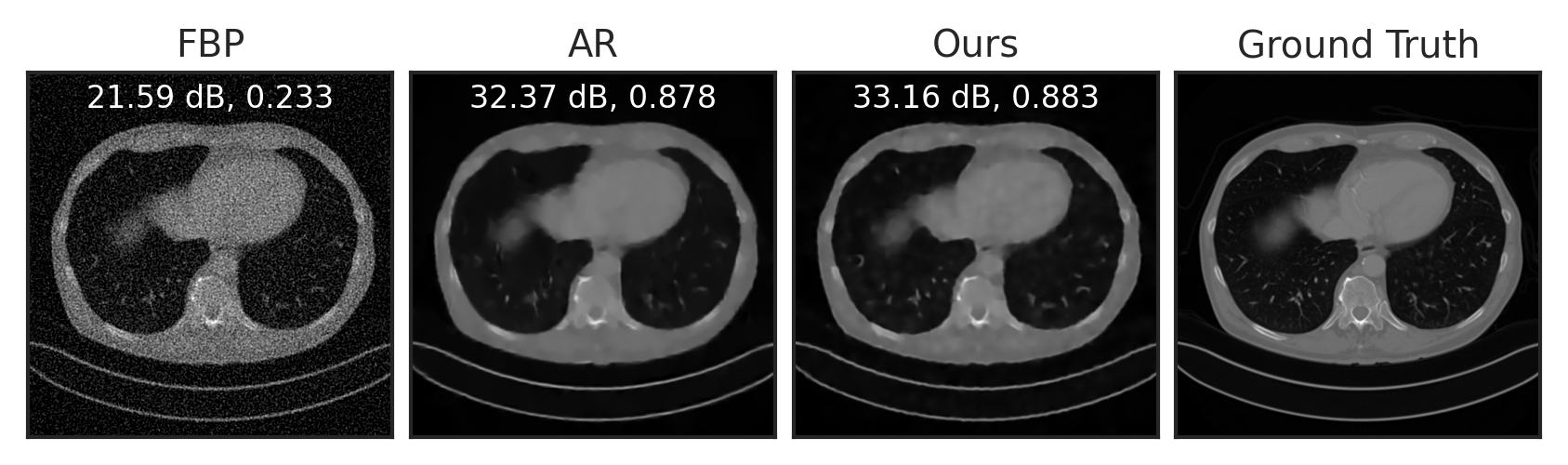}
    }
    \subcaptionbox{
        Compressed sensing (compression rate $=1/16$).
        \label{fig:ct-cs}
    }
    {
        \includegraphics[width=0.58\textwidth]{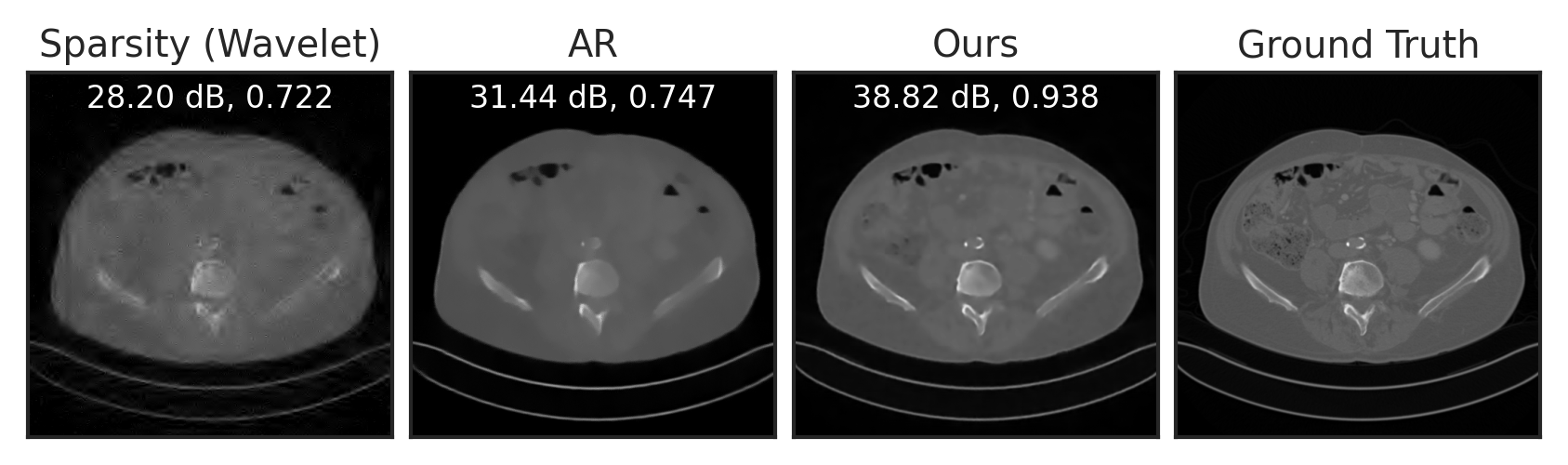}
    }
    
    \caption{Results on the Mayo-CT dataset (details in text).}
    \label{fig:ct}
    \vspace{-6pt}
\end{wrapfigure}
of the log-prior in  \Cref{fig:mnist-cvx}: we evaluate the learned prior at the convex combination of two samples, $\lambda \x + (1-\lambda) \x'$ of two testing images $\x$ and $\x'$, with $\lambda \in [0,1]$. 
As depicted in \Cref{fig:mnist-cvx}, as $\lambda$ goes from $0$ to $1$, the learned prior first increases and then decreases, exhibiting a nonconvex shape. This is natural, since the convex combination of two images no longer resembles a natural image, demonstrating that the true prior should indeed be nonconvex. 
As we see, LPN can correctly learn this qualitative property in the prior, while existing approaches using convex priors, either hand-crafted \cite{tikhonov1977solutions,rudin1992nonlinear,mallat1999wavelet,beck2009fast,elad2006image,chambolle2011first} or data-driven \cite{mukherjee2021end,cohen2021has}, are suboptimal by not faithfully capturing such nonconvexity. 
All these results collectively show that LPN can learn a good approximation of the prior of images from data samples, and the learned prior either recovers the correct log-prior when it is known (in the Laplacian example), or provides a prior that coincides with human preference of natural, realistic images. With this at hand, we now move to address more challenging inverse problems.

\begin{figure}[!t]
\centering

\subcaptionbox{$\sigma_{blur}=1.0, \sigma_{noise}=0.02$.
\label{fig:celeba-1}}
{\includegraphics[width=0.49\linewidth]{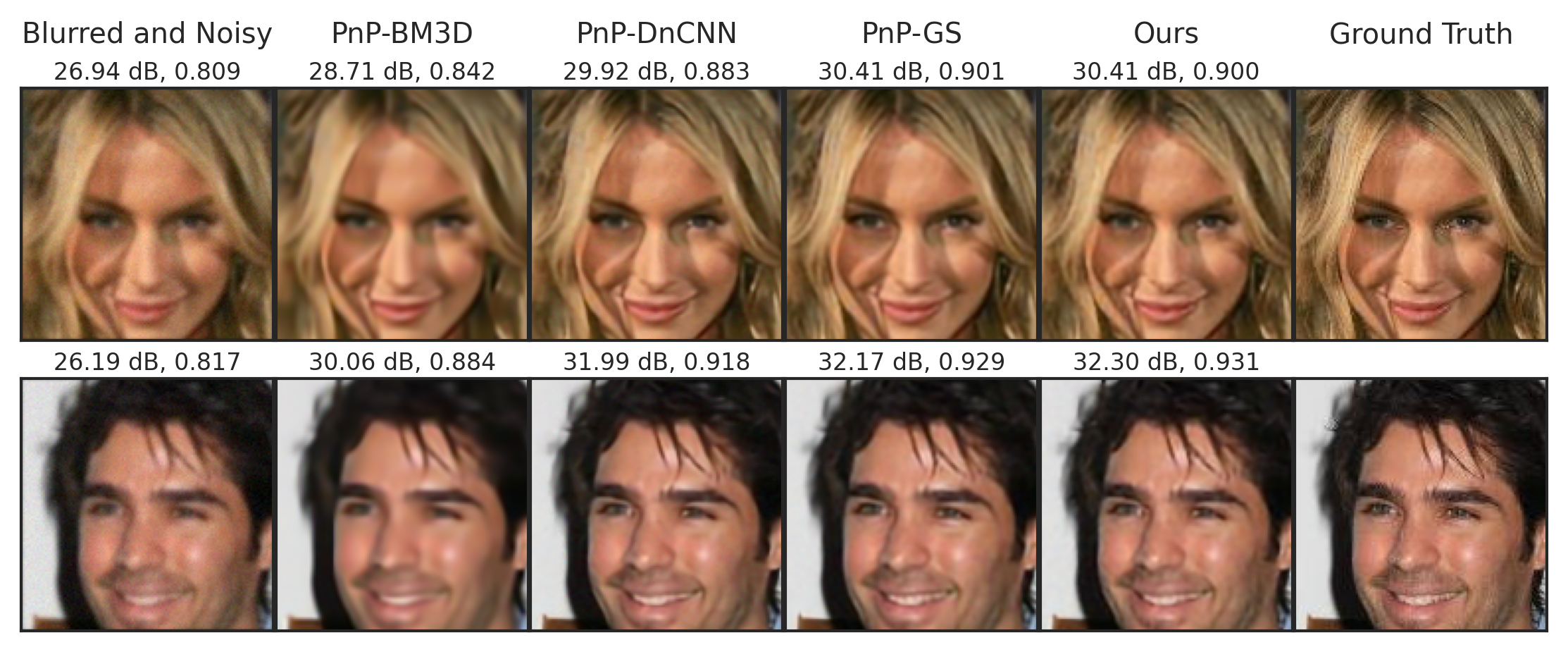}}
\subcaptionbox{$\sigma_{blur}=1.0, \sigma_{noise}=0.04$.
\label{fig:celeba-2}}
{\includegraphics[width=0.49\linewidth]{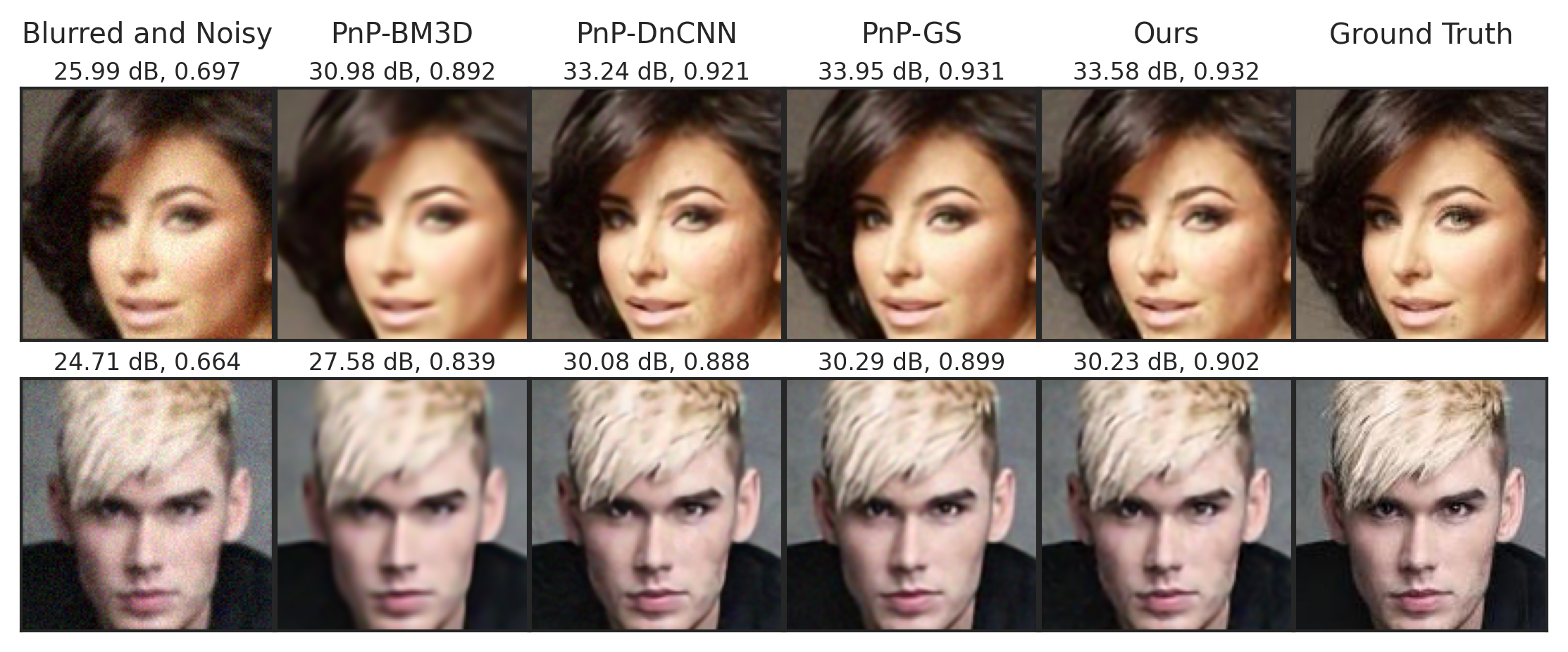}}

\vspace{0pt}
\caption{\rebuttal{Visual results for deblurring on CelebA using Plug-and-Play with different denoisers (BM3D, DnCNN, the gradient step (GS) Prox-DRUNet, and our LPN)}, for different Gaussian blur kernel standard deviation $\sigma_{blur}$ and noise standard deviation $\sigma_{noise}$. PSNR and SSIM are presented above each prediction.}
\label{fig:celeba}
\end{figure}

\begin{table}[]
\vspace{-5pt}
\rebuttal{
\centering
\caption{\rebuttal{Deblurring on CelebA, over 20 samples. \textbf{Bold} (\underline{underline}) for the best (second best) score.}}
\vspace{-5pt}
\label{tab:celeba-deblur}
\resizebox{\textwidth}{!}{%
\begin{tabular}{@{}ccccccccc@{}}
\toprule
\multirow{2}{*}{METHOD} & \multicolumn{2}{c}{$\sigma_{blur}=1, \, \sigma_{noise}=.02$} & \multicolumn{2}{c}{$\sigma_{blur}=1, \, \sigma_{noise}=.04$} & \multicolumn{2}{c}{$\sigma_{blur}=2, \, \sigma_{noise}=.02$} & \multicolumn{2}{c}{$\sigma_{blur}=2, \, \sigma_{noise}=.04$} \\
 \cmidrule(lr){2-3} \cmidrule(lr){4-5} \cmidrule(lr){6-7} \cmidrule(lr){8-9} 
& PSNR($\uparrow$) & SSIM($\uparrow$) & PSNR($\uparrow$) & SSIM($\uparrow$) & PSNR($\uparrow$) & SSIM($\uparrow$) & PSNR($\uparrow$) & SSIM($\uparrow$) \\ \midrule
Blurred and Noisy & 27.0 $\pm$ 1.6 & .80 $\pm$ .03 & 24.9 $\pm$ 1.0 & .63 $\pm$ .05 & 24.0 $\pm$ 1.7 & .69 $\pm$ .04 & 22.8 $\pm$ 1.3 & .54 $\pm$ .04 \\ 
PnP-BM3D \cite{venkatakrishnan2013plug} & 31.0 $\pm$ 2.7 & .88 $\pm$ .04 & 29.5 $\pm$ 2.2 & .84 $\pm$ .05 & 28.5 $\pm$ 2.2 & .82 $\pm$ .05 & 27.6 $\pm$ 2.0 & .79 $\pm$ .05 \\ 
PnP-DnCNN \cite{zhang2017beyond} & 32.3 $\pm$ 2.6 & .90 $\pm$ .03 & 30.9 $\pm$ 2.1 & .87 $\pm$ .04 & 29.5 $\pm$ 2.0 & .84 $\pm$ .04 & 28.3 $\pm$ 1.8 & .79 $\pm$ .05 \\ 
PnP-GS \cite{hurault2022proximal} & \bf 33.0 $\pm$ 3.0 & \bf .92 $\pm$ .03 & \bf 31.4 $\pm$ 2.4 & \bf .89 $\pm$ .03 & \bf 30.1 $\pm$ 2.5 & \bf .87 $\pm$ .04 & \bf 29.3 $\pm$ 2.3 & \bf .84 $\pm$ .05 \\
Ours & \bf 33.0 $\pm$ 2.9 & \textbf{.92 $\pm$ .03} & \underline{31.3 $\pm$ 2.3} & \textbf{.89 $\pm$ .03} & \textbf{30.1 $\pm$ 2.4} & \textbf{.87 $\pm$ .04} & \underline{29.1 $\pm$ 2.2} & \textbf{.84 $\pm$ .04} \\ 
\bottomrule
\end{tabular}%
}
}
\vspace{-10pt}
\end{table}

\vspace{-5pt}
\subsection{Solving inverse problems with LPN}

\vspace{-3pt}

\paragraph{CelebA}  We now showcase the capability of LPN for solving realistic inverse problems. We begin by training an LPN on the CelebA dataset, and employ the PnP-ADMM methodology for deblurring. We compare with state-of-the-art PnP approaches: PnP-BM3D \cite{venkatakrishnan2013plug}, which uses the BM3D denoiser \cite{dabov2007image}, PnP-DnCNN, which uses DnCNN as the denoiser \cite{zhang2017beyond} , \rebuttal{and PnP-GS using the gradient step proximal denoiser called Prox-DRUNet \cite{hurault2022proximal}.} Both DnCNN and Prox-DRUNet have been trained on CelebA.
As shown in \Cref{tab:celeba-deblur}, LPN achieves state-of-the-art result
across multiple blur degrees, noise levels and metrics considered. As visualized
in \Cref{fig:celeba}, LPN significantly improves the quality of the blurred
image, demonstrating the effectiveness of the learned prior for solving inverse
problems. 
\begin{wraptable}{r}{0.5\textwidth}
\vspace{-0.3em}
\captionsetup{justification=centering}
\caption{Numerical results for inverse problems on Mayo-CT, computed over 128 test images.}
\label{tab:ct}
\centering
\resizebox{0.5\textwidth}{!}{

\begin{tabular}{lcc}
\toprule
METHOD & PSNR ($\uparrow$) & SSIM ($\uparrow$) \\ 
\\[-1ex] 
\multicolumn{3}{l}{\textbf{Tomographic reconstruction}} \\ \midrule
FBP & 21.29 & .203 \\ 
\midrule \emph{{Operator-agnostic}} \\ 
AR \cite{lunz2018adversarial} & 33.48 & .890 \\ 
Ours & \textbf{{34.14}} & \textbf{{.891}} \\ 
\midrule \emph{{Operator-specific}} \\ 
UAR \cite{mukherjee2021end} & \textbf{{34.76}} & \textbf{{.897}} \\ 
\\[-1ex] 
\multicolumn{3}{l}{\textbf{Compressed sensing (compression rate $=1/16$)}} \\ \midrule
Sparsity (Wavelet) & 26.54 & .666 \\ 
AR \cite{lunz2018adversarial} & 29.71 & .712 \\ 
Ours & \textbf{{38.03}} & \textbf{{.919}} \\ 
\\[-1ex] 
\multicolumn{3}{l}{\textbf{Compressed sensing (compression rate $=1/4$) }} \\ \midrule
Sparsity (Wavelet) & 36.80 & .921 \\ 
AR \cite{lunz2018adversarial} & 37.94 & .920 \\ 
Ours & \textbf{{44.05}} & \textbf{{.973}} \\ 
\bottomrule
\vspace{-30pt}
\end{tabular}
}
\end{wraptable}
Compared to 
the state-of-the-art methods, LPN can 
produce sharp images
with comparable visual
quality, while allowing for the evaluation of the
obtained prior---which is impossible with any of the other methods.

\paragraph{Mayo-CT}

We train LPN on the public Mayo-CT dataset \cite{mccollough2016tu} of Computed Tomography (CT) images, and evaluate it for two inverse tasks: sparse-view CT reconstruction and compressed sensing. 
For sparse-view CT reconstruction, we compare with filtered back-projection (FBP) \cite{willemink2019evolution}, the adversarial regularizer (AR) method of \cite{lunz2018adversarial} with an explicit regularizer, and its improved and subsequent version using unrolling (UAR) \cite{mukherjee2021end}. 
UAR is trained to solve the inverse problem for a specific measurement operator (i.e., task-specific), while both AR and LPN are generic regularizers that are applicable to any measurement model (i.e., task-agnostic). In other words, the comparison with UAR is not completely fair, but we still include it here for a broader comparison.

Following \textcite{lunz2018adversarial}, we simulate CT sinograms using a parallel-beam geometry with 200 angles and 400 detectors, with an undersampling rate of $\frac{200\times 400}{512^2} \approx 30\%$. See \Cref{sec:details-ct} for experimental details.
As visualized in \Cref{fig:ct-tomo}, compared to the baseline FBP, LPN can significantly reduce noise in the reconstruction. Compared to AR, LPN result is slightly sharper, with higher PNSR. The numerical results in \Cref{tab:ct} show that our method significantly improves over the baseline FBP, outperforms the unsupervised counterpart AR, and performs just slightly worse than the supervised approach UAR---\emph{without even having had access to the used forward operator}.
\Cref{fig:ct-cs} and \Cref{tab:ct} show compressed sensing results with compression rates of $\frac{1}{4}$ and $\frac{1}{16}$. LPN significantly outperforms the baseline and AR, demonstrating better generalizability to different forward operators and inverse problems.

\vspace{0pt}
\section{Conclusion}
\vspace{0pt}
The learned proximal networks presented in this paper are guaranteed to parameterize proximal operators. We showed how the
prox-primitive, regularizer function of the resulting proximal (parameterized by an
LPN) can be recovered, allowing explicit characterization of the prior learned
from data. Furthermore, via proximal matching, LPN can approximately learn the correct prox (i.e. that of
the log-prior) of an unknown distribution from only i.i.d.\ samples. When used to solve general inverse problems, LPN achieves state-of-the-art results while providing more interpretability by explicit characterization of the (nonconvex) prior, with convergence guarantees. The ability to not only provide unsupervised models for general inverse problems but, chiefly, to characterize the priors learned from data open exciting new research questions of uncertainty quantification \cite{angelopoulos2022image,teneggi2023trust,sun2021deep}, sampling \cite{Kadkhodaie2021-kh,Kawar2021-eq,chung2022diffusion,Kawar2022-hu,feng2023score}, equivariant learning \cite{chen2023imaging,chen2021equivariant,chen2022robust}, learning without ground-truth \cite{tachella2023sensing,tachella2022unsupervised,gao2023image}, and robustness \cite{jalal2021robust,darestani2021measuring}, all of which constitute matter of ongoing work.

\clearpage

\section*{Acknowledgments}
This research has been supported by NIH Grant P41EB031771, as well as by the Toffler Charitable Trust and by the Distinguished Graduate Student Fellows program of the KAVLI Neuroscience Discovery Institute.

\bibliography{backmatter/references}

\begin{thebibliography}{128}
\providecommand{\natexlab}[1]{#1}
\providecommand{\url}[1]{\texttt{#1}}
\expandafter\ifx\csname urlstyle\endcsname\relax
  \providecommand{\doi}[1]{doi: #1}\else
  \providecommand{\doi}{doi: \begingroup \urlstyle{rm}\Url}\fi

\bibitem[Adler et~al.(2010)Adler, Hel-Or, and Elad]{adler2010shrinkage}
Amir Adler, Yacov Hel-Or, and Michael Elad.
\newblock A shrinkage learning approach for single image super-resolution with overcomplete representations.
\newblock In \emph{Computer Vision--ECCV 2010: 11th European Conference on Computer Vision, Heraklion, Crete, Greece, September 5-11, 2010, Proceedings, Part II 11}, pp.\  622--635. Springer, 2010.

\bibitem[Adler \& {\"O}ktem(2018)Adler and {\"O}ktem]{adler2018learned}
Jonas Adler and Ozan {\"O}ktem.
\newblock Learned primal-dual reconstruction.
\newblock \emph{IEEE transactions on medical imaging}, 37\penalty0 (6):\penalty0 1322--1332, 2018.

\bibitem[Aggarwal et~al.(2018)Aggarwal, Mani, and Jacob]{aggarwal2018modl}
Hemant~K Aggarwal, Merry~P Mani, and Mathews Jacob.
\newblock Modl: Model-based deep learning architecture for inverse problems.
\newblock \emph{IEEE transactions on medical imaging}, 38\penalty0 (2):\penalty0 394--405, 2018.

\bibitem[Amos et~al.(2017)Amos, Xu, and Kolter]{Amos2017-ql}
Brandon Amos, Lei Xu, and J~Zico Kolter.
\newblock Input convex neural networks.
\newblock In \emph{Proceedings of the 34th International Conference on Machine Learning}, volume~70 of \emph{Proceedings of Machine Learning Research}, pp.\  146--155. PMLR, 2017.
\newblock URL \url{https://proceedings.mlr.press/v70/amos17b.html}.

\bibitem[Angelopoulos et~al.(2022)Angelopoulos, Kohli, Bates, Jordan, Malik, Alshaabi, Upadhyayula, and Romano]{angelopoulos2022image}
Anastasios~N Angelopoulos, Amit~Pal Kohli, Stephen Bates, Michael Jordan, Jitendra Malik, Thayer Alshaabi, Srigokul Upadhyayula, and Yaniv Romano.
\newblock Image-to-image regression with distribution-free uncertainty quantification and applications in imaging.
\newblock In \emph{International Conference on Machine Learning}, pp.\  717--730. PMLR, 2022.

\bibitem[Arridge et~al.(2019)Arridge, Maass, {\"O}ktem, and Sch{\"o}nlieb]{arridge2019solving}
Simon Arridge, Peter Maass, Ozan {\"O}ktem, and Carola-Bibiane Sch{\"o}nlieb.
\newblock Solving inverse problems using data-driven models.
\newblock \emph{Acta Numerica}, 28:\penalty0 1--174, 2019.

\bibitem[Attouch et~al.(2010)Attouch, Bolte, Redont, and Soubeyran]{Attouch2010-vs}
H{\'e}dy Attouch, J{\'e}r{\^o}me Bolte, Patrick Redont, and Antoine Soubeyran.
\newblock Proximal alternating minimization and projection methods for nonconvex problems: An approach based on the {Kurdyka-{\L}Ojasiewicz} inequality.
\newblock \emph{Mathematics of Operations Research}, 35\penalty0 (2):\penalty0 438--457, May 2010.
\newblock ISSN 0364-765X.
\newblock \doi{10.1287/moor.1100.0449}.
\newblock URL \url{http://dx.doi.org/10.1287/moor.1100.0449}.

\bibitem[Attouch et~al.(2013)Attouch, Bolte, and Svaiter]{Attouch2013-vc}
H{\'e}dy Attouch, J{\'e}r{\^o}me Bolte, and Benar~Fux Svaiter.
\newblock Convergence of descent methods for semi-algebraic and tame problems: proximal algorithms, forward--backward splitting, and regularized {Gauss--Seidel} methods.
\newblock \emph{Mathematical Programming. A Publication of the Mathematical Programming Society}, 137\penalty0 (1):\penalty0 91--129, February 2013.
\newblock ISSN 0025-5610, 1436-4646.
\newblock \doi{10.1007/s10107-011-0484-9}.
\newblock URL \url{https://doi.org/10.1007/s10107-011-0484-9}.

\bibitem[Balke et~al.(2022)Balke, Davis~Rivera, Garcia-Cardona, Majee, McCann, Pfister, and Wohlberg]{balke2022scientific}
Thilo Balke, Fernando Davis~Rivera, Cristina Garcia-Cardona, Soumendu Majee, Michael~Thompson McCann, Luke Pfister, and Brendt~Egon Wohlberg.
\newblock Scientific computational imaging code (scico).
\newblock \emph{Journal of Open Source Software}, 7\penalty0 (LA-UR-22-28555), 2022.

\bibitem[Beck(2017)]{beck2017first}
Amir Beck.
\newblock \emph{First-order methods in optimization}.
\newblock SIAM, 2017.

\bibitem[Beck \& Teboulle(2009)Beck and Teboulle]{beck2009fast}
Amir Beck and Marc Teboulle.
\newblock A fast iterative shrinkage-thresholding algorithm for linear inverse problems.
\newblock \emph{SIAM journal on imaging sciences}, 2\penalty0 (1):\penalty0 183--202, 2009.

\bibitem[Benning \& Burger(2018)Benning and Burger]{benning2018modern}
Martin Benning and Martin Burger.
\newblock Modern regularization methods for inverse problems.
\newblock \emph{Acta numerica}, 27:\penalty0 1--111, 2018.

\bibitem[Bertero et~al.(2021)Bertero, Boccacci, and De~Mol]{bertero2021introduction}
Mario Bertero, Patrizia Boccacci, and Christine De~Mol.
\newblock \emph{Introduction to inverse problems in imaging}.
\newblock CRC press, 2021.

\bibitem[Bertsekas(2016)]{Bertsekas2016-db}
Dimitri~P Bertsekas.
\newblock \emph{Nonlinear Programming}.
\newblock Athena Scientific, 2016.
\newblock ISBN 9781886529052.
\newblock URL \url{https://market.android.com/details?id=book-TwOujgEACAAJ}.

\bibitem[Bo{\c t} et~al.(2016)Bo{\c t}, Csetnek, and L{\'a}szl{\'o}]{Bot2016-mw}
Radu~Ioan Bo{\c t}, Ern{\"o}~Robert Csetnek, and Szil{\'a}rd~Csaba L{\'a}szl{\'o}.
\newblock An inertial forward--backward algorithm for the minimization of the sum of two nonconvex functions.
\newblock \emph{EURO Journal on Computational Optimization}, 4\penalty0 (1):\penalty0 3--25, February 2016.
\newblock ISSN 2192-4414.
\newblock \doi{10.1007/s13675-015-0045-8}.
\newblock URL \url{https://doi.org/10.1007/s13675-015-0045-8}.

\bibitem[Boyd et~al.(2011)Boyd, Parikh, Chu, Peleato, Eckstein, et~al.]{boyd2011distributed}
Stephen Boyd, Neal Parikh, Eric Chu, Borja Peleato, Jonathan Eckstein, et~al.
\newblock Distributed optimization and statistical learning via the alternating direction method of multipliers.
\newblock \emph{Foundations and Trends{\textregistered} in Machine learning}, 3\penalty0 (1):\penalty0 1--122, 2011.

\bibitem[Boyd \& Vandenberghe(2004)Boyd and Vandenberghe]{boyd2004convex}
Stephen~P Boyd and Lieven Vandenberghe.
\newblock \emph{Convex optimization}.
\newblock Cambridge university press, 2004.

\bibitem[Bredies et~al.(2010)Bredies, Kunisch, and Pock]{bredies2010total}
Kristian Bredies, Karl Kunisch, and Thomas Pock.
\newblock Total generalized variation.
\newblock \emph{SIAM Journal on Imaging Sciences}, 3\penalty0 (3):\penalty0 492--526, 2010.

\bibitem[Bruckstein et~al.(2009)Bruckstein, Donoho, and Elad]{bruckstein2009sparse}
Alfred~M Bruckstein, David~L Donoho, and Michael Elad.
\newblock From sparse solutions of systems of equations to sparse modeling of signals and images.
\newblock \emph{SIAM review}, 51\penalty0 (1):\penalty0 34--81, 2009.

\bibitem[Chambolle \& Pock(2011)Chambolle and Pock]{chambolle2011first}
Antonin Chambolle and Thomas Pock.
\newblock A first-order primal-dual algorithm for convex problems with applications to imaging.
\newblock \emph{Journal of mathematical imaging and vision}, 40:\penalty0 120--145, 2011.

\bibitem[Chen et~al.(2021)Chen, Tachella, and Davies]{chen2021equivariant}
Dongdong Chen, Juli{\'a}n Tachella, and Mike~E Davies.
\newblock Equivariant imaging: Learning beyond the range space.
\newblock In \emph{Proceedings of the IEEE/CVF International Conference on Computer Vision}, pp.\  4379--4388, 2021.

\bibitem[Chen et~al.(2022{\natexlab{a}})Chen, Tachella, and Davies]{chen2022robust}
Dongdong Chen, Juli{\'a}n Tachella, and Mike~E Davies.
\newblock Robust equivariant imaging: a fully unsupervised framework for learning to image from noisy and partial measurements.
\newblock In \emph{Proceedings of the IEEE/CVF Conference on Computer Vision and Pattern Recognition}, pp.\  5647--5656, 2022{\natexlab{a}}.

\bibitem[Chen et~al.(2023{\natexlab{a}})Chen, Davies, Ehrhardt, Sch{\"o}nlieb, Sherry, and Tachella]{chen2023imaging}
Dongdong Chen, Mike Davies, Matthias~J Ehrhardt, Carola-Bibiane Sch{\"o}nlieb, Ferdia Sherry, and Juli{\'a}n Tachella.
\newblock Imaging with equivariant deep learning: From unrolled network design to fully unsupervised learning.
\newblock \emph{IEEE Signal Processing Magazine}, 40\penalty0 (1):\penalty0 134--147, 2023{\natexlab{a}}.

\bibitem[Chen et~al.(2023{\natexlab{b}})Chen, Lee, and Lu]{chen2023improved}
Hongrui Chen, Holden Lee, and Jianfeng Lu.
\newblock Improved analysis of score-based generative modeling: User-friendly bounds under minimal smoothness assumptions.
\newblock In \emph{International Conference on Machine Learning}, pp.\  4735--4763. PMLR, 2023{\natexlab{b}}.

\bibitem[Chen et~al.(2022{\natexlab{b}})Chen, Chen, Chen, Wang, Heaton, Liu, and Yin]{chen2022learning}
Tianlong Chen, Xiaohan Chen, Wuyang Chen, Zhangyang Wang, Howard Heaton, Jialin Liu, and Wotao Yin.
\newblock Learning to optimize: A primer and a benchmark.
\newblock \emph{The Journal of Machine Learning Research}, 23\penalty0 (1):\penalty0 8562--8620, 2022{\natexlab{b}}.

\bibitem[Chung et~al.(2022)Chung, Kim, Mccann, Klasky, and Ye]{chung2022diffusion}
Hyungjin Chung, Jeongsol Kim, Michael~T Mccann, Marc~L Klasky, and Jong~Chul Ye.
\newblock Diffusion posterior sampling for general noisy inverse problems.
\newblock \emph{arXiv preprint arXiv:2209.14687}, 2022.

\bibitem[Cohen et~al.(2021{\natexlab{a}})Cohen, Blau, Freedman, and Rivlin]{cohen2021has}
Regev Cohen, Yochai Blau, Daniel Freedman, and Ehud Rivlin.
\newblock It has potential: Gradient-driven denoisers for convergent solutions to inverse problems.
\newblock \emph{Advances in Neural Information Processing Systems}, 34:\penalty0 18152--18164, 2021{\natexlab{a}}.

\bibitem[Cohen et~al.(2021{\natexlab{b}})Cohen, Elad, and Milanfar]{Cohen2021-mp}
Regev Cohen, Michael Elad, and Peyman Milanfar.
\newblock Regularization by denoising via {Fixed-Point} projection ({RED-PRO}).
\newblock \emph{SIAM journal on imaging sciences}, 14\penalty0 (3):\penalty0 1374--1406, January 2021{\natexlab{b}}.
\newblock \doi{10.1137/20M1337168}.
\newblock URL \url{https://doi.org/10.1137/20M1337168}.

\bibitem[Combettes \& Pesquet(2007)Combettes and Pesquet]{combettes2007douglas}
Patrick~L Combettes and Jean-Christophe Pesquet.
\newblock A douglas--rachford splitting approach to nonsmooth convex variational signal recovery.
\newblock \emph{IEEE Journal of Selected Topics in Signal Processing}, 1\penalty0 (4):\penalty0 564--574, 2007.

\bibitem[Dabov et~al.(2007)Dabov, Foi, Katkovnik, and Egiazarian]{dabov2007image}
Kostadin Dabov, Alessandro Foi, Vladimir Katkovnik, and Karen Egiazarian.
\newblock Image denoising by sparse 3-d transform-domain collaborative filtering.
\newblock \emph{IEEE Transactions on image processing}, 16\penalty0 (8):\penalty0 2080--2095, 2007.

\bibitem[Darestani et~al.(2021)Darestani, Chaudhari, and Heckel]{darestani2021measuring}
Mohammad~Zalbagi Darestani, Akshay~S Chaudhari, and Reinhard Heckel.
\newblock Measuring robustness in deep learning based compressive sensing.
\newblock In \emph{International Conference on Machine Learning}, pp.\  2433--2444. PMLR, 2021.

\bibitem[Delbracio \& Milanfar(2023)Delbracio and Milanfar]{Delbracio2023-oc}
Mauricio Delbracio and Peyman Milanfar.
\newblock Inversion by direct iteration: An alternative to denoising diffusion for image restoration.
\newblock March 2023.
\newblock URL \url{http://arxiv.org/abs/2303.11435}.

\bibitem[Dold(2012)]{Dold2012-nj}
Albrecht Dold.
\newblock \emph{Lectures on Algebraic Topology}.
\newblock Springer Science \& Business Media, December 2012.
\newblock ISBN 9783642678219.
\newblock URL \url{https://play.google.com/store/books/details?id=P-xrCQAAQBAJ}.

\bibitem[Douglas \& Rachford(1956)Douglas and Rachford]{douglas1956numerical}
Jim Douglas and Henry~H Rachford.
\newblock On the numerical solution of heat conduction problems in two and three space variables.
\newblock \emph{Transactions of the American mathematical Society}, 82\penalty0 (2):\penalty0 421--439, 1956.

\bibitem[Elad \& Aharon(2006)Elad and Aharon]{elad2006image}
Michael Elad and Michal Aharon.
\newblock Image denoising via sparse and redundant representations over learned dictionaries.
\newblock \emph{IEEE Transactions on Image processing}, 15\penalty0 (12):\penalty0 3736--3745, 2006.

\bibitem[Engl et~al.(1996)Engl, Hanke, and Neubauer]{engl1996regularization}
Heinz~Werner Engl, Martin Hanke, and Andreas Neubauer.
\newblock \emph{Regularization of inverse problems}, volume 375.
\newblock Springer Science \& Business Media, 1996.

\bibitem[Fang et~al.(2023)Fang, Lai, van Zijl, Li, and Sulam]{fang2023deepsti}
Zhenghan Fang, Kuo-Wei Lai, Peter van Zijl, Xu~Li, and Jeremias Sulam.
\newblock Deepsti: Towards tensor reconstruction using fewer orientations in susceptibility tensor imaging.
\newblock \emph{Medical image analysis}, 87:\penalty0 102829, 2023.

\bibitem[Feng et~al.(2023)Feng, Smith, Rubinstein, Chang, Bouman, and Freeman]{feng2023score}
Berthy~T Feng, Jamie Smith, Michael Rubinstein, Huiwen Chang, Katherine~L Bouman, and William~T Freeman.
\newblock Score-based diffusion models as principled priors for inverse imaging.
\newblock \emph{arXiv preprint arXiv:2304.11751}, 2023.

\bibitem[Feng et~al.(2020)Feng, Fan, and Suykens]{Feng2020-fw}
Yunlong Feng, Jun Fan, and Johan A~K Suykens.
\newblock A statistical learning approach to modal regression.
\newblock \emph{Journal of machine learning research: JMLR}, 21\penalty0 (2):\penalty0 1--35, 2020.
\newblock ISSN 1532-4435, 1533-7928.
\newblock URL \url{https://jmlr.org/papers/v21/17-068.html}.

\bibitem[Gao et~al.(2023)Gao, Leong, Sun, and Bouman]{gao2023image}
Angela~F Gao, Oscar Leong, He~Sun, and Katherine~L Bouman.
\newblock Image reconstruction without explicit priors.
\newblock In \emph{ICASSP 2023-2023 IEEE International Conference on Acoustics, Speech and Signal Processing (ICASSP)}, pp.\  1--5. IEEE, 2023.

\bibitem[Geman \& Yang(1995)Geman and Yang]{geman1995nonlinear}
Donald Geman and Chengda Yang.
\newblock Nonlinear image recovery with half-quadratic regularization.
\newblock \emph{IEEE transactions on Image Processing}, 4\penalty0 (7):\penalty0 932--946, 1995.

\bibitem[Gilton et~al.(2019)Gilton, Ongie, and Willett]{gilton2019neumann}
Davis Gilton, Greg Ongie, and Rebecca Willett.
\newblock Neumann networks for linear inverse problems in imaging.
\newblock \emph{IEEE Transactions on Computational Imaging}, 6:\penalty0 328--343, 2019.

\bibitem[Gilton et~al.(2021)Gilton, Ongie, and Willett]{gilton2021deep}
Davis Gilton, Gregory Ongie, and Rebecca Willett.
\newblock Deep equilibrium architectures for inverse problems in imaging.
\newblock \emph{IEEE Transactions on Computational Imaging}, 7:\penalty0 1123--1133, 2021.

\bibitem[Gneiting(2011)]{Gneiting2011-hf}
Tilmann Gneiting.
\newblock Making and evaluating point forecasts.
\newblock \emph{Journal of the American Statistical Association}, 106\penalty0 (494):\penalty0 746--762, 2011.
\newblock ISSN 0162-1459.
\newblock URL \url{http://www.jstor.org/stable/41416407}.

\bibitem[Goujon et~al.(2023)Goujon, Neumayer, and Unser]{Goujon2023-wa}
Alexis Goujon, Sebastian Neumayer, and Michael Unser.
\newblock Learning weakly convex regularizers for convergent {Image-Reconstruction} algorithms.
\newblock August 2023.
\newblock URL \url{http://arxiv.org/abs/2308.10542}.

\bibitem[Gregor \& LeCun(2010)Gregor and LeCun]{gregor2010learning}
Karol Gregor and Yann LeCun.
\newblock Learning fast approximations of sparse coding.
\newblock In \emph{Proceedings of the 27th international conference on international conference on machine learning}, pp.\  399--406, 2010.

\bibitem[Gribonval(2011)]{Gribonval2011-pf}
R{\'e}mi Gribonval.
\newblock Should penalized least squares regression be interpreted as maximum a posteriori estimation?
\newblock \emph{IEEE transactions on signal processing: a publication of the IEEE Signal Processing Society}, 59\penalty0 (5):\penalty0 2405--2410, May 2011.
\newblock ISSN 1053-587X, 1941-0476.
\newblock \doi{10.1109/TSP.2011.2107908}.
\newblock URL \url{http://dx.doi.org/10.1109/TSP.2011.2107908}.

\bibitem[Gribonval \& Nikolova(2020)Gribonval and Nikolova]{gribonval2020characterization}
R{\'e}mi Gribonval and Mila Nikolova.
\newblock A characterization of proximity operators.
\newblock \emph{Journal of Mathematical Imaging and Vision}, 62\penalty0 (6-7):\penalty0 773--789, 2020.

\bibitem[Hauptmann et~al.(2023)Hauptmann, Mukherjee, Sch{\"o}nlieb, and Sherry]{Hauptmann2023-nu}
Andreas Hauptmann, Subhadip Mukherjee, Carola-Bibiane Sch{\"o}nlieb, and Ferdia Sherry.
\newblock Convergent regularization in inverse problems and linear plug-and-play denoisers.
\newblock July 2023.
\newblock URL \url{http://arxiv.org/abs/2307.09441}.

\bibitem[Heinrich(2014)]{Heinrich2014-zo}
C~Heinrich.
\newblock The mode functional is not elicitable.
\newblock \emph{Biometrika}, 101\penalty0 (1):\penalty0 245--251, 2014.
\newblock ISSN 0006-3444.
\newblock URL \url{http://www.jstor.org/stable/43305608}.

\bibitem[Huang et~al.(2021)Huang, Chen, Tsirigotis, and Courville]{Huang2021-ds}
Chin-Wei Huang, Ricky T~Q Chen, Christos Tsirigotis, and Aaron Courville.
\newblock Convex potential flows: Universal probability distributions with optimal transport and convex optimization.
\newblock In \emph{International Conference on Learning Representations}, 2021.
\newblock URL \url{https://openreview.net/forum?id=te7PVH1sPxJ}.

\bibitem[Hurault et~al.(2022{\natexlab{a}})Hurault, Leclaire, and Papadakis]{hurault2022gradient}
Samuel Hurault, Arthur Leclaire, and Nicolas Papadakis.
\newblock Gradient step denoiser for convergent plug-and-play.
\newblock In \emph{International Conference on Learning Representations}, 2022{\natexlab{a}}.

\bibitem[Hurault et~al.(2022{\natexlab{b}})Hurault, Leclaire, and Papadakis]{hurault2022proximal}
Samuel Hurault, Arthur Leclaire, and Nicolas Papadakis.
\newblock Proximal denoiser for convergent plug-and-play optimization with nonconvex regularization.
\newblock In \emph{International Conference on Machine Learning}, pp.\  9483--9505. PMLR, 2022{\natexlab{b}}.

\bibitem[Jalal et~al.(2021{\natexlab{a}})Jalal, Arvinte, Daras, Price, Dimakis, and Tamir]{jalal2021robust}
Ajil Jalal, Marius Arvinte, Giannis Daras, Eric Price, Alexandros~G Dimakis, and Jon Tamir.
\newblock Robust compressed sensing mri with deep generative priors.
\newblock \emph{Advances in Neural Information Processing Systems}, 34:\penalty0 14938--14954, 2021{\natexlab{a}}.

\bibitem[Jalal et~al.(2021{\natexlab{b}})Jalal, Karmalkar, Dimakis, and Price]{jalal2021instance}
Ajil Jalal, Sushrut Karmalkar, Alexandros~G Dimakis, and Eric Price.
\newblock Instance-optimal compressed sensing via posterior sampling.
\newblock \emph{arXiv preprint arXiv:2106.11438}, 2021{\natexlab{b}}.

\bibitem[Ji \& Telgarsky(2020)Ji and Telgarsky]{Ji2020-mg}
Ziwei Ji and Matus Telgarsky.
\newblock Directional convergence and alignment in deep learning.
\newblock In H~Larochelle, M~Ranzato, R~Hadsell, M~F Balcan, and H~Lin (eds.), \emph{Advances in Neural Information Processing Systems}, volume~33, pp.\  17176--17186. Curran Associates, Inc., 2020.
\newblock URL \url{https://proceedings.neurips.cc/paper_files/paper/2020/file/c76e4b2fa54f8506719a5c0dc14c2eb9-Paper.pdf}.

\bibitem[Kadkhodaie \& Simoncelli(2021)Kadkhodaie and Simoncelli]{Kadkhodaie2021-kh}
Zahra Kadkhodaie and Eero Simoncelli.
\newblock Stochastic solutions for linear inverse problems using the prior implicit in a denoiser.
\newblock \emph{Adv. Neural Inf. Process. Syst.}, 34:\penalty0 13242--13254, 2021.

\bibitem[Kamilov et~al.(2023{\natexlab{a}})Kamilov, Bouman, Buzzard, and Wohlberg]{Kamilov2023-uq}
Ulugbek~S Kamilov, Charles~A Bouman, Gregery~T Buzzard, and Brendt Wohlberg.
\newblock {Plug-and-Play} methods for integrating physical and learned models in computational imaging: Theory, algorithms, and applications.
\newblock \emph{IEEE Signal Processing Magazine}, 40\penalty0 (1):\penalty0 85--97, January 2023{\natexlab{a}}.
\newblock ISSN 1558-0792.
\newblock \doi{10.1109/MSP.2022.3199595}.
\newblock URL \url{http://dx.doi.org/10.1109/MSP.2022.3199595}.

\bibitem[Kamilov et~al.(2023{\natexlab{b}})Kamilov, Bouman, Buzzard, and Wohlberg]{kamilov2023plug}
Ulugbek~S Kamilov, Charles~A Bouman, Gregery~T Buzzard, and Brendt Wohlberg.
\newblock Plug-and-play methods for integrating physical and learned models in computational imaging: Theory, algorithms, and applications.
\newblock \emph{IEEE Signal Processing Magazine}, 40\penalty0 (1):\penalty0 85--97, 2023{\natexlab{b}}.

\bibitem[Kawar et~al.(2021)Kawar, Vaksman, and Elad]{Kawar2021-eq}
Bahjat Kawar, Gregory Vaksman, and Michael Elad.
\newblock {SNIPS}: Solving noisy inverse problems stochastically.
\newblock May 2021.

\bibitem[Kawar et~al.(2022)Kawar, Elad, Ermon, and Song]{Kawar2022-hu}
Bahjat Kawar, Michael Elad, Stefano Ermon, and Jiaming Song.
\newblock Denoising diffusion restoration models.
\newblock January 2022.

\bibitem[Kingma \& Ba(2014)Kingma and Ba]{kingma2014adam}
Diederik~P Kingma and Jimmy Ba.
\newblock Adam: A method for stochastic optimization.
\newblock \emph{arXiv preprint arXiv:1412.6980}, 2014.

\bibitem[Kobler et~al.(2017)Kobler, Klatzer, Hammernik, and Pock]{Kobler2017-wr}
Erich Kobler, Teresa Klatzer, Kerstin Hammernik, and Thomas Pock.
\newblock Variational networks: Connecting variational methods and deep learning.
\newblock In \emph{Pattern Recognition}, Lecture Notes in Computer Science, pp.\  281--293. Springer, Cham, September 2017.
\newblock ISBN 9783319667089, 9783319667096.
\newblock \doi{10.1007/978-3-319-66709-6\_23}.
\newblock URL \url{https://link.springer.com/chapter/10.1007/978-3-319-66709-6_23}.

\bibitem[Kobler et~al.(2020)Kobler, Effland, Kunisch, and Pock]{kobler2020total}
Erich Kobler, Alexander Effland, Karl Kunisch, and Thomas Pock.
\newblock Total deep variation for linear inverse problems.
\newblock In \emph{Proceedings of the IEEE/CVF Conference on computer vision and pattern recognition}, pp.\  7549--7558, 2020.

\bibitem[Lai et~al.(2020)Lai, Aggarwal, van Zijl, Li, and Sulam]{lai2020learned}
Kuo-Wei Lai, Manisha Aggarwal, Peter van Zijl, Xu~Li, and Jeremias Sulam.
\newblock Learned proximal networks for quantitative susceptibility mapping.
\newblock In \emph{Medical Image Computing and Computer Assisted Intervention--MICCAI 2020: 23rd International Conference, Lima, Peru, October 4--8, 2020, Proceedings, Part II 23}, pp.\  125--135. Springer, 2020.

\bibitem[LeCun(1998)]{lecun1998mnist}
Yann LeCun.
\newblock The mnist database of handwritten digits.
\newblock \emph{http://yann. lecun. com/exdb/mnist/}, 1998.

\bibitem[Lehtinen et~al.(2018)Lehtinen, Munkberg, Hasselgren, Laine, Karras, Aittala, and Aila]{lehtinen2018noise2noisea}
Jaakko Lehtinen, Jacob Munkberg, Jon Hasselgren, Samuli Laine, Tero Karras, Miika Aittala, and Timo Aila.
\newblock {N}oise2{N}oise: Learning image restoration without clean data.
\newblock In Jennifer Dy and Andreas Krause (eds.), \emph{Proceedings of the 35th International Conference on Machine Learning}, volume~80 of \emph{Proceedings of Machine Learning Research}, pp.\  2965--2974. PMLR, 10--15 Jul 2018.
\newblock URL \url{https://proceedings.mlr.press/v80/lehtinen18a.html}.

\bibitem[Li \& Pong(2016)Li and Pong]{Li2016-ty}
Guoyin Li and Ting~Kei Pong.
\newblock {Douglas--Rachford} splitting for nonconvex optimization with application to nonconvex feasibility problems.
\newblock \emph{Mathematical Programming. A Publication of the Mathematical Programming Society}, 159\penalty0 (1):\penalty0 371--401, September 2016.
\newblock ISSN 0025-5610, 1436-4646.
\newblock \doi{10.1007/s10107-015-0963-5}.
\newblock URL \url{https://doi.org/10.1007/s10107-015-0963-5}.

\bibitem[Li et~al.(2020)Li, Schwab, Antholzer, and Haltmeier]{li2020nett}
Housen Li, Johannes Schwab, Stephan Antholzer, and Markus Haltmeier.
\newblock Nett: Solving inverse problems with deep neural networks.
\newblock \emph{Inverse Problems}, 36\penalty0 (6):\penalty0 065005, 2020.

\bibitem[Lions \& Mercier(1979)Lions and Mercier]{lions1979splitting}
Pierre-Louis Lions and Bertrand Mercier.
\newblock Splitting algorithms for the sum of two nonlinear operators.
\newblock \emph{SIAM Journal on Numerical Analysis}, 16\penalty0 (6):\penalty0 964--979, 1979.

\bibitem[Liu et~al.(2019)Liu, Chen, Wang, and Yin]{liu2018alista}
Jialin Liu, Xiaohan Chen, Zhangyang Wang, and Wotao Yin.
\newblock {ALISTA}: Analytic weights are as good as learned weights in {LISTA}.
\newblock In \emph{International Conference on Learning Representations}, 2019.
\newblock URL \url{https://openreview.net/forum?id=B1lnzn0ctQ}.

\bibitem[Liu et~al.(2022)Liu, Xu, Gan, Kamilov, et~al.]{liu2022online}
Jiaming Liu, Xiaojian Xu, Weijie Gan, Ulugbek Kamilov, et~al.
\newblock Online deep equilibrium learning for regularization by denoising.
\newblock \emph{Advances in Neural Information Processing Systems}, 35:\penalty0 25363--25376, 2022.

\bibitem[Liu et~al.(2018)Liu, Luo, Wang, and Tang]{liu2018large}
Ziwei Liu, Ping Luo, Xiaogang Wang, and Xiaoou Tang.
\newblock Large-scale celebfaces attributes (celeba) dataset.
\newblock \emph{Retrieved August}, 15\penalty0 (2018):\penalty0 11, 2018.

\bibitem[Loi(2010)]{Loi2010-dx}
Ta~L{\^e} Loi.
\newblock Lecture 1: O-minimal structures.
\newblock In \emph{The {Japanese-Australian} Workshop on Real and Complex Singularities: {JARCS} {III}}, volume~43, pp.\  19--31. Australian National University, Mathematical Sciences Institute, January 2010.
\newblock URL \url{https://projecteuclid.org/ebooks/proceedings-of-the-centre-for-mathematics-and-its-applications/The-Japanese-Australian-Workshop-on-Real-and-Complex-Singularities/chapter/Lecture-1-O-minimal-Structures/pcma/1416320994}.

\bibitem[Lojasiewicz(1963)]{lojasiewicz1963propriete}
Stanislaw Lojasiewicz.
\newblock Une propri{\'e}t{\'e} topologique des sous-ensembles analytiques r{\'e}els.
\newblock \emph{Les {\'e}quations aux d{\'e}riv{\'e}es partielles}, 117:\penalty0 87--89, 1963.

\bibitem[Lunz et~al.(2018)Lunz, {\"O}ktem, and Sch{\"o}nlieb]{lunz2018adversarial}
Sebastian Lunz, Ozan {\"O}ktem, and Carola-Bibiane Sch{\"o}nlieb.
\newblock Adversarial regularizers in inverse problems.
\newblock \emph{Advances in neural information processing systems}, 31, 2018.

\bibitem[Makkuva et~al.(2020)Makkuva, Taghvaei, Oh, and Lee]{makkuva2020optimal}
Ashok Makkuva, Amirhossein Taghvaei, Sewoong Oh, and Jason Lee.
\newblock Optimal transport mapping via input convex neural networks.
\newblock In \emph{International Conference on Machine Learning}, pp.\  6672--6681. PMLR, 2020.

\bibitem[Mallat(1999)]{mallat1999wavelet}
St{\'e}phane Mallat.
\newblock \emph{A wavelet tour of signal processing}.
\newblock Elsevier, 1999.

\bibitem[Mardani et~al.(2018)Mardani, Sun, Donoho, Papyan, Monajemi, Vasanawala, and Pauly]{Mardani2018-rc}
Morteza Mardani, Qingyun Sun, David Donoho, Vardan Papyan, Hatef Monajemi, Shreyas Vasanawala, and John Pauly.
\newblock Neural proximal gradient descent for compressive imaging.
\newblock \emph{Advances in Neural Information Processing Systems}, 31, 2018.

\bibitem[McCann et~al.(2017)McCann, Jin, and Unser]{mccann2017convolutional}
Michael~T McCann, Kyong~Hwan Jin, and Michael Unser.
\newblock Convolutional neural networks for inverse problems in imaging: A review.
\newblock \emph{IEEE Signal Processing Magazine}, 34\penalty0 (6):\penalty0 85--95, 2017.

\bibitem[McCollough(2016)]{mccollough2016tu}
C~McCollough.
\newblock Tu-fg-207a-04: overview of the low dose ct grand challenge.
\newblock \emph{Medical physics}, 43\penalty0 (6Part35):\penalty0 3759--3760, 2016.

\bibitem[Meinhardt et~al.(2017)Meinhardt, Moller, Hazirbas, and Cremers]{meinhardt2017learning}
Tim Meinhardt, Michael Moller, Caner Hazirbas, and Daniel Cremers.
\newblock Learning proximal operators: Using denoising networks for regularizing inverse imaging problems.
\newblock In \emph{Proceedings of the IEEE International Conference on Computer Vision}, pp.\  1781--1790, 2017.

\bibitem[Monga et~al.(2021)Monga, Li, and Eldar]{Monga2021-hb}
Vishal Monga, Yuelong Li, and Yonina~C Eldar.
\newblock Algorithm unrolling: Interpretable, efficient deep learning for signal and image processing.
\newblock \emph{IEEE Signal Processing Magazine}, 38\penalty0 (2):\penalty0 18--44, March 2021.
\newblock ISSN 1558-0792.
\newblock \doi{10.1109/MSP.2020.3016905}.
\newblock URL \url{http://dx.doi.org/10.1109/MSP.2020.3016905}.

\bibitem[Moreau(1965)]{moreau1965proximite}
Jean-Jacques Moreau.
\newblock Proximit{\'e} et dualit{\'e} dans un espace hilbertien.
\newblock \emph{Bulletin de la Soci{\'e}t{\'e} math{\'e}matique de France}, 93:\penalty0 273--299, 1965.

\bibitem[Mukherjee et~al.(2020)Mukherjee, Dittmer, Shumaylov, Lunz, {\"O}ktem, and Sch{\"o}nlieb]{mukherjee2020learned}
Subhadip Mukherjee, S{\"o}ren Dittmer, Zakhar Shumaylov, Sebastian Lunz, Ozan {\"O}ktem, and Carola-Bibiane Sch{\"o}nlieb.
\newblock Learned convex regularizers for inverse problems.
\newblock \emph{arXiv preprint arXiv:2008.02839}, 2020.

\bibitem[Mukherjee et~al.(2021)Mukherjee, Carioni, {\"O}ktem, and Sch{\"o}nlieb]{mukherjee2021end}
Subhadip Mukherjee, Marcello Carioni, Ozan {\"O}ktem, and Carola-Bibiane Sch{\"o}nlieb.
\newblock End-to-end reconstruction meets data-driven regularization for inverse problems.
\newblock \emph{Advances in Neural Information Processing Systems}, 34:\penalty0 21413--21425, 2021.

\bibitem[Ongie et~al.(2020)Ongie, Jalal, Metzler, Baraniuk, Dimakis, and Willett]{ongie2020deep}
Gregory Ongie, Ajil Jalal, Christopher~A Metzler, Richard~G Baraniuk, Alexandros~G Dimakis, and Rebecca Willett.
\newblock Deep learning techniques for inverse problems in imaging.
\newblock \emph{IEEE Journal on Selected Areas in Information Theory}, 1\penalty0 (1):\penalty0 39--56, 2020.

\bibitem[Richter-Powell et~al.(2021)Richter-Powell, Lorraine, and Amos]{richter2021input}
Jack Richter-Powell, Jonathan Lorraine, and Brandon Amos.
\newblock Input convex gradient networks.
\newblock \emph{arXiv preprint arXiv:2111.12187}, 2021.

\bibitem[Rockafellar \& Wets(1998)Rockafellar and Wets]{Tyrrell_Rockafellar1998-ol}
R~Tyrell Rockafellar and Roger J-B Wets.
\newblock \emph{Variational Analysis}.
\newblock Grundlehren der mathematischen Wissenschaften. Springer-Verlag Berlin Heidelberg, 1 edition, 1998.
\newblock ISBN 9783642024313, 9783540627722.
\newblock \doi{10.1007/978-3-642-02431-3}.
\newblock URL \url{https://www.springer.com/us/book/9783540627722}.

\bibitem[Romano \& Elad(2015)Romano and Elad]{romano2015boosting}
Yaniv Romano and Michael Elad.
\newblock Boosting of image denoising algorithms.
\newblock \emph{SIAM Journal on Imaging Sciences}, 8\penalty0 (2):\penalty0 1187--1219, 2015.

\bibitem[Romano et~al.(2017)Romano, Elad, and Milanfar]{romano2017little}
Yaniv Romano, Michael Elad, and Peyman Milanfar.
\newblock The little engine that could: Regularization by denoising (red).
\newblock \emph{SIAM Journal on Imaging Sciences}, 10\penalty0 (4):\penalty0 1804--1844, 2017.

\bibitem[Rudin et~al.(1992)Rudin, Osher, and Fatemi]{rudin1992nonlinear}
Leonid~I Rudin, Stanley Osher, and Emad Fatemi.
\newblock Nonlinear total variation based noise removal algorithms.
\newblock \emph{Physica D: nonlinear phenomena}, 60\penalty0 (1-4):\penalty0 259--268, 1992.

\bibitem[Rudin(1976)]{Rudin1976-ij}
Walter Rudin.
\newblock \emph{Principles of mathematical analysis}.
\newblock McGraw-Hill, New York, 3 edition, 1976.
\newblock ISBN 9780070542358.
\newblock URL \url{https://openlibrary.org/books/OL5195991M.opds}.

\bibitem[Ryu et~al.(2019)Ryu, Liu, Wang, Chen, Wang, and Yin]{Ryu2019-ca}
Ernest Ryu, Jialin Liu, Sicheng Wang, Xiaohan Chen, Zhangyang Wang, and Wotao Yin.
\newblock {Plug-and-Play} methods provably converge with properly trained denoisers.
\newblock In Kamalika Chaudhuri and Ruslan Salakhutdinov (eds.), \emph{Proceedings of the 36th International Conference on Machine Learning}, volume~97 of \emph{Proceedings of Machine Learning Research}, pp.\  5546--5557. PMLR, 2019.
\newblock URL \url{http://proceedings.mlr.press/v97/ryu19a.html}.

\bibitem[Shenoy et~al.(2023)Shenoy, Marks, Mansour, and Lohit]{shenoy2023unrolled}
Vineet~R Shenoy, Tim~K Marks, Hassan Mansour, and Suhas Lohit.
\newblock Unrolled ippg: Video heart rate estimation via unrolling proximal gradient descent.
\newblock In \emph{2023 IEEE International Conference on Image Processing (ICIP)}, pp.\  2715--2719. IEEE, 2023.

\bibitem[Sreehari et~al.(2016)Sreehari, Venkatakrishnan, Wohlberg, Buzzard, Drummy, Simmons, and Bouman]{Sreehari2016-in}
Suhas Sreehari, S~V Venkatakrishnan, Brendt Wohlberg, Gregery~T Buzzard, Lawrence~F Drummy, Jeffrey~P Simmons, and Charles~A Bouman.
\newblock {Plug-and-Play} priors for bright field electron tomography and sparse interpolation.
\newblock \emph{IEEE Transactions on Computational Imaging}, 2\penalty0 (4):\penalty0 408--423, December 2016.
\newblock ISSN 2333-9403.
\newblock \doi{10.1109/TCI.2016.2599778}.
\newblock URL \url{http://dx.doi.org/10.1109/TCI.2016.2599778}.

\bibitem[Stein \& Shakarchi(2005)Stein and Shakarchi]{Stein2005-io}
Elias~M Stein and Rami Shakarchi.
\newblock \emph{Real Analysis: Measure Theory, Integration, and Hilbert Spaces}.
\newblock Princeton University Press, April 2005.
\newblock ISBN 9780691113869.
\newblock URL \url{https://play.google.com/store/books/details?id=DLumDwAAQBAJ}.

\bibitem[Sulam et~al.(2014)Sulam, Ophir, and Elad]{sulam2014image}
Jeremias Sulam, Boaz Ophir, and Michael Elad.
\newblock Image denoising through multi-scale learnt dictionaries.
\newblock In \emph{2014 IEEE International Conference on Image Processing (ICIP)}, pp.\  808--812. IEEE, 2014.

\bibitem[Sulam et~al.(2019)Sulam, Aberdam, Beck, and Elad]{sulam2019multi}
Jeremias Sulam, Aviad Aberdam, Amir Beck, and Michael Elad.
\newblock On multi-layer basis pursuit, efficient algorithms and convolutional neural networks.
\newblock \emph{IEEE transactions on pattern analysis and machine intelligence}, 42\penalty0 (8):\penalty0 1968--1980, 2019.

\bibitem[Sulam et~al.(2020)Sulam, Muthukumar, and Arora]{sulam2020adversarial}
Jeremias Sulam, Ramchandran Muthukumar, and Raman Arora.
\newblock Adversarial robustness of supervised sparse coding.
\newblock \emph{Advances in neural information processing systems}, 33:\penalty0 2110--2121, 2020.

\bibitem[Sun \& Bouman(2021)Sun and Bouman]{sun2021deep}
He~Sun and Katherine~L Bouman.
\newblock Deep probabilistic imaging: Uncertainty quantification and multi-modal solution characterization for computational imaging.
\newblock In \emph{Proceedings of the AAAI Conference on Artificial Intelligence}, volume~35, pp.\  2628--2637, 2021.

\bibitem[Sun et~al.(2019)Sun, Wohlberg, and Kamilov]{Sun2019-zc}
Yu~Sun, Brendt Wohlberg, and Ulugbek~S Kamilov.
\newblock An online {Plug-and-Play} algorithm for regularized image reconstruction.
\newblock \emph{IEEE Transactions on Computational Imaging}, 5\penalty0 (3):\penalty0 395--408, September 2019.
\newblock ISSN 2333-9403.
\newblock \doi{10.1109/TCI.2019.2893568}.
\newblock URL \url{http://dx.doi.org/10.1109/TCI.2019.2893568}.

\bibitem[Sun et~al.(2021)Sun, Wu, Xu, Wohlberg, and Kamilov]{Sun2021-ll}
Yu~Sun, Zihui Wu, Xiaojian Xu, Brendt Wohlberg, and Ulugbek~S Kamilov.
\newblock Scalable {Plug-and-Play} {ADMM} with convergence guarantees.
\newblock \emph{IEEE Transactions on Computational Imaging}, 7:\penalty0 849--863, 2021.
\newblock ISSN 2333-9403.
\newblock \doi{10.1109/TCI.2021.3094062}.
\newblock URL \url{http://dx.doi.org/10.1109/TCI.2021.3094062}.

\bibitem[Tachella et~al.(2019)Tachella, Altmann, Mellado, McCarthy, Tobin, Buller, Tourneret, and McLaughlin]{tachella2019real}
Juli{\'a}n Tachella, Yoann Altmann, Nicolas Mellado, Aongus McCarthy, Rachael Tobin, Gerald~S Buller, Jean-Yves Tourneret, and Stephen McLaughlin.
\newblock Real-time 3d reconstruction from single-photon lidar data using plug-and-play point cloud denoisers.
\newblock \emph{Nature communications}, 10\penalty0 (1):\penalty0 4984, 2019.

\bibitem[Tachella et~al.(2022)Tachella, Chen, and Davies]{tachella2022unsupervised}
Juli{\'a}n Tachella, Dongdong Chen, and Mike Davies.
\newblock Unsupervised learning from incomplete measurements for inverse problems.
\newblock \emph{Advances in Neural Information Processing Systems}, 35:\penalty0 4983--4995, 2022.

\bibitem[Tachella et~al.(2023)Tachella, Chen, and Davies]{tachella2023sensing}
Juli{\'a}n Tachella, Dongdong Chen, and Mike Davies.
\newblock Sensing theorems for unsupervised learning in linear inverse problems.
\newblock \emph{Journal of Machine Learning Research}, 24\penalty0 (39):\penalty0 1--45, 2023.

\bibitem[Tan et~al.(2023)Tan, Mukherjee, Tang, and Sch{\"o}nlieb]{Tan2023-gd}
Hong~Ye Tan, Subhadip Mukherjee, Junqi Tang, and Carola-Bibiane Sch{\"o}nlieb.
\newblock Provably convergent {Plug-and-Play} {Quasi-Newton} methods.
\newblock March 2023.
\newblock URL \url{http://arxiv.org/abs/2303.07271}.

\bibitem[Teneggi et~al.(2023)Teneggi, Tivnan, Stayman, and Sulam]{teneggi2023trust}
Jacopo Teneggi, Matthew Tivnan, Web Stayman, and Jeremias Sulam.
\newblock How to trust your diffusion model: A convex optimization approach to conformal risk control.
\newblock In \emph{International Conference on Machine Learning}, pp.\  33940--33960. PMLR, 2023.

\bibitem[Teodoro et~al.(2018)Teodoro, Bioucas-Dias, and Figueiredo]{Teodoro2018-ly}
Afonso~M Teodoro, Jose~M Bioucas-Dias, and Mario A~T Figueiredo.
\newblock A convergent image fusion algorithm using {Scene-Adapted} {Gaussian-Mixture-Based} denoising.
\newblock \emph{IEEE transactions on image processing: a publication of the IEEE Signal Processing Society}, September 2018.
\newblock ISSN 1057-7149, 1941-0042.
\newblock \doi{10.1109/TIP.2018.2869727}.
\newblock URL \url{http://dx.doi.org/10.1109/TIP.2018.2869727}.

\bibitem[Themelis \& Patrinos(2020)Themelis and Patrinos]{Themelis2020-jj}
Andreas Themelis and Panagiotis Patrinos.
\newblock {Douglas--Rachford} splitting and {ADMM} for nonconvex optimization: Tight convergence results.
\newblock \emph{SIAM journal on optimization: a publication of the Society for Industrial and Applied Mathematics}, 30\penalty0 (1):\penalty0 149--181, January 2020.
\newblock ISSN 1052-6234.
\newblock \doi{10.1137/18M1163993}.
\newblock URL \url{https://doi.org/10.1137/18M1163993}.

\bibitem[Tian et~al.(2020)Tian, Fei, Zheng, Xu, Zuo, and Lin]{tian2020deep}
Chunwei Tian, Lunke Fei, Wenxian Zheng, Yong Xu, Wangmeng Zuo, and Chia-Wen Lin.
\newblock Deep learning on image denoising: An overview.
\newblock \emph{Neural Networks}, 131:\penalty0 251--275, 2020.

\bibitem[Tikhonov \& Arsenin(1977)Tikhonov and Arsenin]{tikhonov1977solutions}
Andrey~N Tikhonov and Vasiliy~Y Arsenin.
\newblock Solutions of ill-posed problems. vh winston \& sons, 1977.

\bibitem[Tolooshams et~al.(2023)Tolooshams, Mulleti, Ba, and Eldar]{tolooshams2023unrolled}
Bahareh Tolooshams, Satish Mulleti, Demba Ba, and Yonina~C Eldar.
\newblock Unrolled compressed blind-deconvolution.
\newblock \emph{IEEE Transactions on Signal Processing}, 2023.

\bibitem[Van~den Dries(1998)]{Van_den_Dries1998-ru}
Lou Van~den Dries.
\newblock \emph{Tame Topology and O-minimal Structures}.
\newblock Cambridge University Press, May 1998.
\newblock ISBN 9780521598385.
\newblock \doi{10.1017/CBO9780511525919}.
\newblock URL \url{https://play.google.com/store/books/details?id=CLnElinpjOgC}.

\bibitem[van~den Dries \& Miller(1994)van~den Dries and Miller]{Van_den_Dries1994-yz}
Lou van~den Dries and Chris Miller.
\newblock On the real exponential field with restricted analytic functions.
\newblock \emph{Israel Journal of Mathematics}, 85\penalty0 (1):\penalty0 19--56, February 1994.
\newblock ISSN 0021-2172, 1565-8511.
\newblock \doi{10.1007/BF02758635}.
\newblock URL \url{https://doi.org/10.1007/BF02758635}.

\bibitem[van~den Dries \& Miller(1996)van~den Dries and Miller]{Van_den_Dries1996-oz}
Lou van~den Dries and Chris Miller.
\newblock Geometric categories and o-minimal structures.
\newblock \emph{Duke Mathematical Journal}, 84\penalty0 (2):\penalty0 497--540, August 1996.
\newblock ISSN 0012-7094, 1547-7398.
\newblock \doi{10.1215/S0012-7094-96-08416-1}.
\newblock URL \url{https://projecteuclid.org/journals/duke-mathematical-journal/volume-84/issue-2/Geometric-categories-and-o-minimal-structures/10.1215/S0012-7094-96-08416-1.full}.

\bibitem[Venkatakrishnan et~al.(2013)Venkatakrishnan, Bouman, and Wohlberg]{venkatakrishnan2013plug}
Singanallur~V Venkatakrishnan, Charles~A Bouman, and Brendt Wohlberg.
\newblock Plug-and-play priors for model based reconstruction.
\newblock In \emph{2013 IEEE Global Conference on Signal and Information Processing}, pp.\  945--948. IEEE, 2013.

\bibitem[Willemink \& No{\"e}l(2019)Willemink and No{\"e}l]{willemink2019evolution}
Martin~J Willemink and Peter~B No{\"e}l.
\newblock The evolution of image reconstruction for ct—from filtered back projection to artificial intelligence.
\newblock \emph{European radiology}, 29:\penalty0 2185--2195, 2019.

\bibitem[Xu et~al.(2020)Xu, Sun, Liu, Wohlberg, and Kamilov]{Xu2020-my}
Xiaojian Xu, Yu~Sun, Jiaming Liu, Brendt Wohlberg, and Ulugbek~S Kamilov.
\newblock Provable convergence of {Plug-and-Play} priors with {MMSE} denoisers.
\newblock \emph{IEEE Signal Processing Letters}, 27:\penalty0 1280--1284, 2020.
\newblock ISSN 1558-2361.
\newblock \doi{10.1109/LSP.2020.3006390}.
\newblock URL \url{http://dx.doi.org/10.1109/LSP.2020.3006390}.

\bibitem[Yan \& Yin(2016)Yan and Yin]{Yan2016-fy}
Ming Yan and Wotao Yin.
\newblock Self equivalence of the alternating direction method of multipliers.
\newblock In Roland Glowinski, Stanley~J Osher, and Wotao Yin (eds.), \emph{Splitting Methods in Communication, Imaging, Science, and Engineering}, pp.\  165--194. Springer International Publishing, Cham, 2016.
\newblock ISBN 9783319415895.
\newblock \doi{10.1007/978-3-319-41589-5\_5}.
\newblock URL \url{https://doi.org/10.1007/978-3-319-41589-5_5}.

\bibitem[Yu et~al.(2019)Yu, Park, and Jeong]{yu2019deep}
Songhyun Yu, Bumjun Park, and Jechang Jeong.
\newblock Deep iterative down-up cnn for image denoising.
\newblock In \emph{Proceedings of the IEEE/CVF conference on computer vision and pattern recognition workshops}, pp.\  0--0, 2019.

\bibitem[Zhang et~al.(2017{\natexlab{a}})Zhang, Zuo, Chen, Meng, and Zhang]{zhang2017beyond}
Kai Zhang, Wangmeng Zuo, Yunjin Chen, Deyu Meng, and Lei Zhang.
\newblock Beyond a gaussian denoiser: Residual learning of deep cnn for image denoising.
\newblock \emph{IEEE transactions on image processing}, 26\penalty0 (7):\penalty0 3142--3155, 2017{\natexlab{a}}.

\bibitem[Zhang et~al.(2017{\natexlab{b}})Zhang, Zuo, Gu, and Zhang]{zhang2017learning}
Kai Zhang, Wangmeng Zuo, Shuhang Gu, and Lei Zhang.
\newblock Learning deep cnn denoiser prior for image restoration.
\newblock In \emph{Proceedings of the IEEE conference on computer vision and pattern recognition}, pp.\  3929--3938, 2017{\natexlab{b}}.

\bibitem[Zhang et~al.(2020)Zhang, Gool, and Timofte]{zhang2020deep}
Kai Zhang, Luc~Van Gool, and Radu Timofte.
\newblock Deep unfolding network for image super-resolution.
\newblock In \emph{Proceedings of the IEEE/CVF conference on computer vision and pattern recognition}, pp.\  3217--3226, 2020.

\bibitem[Zhang et~al.(2021)Zhang, Li, Zuo, Zhang, Van~Gool, and Timofte]{zhang2021plug}
Kai Zhang, Yawei Li, Wangmeng Zuo, Lei Zhang, Luc Van~Gool, and Radu Timofte.
\newblock Plug-and-play image restoration with deep denoiser prior.
\newblock \emph{IEEE Transactions on Pattern Analysis and Machine Intelligence}, 44\penalty0 (10):\penalty0 6360--6376, 2021.

\bibitem[Zhang et~al.(2018)Zhang, Isola, Efros, Shechtman, and Wang]{zhang2018unreasonable}
Richard Zhang, Phillip Isola, Alexei~A Efros, Eli Shechtman, and Oliver Wang.
\newblock The unreasonable effectiveness of deep features as a perceptual metric.
\newblock In \emph{Proceedings of the IEEE conference on computer vision and pattern recognition}, pp.\  586--595, 2018.

\bibitem[Zhu et~al.(2018)Zhu, Liu, Cauley, Rosen, and Rosen]{zhu2018image}
Bo~Zhu, Jeremiah~Z Liu, Stephen~F Cauley, Bruce~R Rosen, and Matthew~S Rosen.
\newblock Image reconstruction by domain-transform manifold learning.
\newblock \emph{Nature}, 555\penalty0 (7697):\penalty0 487--492, 2018.

\bibitem[Zou et~al.(2023)Zou, Liu, Wohlberg, and Kamilov]{zou2023deep}
Zihao Zou, Jiaming Liu, Brendt Wohlberg, and Ulugbek~S Kamilov.
\newblock Deep equilibrium learning of explicit regularizers for imaging inverse problems.
\newblock \emph{arXiv preprint arXiv:2303.05386}, 2023.

\end{thebibliography}
\bibliographystyle{iclr2024_conference}

\newpage
\appendix

\section{Related Works}
\label{sec:related-works}
\paragraph{Deep Unrolling}
In addition to Plug-and-Play, deep unrolling is another approach using deep neural networks to replace proximal operators for solving inverse problems. Similar to PnP, the deep unrolling model is parameterized by an unrolled iterative algorithm, with certain (proximal) steps replaced by deep neural nets. In contrast to PnP, the unrolling model is trained in an end-to-end fashion by paired data of ground truth and corresponding measurements from specific forward operators. Truncated deep unrolling methods unfold the algorithm for a fixed number of steps \cite{gregor2010learning,adler2010shrinkage,liu2018alista,aggarwal2018modl,adler2018learned,zhang2020deep,Monga2021-hb,gilton2019neumann,tolooshams2023unrolled,Kobler2017-wr,chen2022learning,Mardani2018-rc, sulam2019multi}, while infinite-step models have been recently developed based on deep equilibrium learning \cite{gilton2021deep,liu2022online,zou2023deep}. In future work, LPN can improve the performance and interpretability of deep unrolling methods in e.g., medical applications \cite{lai2020learned,fang2023deepsti,shenoy2023unrolled} or in cases that demand the analysis of robustness \cite{sulam2020adversarial}. The end-to-end supervision in unrolling can also help increase the performance of LPN-based methods for inverse problems in general.

\paragraph{Explicit Regularizer}
A series of works have been dedicated to designing explicit data-driven
regularizer for inverse problems, such as RED \cite{romano2017little}, AR
\cite{lunz2018adversarial}, ACR \cite{mukherjee2020learned}, UAR
\cite{mukherjee2021end} and others
\cite{li2020nett,kobler2020total,cohen2021has,zou2023deep,Goujon2023-wa}. Our
work contributes a new angle to this field, by learning a proximal operator for
the log-prior and then recovering the regularizer from the learned proximal.

\paragraph{Gradient Denoiser}
Gradient step (GS) denoisers  \cite{cohen2021has,hurault2022gradient,hurault2022proximal} are a cluster of recent approaches that parameterize a denoiser as a gradient descent step using the gradient map of a neural network. Although these works share similarities to our LPN, there are a few key differences.
\begin{enumerate}
    \item Parameterization. In GS denoisers, the denoiser is defined as a gradient descent step: $f = \Id - \nabla g$, where $\Id$ represents the identity operator, and $g$ is a scalar-valued function that is either directly
    parameterized by a neural network \cite{cohen2021has}, or implicitly defined by a network $N:\R^n \rightarrow \R^n$ as $g(\x) = \frac{1}{2} \|\x
    - N(\x)\|_2^2$ \cite{hurault2022gradient,hurault2022proximal}. 
    \citet{cohen2021has} also experiment with a denoiser architecture analogous
    to our LPN architecture, but find its denoising performance to be inferior
    to the GS denoiser (we will discuss this further in the
    final bullet below). 
    In order to have accompanying convergence guarantees when used in PnP
    schemes, these GS parameterizations demand special structures on the learned
    denoiser---in particular, Lipschitz constraints on $\nabla g$---which can be
    challenging to enforce in practice.
    \item Proximal operator guarantee. The GS denoisers in
        \textcite{cohen2021has,hurault2022gradient} are not a priori guaranteed to be
        proximal operators. \textcite{hurault2022proximal} proposed to
        constrain the GS denoiser to be a proximal operator by limiting the Lipschitz
        constant of $\nabla g$, also exploiting the characterization of
        \textcite{gribonval2020characterization}. However, as a result, their
        denoiser necessarily has a bounded Lipschitz constant.
        \rebuttal{Furthermore, in practice, such a constraint is not strictly enforced, but instead realized by adding a regularization term on the spectral norm of the network during training. Such a regularization only penalizes large Lipschitz constants, but does not guarantee that the Lipschitz constant will be lower than the required threshold. Additionally, the regularization is only computed at training data points, thus either not regularizing the network’s behavior globally or resulting in loose upper-bounds for it. In other words, such proximal GS denoiser is only ``encouraged'' to resemble a proximal, but it is not guaranteed. On the other hand, our LPN provides the guarantee that the learned network will always parameterize a proximal operator.}
    \item Training. All GS denoiser methods used the conventional $\ell_2$ loss for
        training. We propose the proximal matching loss and show that it is
        essential for the network to learn the correct proximal operator of the log-prior of
        data distribution.
        Indeed, we attribute the inferior performance of the ICNN-based architecture that
        \textcite{cohen2021has} experiment with, which is analogous to our LPN,
        to the fact that their experiments train this architecture on MMSE-based
        denoising, where ``regression to the mean'' on multimodal and nonlinear
        natural image data hinders performance (see, e.g.,
        \textcite{Delbracio2023-oc} in this connection).
        The key insight that powers our successful application of LPNs in
        experiments is the proximal matching training framework, which allows
        us to make full use of the constrained capacity of the LPN in
        representing highly expressive proximal operators (corresponding to
        (nearly) maximum a-posteriori estimators for data distributions).
\end{enumerate}

\rebuttal{
\paragraph{Comparisons to Diffusion Models} Recently, score-based diffusion
models have proven very efficient for unconditional and conditional image
generation. There are several key differences between our work and diffusion
models. First, conditional diffusion models do not minimize a variational
problem as we do in this paper (as in \Cref{eq:main_inverse_problem}), but
instead provide samples from the posterior distribution. Moreover, the diffusion
models rely on inverting a diffusion process which requires an MMSE denoiser,
and---just as in the case of regular denoisers---they do not approximate any MAP estimate, whereas we are concerned with networks that compute a MAP estimate for a learned prior. In terms of strict advantages, one should again note that our approach solves (provides a MAP estimate) for a denoising problem with a single forward-pass, whereas sampling with diffusion models requires a large sequence of forward passes of a denoising network. 
Lastly, but also importantly, our method provides an exact proximal operator for a learned prior distribution. Diffusion models have no such guarantee: all these results provide samples from an approximate posterior distribution, which relies on the approximation qualities of the MMSE denoiser that do not exist for general cases \cite{chen2023improved}.
}

\paragraph{Proximal Matching Loss and Mode-Seeking Regression Objectives}
In the literature on both deep learning-based denoising and statistical methodology, prior works have explored training schemes that promote learning the mode of a distribution (or, in our denoising setting, the conditional mode/MAP estimate of the prior).
On the methodological side, it is noted that training with respect to a \textit{single} objective function cannot lead to the optimal denoiser being the mode uniformly over sufficiently-expressive classes of denoisers and priors, a concept formalized as \textit{inelicitability of the mode functional}  \cite{Gneiting2011-hf,Heinrich2014-zo}. In contrast, our nonparametric result on the proximal matching loss, \Cref{thm:continuous}, characterizes the minimizer of \textit{a limit of a sequence of losses}. 
This is both outside the framework of the preceding references, and distinct from what occurrs in practice, where we attempt to minimize the proximal matching loss with a sufficiently small parameter $\gamma > 0$. We expect this latter setting to coincide with correct learning of the mode/MAP estimate of the prior in practical settings of interest, when $\gamma$ is much smaller than the `characteristic scale' of the prior. Prior work has also considered modal regression in an abstract statistical learning setting \cite{Feng2020-fw}, where in contrast to our proximal matching-based objective, an approach based on kernel density estimation was advanced.

In the literature on learning deep denoisers,
we note that a previous work \cite{lehtinen2018noise2noisea} used an annealed version of an ``$\ell_0$ loss'' for mode approximation, with motivation similar to that of proximal matching. Their loss takes a different form, $\sum_i(|f(\y) - \x|_i + \epsilon)^\gamma$, where $\epsilon$ is a small constant and $\gamma \in [0, 2]$ is the annealing parameter. Their loss is designed for learning from corrupted targets with random impulse noise, and does not recover the mode of the posterior (as in the case of proximal matching), but rather the zero-crossing of the Hilbert transform of the probability density function.

\section{Additional Theorems}

\subsection{Learning via proximal matching (discrete case)}
\label{sec:thm-discrete}

\begin{theorem}[Learning via Proximal Matching (Discrete Case)]
\label{thm:discrete}
Consider a signal $\x \sim P(\x)$, with $P(\x)$ a discrete distribution,
and a noisy observation $\y = \x + \sigma \beps$,
where $\beps \sim \mathcal{N}(0, \mathbf{I})$ and $\sigma > 0$.
Let $m_\gamma(x) : \R \to \R$ be defined by $m_{\gamma}(x) =
1 - \exp\left(-\frac{x^2}{\gamma^2}\right)$ \footnote{This definition of $m_\gamma$ differs slightly from the one in \Cref{eq:loss}, but they are equivalent in terms of minimization objective as they only differ by a scaling constant.}. %
Consider the optimization problem
\begin{equation*}
f^* = \argmin_{f\ \mathrm{measurable}} \lim_{\gamma \searrow 0} \E_{\x,\y} \left[
m_\gamma \left( \|f(\y) - \x\|_2 \right)  \right].
\end{equation*}
Then, almost surely (i.e., for almost all $\y$),
$f^*(\y) = \argmax_{\cc} P(\x=\cc\mid\y).$
\end{theorem}

The proof is deferred to \Cref{sec:proof-discrete}.

\subsection{Convergence of PnP-PGD using LPN}
\label{sec:thm-pnp-pgd-lpn}

\begin{theorem}[Convergence guarantee for running PnP-PGD with LPNs]\label{thm:pnp-pgd-lpn-stationary}
    Consider the sequence of iterates $\x_{k}$, $k \in \{0, 1, \dots\}$,
    defined by \Cref{alg:pgd} run with a linear measurement operator $\A$ and an 
    LPN $f_{\theta}$ with softplus activations, trained with $0 < \alpha < 1$.
    Assume that the step size satisfies $0 < \eta < 1/\| \A^T \A \|$.
    Then, the iterates $\x_k$ converge to a fixed point $\x^*$ of
    \Cref{alg:pgd}: that is, there exists $\x^* \in \R^n$ such that
    $\lim_{k \to \infty} \x_k = \x^*$, and
    \begin{equation}
        f_{\theta}\left(
            \x^* - \eta \nabla h(\x^*)
        \right)
        =
        \x^*.
        \label{eq:opti-theorem-fxp}
    \end{equation}
\end{theorem}

The proof is deferred to \cref{sec:proof-pnp-pgd-lpn-stationary}.

\section{Proofs}
In this section, we include the proofs for the results presented in this paper.

\subsection{Proof of \Cref{prop:lpn}}
\label{sec:proof-lpn}
\begin{proof}
By \textcite[Proposition 1]{Amos2017-ql}, $\psi_\theta$ is convex. Since the activation $g$ is differentiable, $\psi_\theta$ is also differentiable. Hence, $f_\theta = \nabla \psi_\theta$ is the gradient of a convex function. Thus, by \Cref{prop:characterization-continuous-prox}, $f_\theta$ is a proximal operator of a function.
\end{proof}

\subsection{Proof of \Cref{thm:continuous}}
\label{sec:proof-continuous}
\begin{proof}

    First, note by linearity of the expectation that for any measurable $f$,
    one has
    \begin{equation}
        \lim_{\gamma \searrow 0} \E_{\x,\y} \left[ m_\gamma \left( \|f(\y) - \x\|_2
        \right)  \right]
        =
        1 - 
        \lim_{\gamma \searrow 0} \E_{\x,\y} 
        \left[ \varphi_{\gamma^2/2}(f(\y) - \x) \right],
        \label{eq:limit-loss-simplified}
    \end{equation}
    where $\varphi_{\gamma^2/2}$ denotes the density of an isotropic Gaussian
    random variable with mean zero and variance $\gamma^2/2$.
        Because $p(\x)$ is a continuous density with respect to the Lebesgue measure $d \x$, by Gaussian conditioning, we have that the
        conditional distribution of $\x$ given $\y$ admits a density $p_{\x \mid
        \y}$ with respect to $d \x$ as well.
    Taking conditional expectations, we have
    \begin{equation}
        \lim_{\gamma \searrow 0} \E_{\x,\y} 
        \left[ \varphi_{\gamma^2/2}(f(\y) - \x) \right]
        =
        \lim_{\gamma \searrow 0} \E_{\y} \E_{\x \mid \y}
        \left[ \varphi_{\gamma^2/2}(f(\y) - \x) \right].
        \label{eq:cts-case-expectation-factoring}
    \end{equation}
    From here, we can state the intuition for the remaining portion of the
    proof. Intuitively, because the Gaussian density $\varphi_{\sigma^2/2}$
    concentrates more and more at zero as $\gamma \searrow 0$, and meanwhile
    is nevertheless a probability density for every $\gamma > 0$,\footnote{For
        readers familiar with signal processing or Schwartz's theory of
        distributions, this could be alternately stated as ``the small-variance
    limit of the Gaussian density behaves like a Dirac delta distribution''.}
    the inner expectation over $\x \mid \y$ leads to simply replacing the
    integrand with its value at $\x = f(\y)$; the integrand is of course the
    conditional density of $\x$ given $\y$, and from here it is straightforward
    to argue that this leads the optimal $f$ to be (almost surely) the
    conditional maximum a posteriori (MAP) estimate, under our regularity
    assumptions on $p(\x)$. 

    To make this intuitive argument rigorous, we need to translate our
    regularity assumptions on $p(\x)$ into regularity of $p_{\x \mid \y}$, 
    interchange the $\gamma$ limit in \Cref{eq:cts-case-expectation-factoring}
    with the expectation over $\y$, and instantiate a rigorous analogue of the
    heuristic ``concentration'' argument.  
    First, we have by Bayes' rule and Gaussian conditioning
    \begin{equation*}
        p_{\x \mid \y}(\x) = \frac{ \varphi_{\sigma^2}(\y - \x) p(\x) }{
            (\varphi_{\sigma^2} \conv p)(\y),
        }
    \end{equation*}
    where $\conv$ denotes convolution of densities; the denominator is the
    density of $\y$, and it satisfies $\varphi_{\sigma^2} \conv p > 0$ since
    $\varphi_{\sigma^2} > 0$. In particular, this implies that $p_{\x \mid \y}$
    is a continuous function of $(\x, \y)$, because $p(\x)$ is continuous by
    assumption. We can then write, by the definition of convolution,
    \begin{equation*}
        \E_{\x \mid \y}
        \left[ \varphi_{\gamma^2/2}(f(\y) - \x) \right]
        =
        \varphi_{\gamma^2/2} \ast p_{\x \mid \y}(f(\y)),
    \end{equation*}
    so following \Cref{eq:cts-case-expectation-factoring}, we have
    \begin{equation}
        \lim_{\gamma \searrow 0} \E_{\x,\y} 
        \left[ \varphi_{\gamma^2/2}(f(\y) - \x) \right]
        =
        \lim_{\gamma \searrow 0} \E_{\y}\left[
            \varphi_{\gamma^2/2} \ast p_{\x \mid \y}(f(\y))
        \right].
        \label{eq:cts-case-expectation-factoring-1}
    \end{equation}
    We are going to argue that the limit can be moved inside the expectation in
    \Cref{eq:cts-case-expectation-factoring-1} momentarily; for the moment, we
    consider the quantity that results after moving the limit inside the
    expectation.
    To treat this term, we apply a standard approximation to the identity
    argument to evaluate the limit of the preceding expression.
    \cite[Ch.\ 3, Example 3]{Stein2005-io} implies that the densities
    $\varphi_{\gamma^2/2} $ constitute an approximation to the identity as
    $\gamma \to 0$,
    and because $p_{\x \mid \y}$ is continuous, we can then apply \cite[Ch.\ 3,
    Theorem 2.1]{Stein2005-io} to obtain that
    \begin{equation*}
        \lim_{\gamma \searrow 0}
        \varphi_{\gamma^2/2} \ast p_{\x \mid \y}(f(\y))
        =
        p_{\x \mid \y}(f(\y)).
    \end{equation*}
    In particular, after justifying the interchange of limit and expectation in
    \Cref{eq:cts-case-expectation-factoring-1}, we will have shown, by following
    our manipulations from \Cref{eq:limit-loss-simplified}, that
    \begin{equation}
        \lim_{\gamma \searrow 0} \E_{\x,\y} \left[ m_\gamma \left( \|f(\y) - \x\|_2
        \right)  \right]
        =
        1 - 
        \E_{\y} \left[
            p_{\x \mid \y}(f(\y))
        \right].
        \label{eq:cts-case-thing-to-justify}
    \end{equation}
    We will proceed to conclude the proof from this expression, and justify the
    limit-expectation interchange at the end of the proof. The problem at hand
    is equivalent to the problem 
    \begin{equation*}
        \argmax_{f\ \mathrm{measurable}} 
        \E_{\y} \left[
            p_{\x \mid \y}(f(\y))
        \right].
    \end{equation*}
    Writing the expectation as an integral, we have by Bayes' rule as above
    \begin{equation*}
        \E_{\y} \left[
            p_{\x \mid \y}(f(\y))
        \right]
        =
        \int_{\R^n} \varphi_{\sigma^2}(\y - f(\y)) p(f(\y)) d\y.
    \end{equation*}
    Let us define an auxiliary function $g : \R^n \times \R^n \to \R$ by
    $g(\x, \y) = \varphi_{\sigma^2}(\y - \x) p(\x)$. Then 
    \begin{equation*}
        \E_{\y} \left[
            p_{\x \mid \y}(f(\y))
        \right]
        =
        \int_{\R^n} g(f(\y), \y) d\y,
    \end{equation*}
    and moreover,
    for every $\y$, $g(\,\cdot\,,\y)$ is continuous and compactly supported, by
    continuity and boundedness of the Gaussian density and the assumption that
    $p(\x)$ is continuous and the random variable $\x \sim p(\x)$ is bounded.
    We have for any measurable $f$
    \begin{equation}
        g(f(\y), \y) \leq \max_{\x \in \R^n} g(\x, \y).
        \label{eq:cts-case-make-bound-tight}
    \end{equation}
    Our aim is thus to argue that there is a choice of measurable $f$ such that
    the preceding bound can be made tight; this will imply that any measurable
    $f$ maximizing the objective $\E_{\y} [ p_{\x \mid \y}(f(\y)) ]$ satisfies
    $g(f(\y), \y) = \max_{\x \in \R^n} g(\x, \y)$ almost surely, or equivalently
    that $f(\y) \in \argmax_{\x \in \R^n} g(\x, \y)$ almost surely.
    The claim will then follow, because $\argmax_{\x \in \R^n} g(\x, \y) =
    \argmax_{\x \in \R^n} p_{\x \mid \y}(\x)$.

    To this end, define $h(\y) = \max_{\x \in \R^n} g(\x, \y)$. Then by the
    Weierstrass theorem, $h$ is finite-valued,
    and for every $\y$ there exists some $\cc \in \R^n$ such that $h(\y) =
    g(\cc, \y)$. Because $g$ is continuous, it then
    follows from \textcite[Theorem 1.17(c)]{Tyrrell_Rockafellar1998-ol} 
    that $h$ is continuous. Moreover, because $g$ is continuous and for every
    $\y$, $g(\,\cdot\,, \y)$ is compactly supported, $g$ is in particular
    level-bounded in $\x$ locally uniformly in $\y$ in the sense of
    \textcite[Definition 1.16]{Tyrrell_Rockafellar1998-ol}, and it follows that the
    set-valued mapping $\y\mapsto \argmax_{\x} g(\x, \y) : \R^n \rightrightarrows \R^n$ is
    compact-valued, by the Weierstrass theorem, and outer semicontinuous
    relative to $\R^n$, by \textcite[Example 5.22]{Tyrrell_Rockafellar1998-ol}.
    Applying \textcite[Exercise 14.9, Corollary
    14.6]{Tyrrell_Rockafellar1998-ol}, we conclude that the set-valued mapping 
    $\y\mapsto \argmax_{\x} g(\x, \y)$ is measurable, and that in particular there exists
    a measurable function $f^* : \R^n \to \R^n$ such that $f^*(\y) \in \argmax_{\x}
    g(\x, \y)$ for every $\y \in \R^n$. Thus, there is a measurable $f$
    attaining the bound in \Cref{eq:cts-case-make-bound-tight}, and the claim
    follows after we can justify the preceding interchange of limit and
    expectation. 

    To justify the interchange of limit and expectation, we will apply the
    dominated convergence theorem, which requires us to show an integrable (with
    respect to the density of $\y$) upper bound for the function $\y \mapsto \E_{\x \mid
    \y}[ \varphi_{\gamma^2/2}(f(\y) - \x)]$.
    For this, we calculate
    \begin{align*}
        \E_{\x \mid \y}
        \left[ \varphi_{\gamma^2/2}(f(\y) - \x) \right]
        &=
        \frac{1}{(\varphi_{\sigma^2} \ast p)(\y)}
        \int_{\R^n}
        \varphi_{\sigma^2}(\y - \x) p(\x) 
        \varphi_{\gamma^2/2}(f(\y) - \x)
        d \x
        \\
        &\leq
        \frac{1}{(\varphi_{\sigma^2} \ast p)(\y)}
        \left[
            \sup_{\x}\,
            \varphi_{\sigma^2}(\y - \x) p(\x) 
        \right]
        \int_{\R^n}
        \varphi_{\gamma^2/2}(f(\y) - \x)
        d \x
        \\
        &=
        \frac{1}{(\varphi_{\sigma^2} \ast p)(\y)}
        \left[
            \sup_{\x}\,
            \varphi_{\sigma^2}(\y - \x) p(\x) 
        \right],
    \end{align*}
    by H\"{o}lder's inequality and the fact that $\varphi_{\gamma^2/2}$ is a
    probability density.
    Because the random variable $\x \sim p(\x)$ is assumed bounded, the 
    density $p(\x)$ has compact support, and the density $p(\x)$ is assumed
    continuous, so there exists $R > 0$ such that if $\| \x \|_2 > R$ then
    $p(\x) = 0$, and $M > 0$ such that $p(\x) \leq M$. We then have
    \begin{equation*}
        \sup_{\x}\,
        \varphi_{\sigma^2}(\y - \x) p(\x) 
        \leq
        M
        \sup_{\x}\,
        \varphi_{\sigma^2}(\y - \x) \mathds{1}_{\| \x \|_2 \leq R}.
    \end{equation*}
    This means that the supremum can attain a nonzero value only on points
    where $\| \x \|_2 \leq R$.
    On the other hand, for every $\y$ with $\| \y \|_2 \geq 2R$, 
    whenever $\| \x \|_2 \leq R$ the triangle inequality implies $\| \y - \x
    \|_2 \geq \| \y \|_2 - \| \x \|_2 \geq \tfrac{1}{2} \| \y \|_2$. Because
    the Gaussian density $\varphi_{\sigma^2}$ is a radial function, 
    we conclude that if $\| \y \|_2 \geq 2R$, one has
    \begin{equation*}
        \sup_{\x}\,
        \varphi_{\sigma^2}(\y - \x) p(\x) 
        \leq
        M \varphi_{\sigma^2}(\y/2)
        =
        C M \varphi_{4\sigma^2}(\y),
    \end{equation*}
    where $C>0$ depends only on $n$. At the same time, we always have 
    \begin{equation*}
        \sup_{\x}\,
        \varphi_{\sigma^2}(\y - \x) p(\x) 
        \leq
        \frac{M}{(2\pi \sigma^2)^{n/2}}.
    \end{equation*}
    Consequently, we have the composite upper bound
    \begin{equation*}
        \sup_{\x}\,
        \varphi_{\sigma^2}(\y - \x) p(\x) 
        \leq
        \begin{cases}
            \tfrac{M}{(2\pi \sigma^2)^{n/2}} & \| \y \|_2 < 2R \\
            2M \varphi_{4\sigma^2}(\y) & \| \y \|_2 \geq 2R, 
        \end{cases}
    \end{equation*}
    and by our work above
    \begin{align*}
        \E_{\x \mid \y}
        \left[ \varphi_{\gamma^2/2}(f(\y) - \x) \right]
        \leq
        \frac{1}{(\varphi_{\sigma^2} \ast p)(\y)}
        \times
        \begin{cases}
            \tfrac{M}{(2\pi \sigma^2)^{n/2}} & \| \y \|_2 < 2R \\
            2M \varphi_{4\sigma^2}(\y) & \| \y \|_2 \geq 2R.
        \end{cases}
    \end{align*}
    Because $\varphi_{\sigma^2} \conv p$ is the density of $\y$, this upper
    bound is sufficient to apply the dominated convergence theorem to obtain
    \begin{equation*}
        \lim_{\gamma \searrow 0} \E_{\x,\y} 
        \left[ \varphi_{\gamma^2/2}(f(\y) - \x) \right]
        =
        \E_{\y} \lim_{\gamma \searrow 0} 
        \E_{\x \mid \y} \left[ \varphi_{\gamma^2/2}(f(\y) - \x) \right].
    \end{equation*}
    Combining this assertion with the argument surrounding
    \Cref{eq:cts-case-thing-to-justify}, we conclude the proof.

\end{proof}

\begin{remark}[Other loss choices]
\label{remark:other-prox-matching-loss}
\Cref{thm:continuous} also holds for any $m_\gamma$ such that
$m_{\gamma}$ is uniformly (in $\gamma$) bounded above, 
for each $\gamma > 0$ uniquely minimized at $0$, and $\sup_{x \in \R} m_\gamma(x) - m_\gamma(\| \x \|_2)$
is an approximation to the identity as $\gamma \searrow 0$ (see \cite[Ch.\ 3, \S
2]{Stein2005-io}).
\end{remark}

\subsection{Proof of \Cref{thm:discrete}}
\label{sec:proof-discrete}
\begin{proof}
    For brevity, we denote $\argmax_{\cc} P(\x = \cc \mid \y)$ by $\MAP [\x \mid \y]$, i.e., the maximum a posteriori estimate of $\x$ given $\y$.
    
    First, we show that $\MAP [\x \mid \y]$ is unique for almost all $\y$.

    Consider $\y$ such that $\MAP [\x \mid \y]$ is not unique. There exists $i \neq j$, such that 
    \begin{align*}
        &
        P(\x_i \mid \y) = P(\x_j \mid \y) \\
        \iff &
        p(\y \mid \x_i) P(\x_i) = p(\y \mid \x_j) P(\x_j) \\
        \iff &
        -\frac{1}{2}\|\y - \x_i\|^2 + \sigma^2 \log P(\x_i) = -\frac{1}{2}\|\y - \x_j\|^2 + \sigma^2 \log P(\x_j) \\
        \iff &
        \langle \y, \frac{\x_i - \x_j}{2} \rangle = \frac{1}{2}\|\x_i\|^2 - \frac{1}{2} \|\x_j\|^2 - \sigma^2 \log P(\x_i) + \sigma^2 \log P(\x_j).
    \end{align*}
    i.e., $\y$ lies in a hyperplane defined by $\x_i, \x_j$ (note that $\x_i \neq \x_j$).
    Denote the hyperplane by \[\mathcal{H}_{i,j} := \left\{\y \mid \langle \y,
        \frac{\x_i - \x_j}{2} \rangle = \frac{1}{2}\|\x_i\|^2 - \frac{1}{2}
    \|\x_j\|^2 - \sigma^2 \log P(\x_i) + \sigma^2 \log P(\x_j)\right\}. \]

    Consider \[\mathcal{U} := \cup_{i\neq j} \mathcal{H}_{i,j}. \]
    We have that $\forall \y$ with non-unique $\MAP [\x \mid \y]$, 
    \begin{align*}
        &\exists i \neq j, \y \in \mathcal{H}_{i,j} \\
        \iff & 
        \y \in \mathcal{U}.
    \end{align*}
    Note that $\mathcal{U}$ has zero measure as a countable union of
    zero-measure sets, hence the measure of all $\y$
    with non-unique $\MAP [\x \mid \y]$ is zero. Hence, for almost all $\y$,
    $\MAP [\x \mid \y]$ is unique.

    Next, we show that for almost all $\y$, 
    \begin{align*}
        f^*(\y) = \argmin_{\cc} \E_{\x \mid \y} [ \mathds{1}_{\cc \neq \x}].
    \end{align*}
    Note that 
    \begin{align*}
    & 
    \lim_{\gamma \searrow 0} \E_{\x,\y} \left[ m_\gamma \left( \|f(\y) - \x\|_2 \right)  \right] \\
    = &
    \E_{\x,\y} \left[ \lim_{\gamma \searrow 0} 
 m_\gamma \left( \|f(\y) - \x\|_2 \right)  \right] \\
    = &
    \E_{\x,\y} \left[ \mathds{1}_{\|f(\y) - \x\|_2 \neq 0}  \right] \\
    = & 
    \E_{\x,\y} \left[ \mathds{1}_{f(\y) \neq \x}  \right] .
    \end{align*}
    Above, the first equality uses the monotone convergence theorem.
    Use the law of iterated expectations,
    \begin{align*}
    \E_{\x,\y} \left[ \mathds{1}_{f(\y) \neq \x}  \right] =
    \E_{\y} \E_{\x \mid \y} \left[ \mathds{1}_{f(\y) \neq \x}  \right].
    \end{align*}
    We will use this expression to study the global minimizers of the objective. 
    By conditioning,
    \begin{equation*}
        \E_{\x \mid \y} \left[ \mathds{1}_{f(\y) \neq \x}  \right]
        \geq
        \min_{\cc} \E_{\x \mid \y} [ \mathds{1}_{\cc \neq \x}  ],
    \end{equation*}
    and so
    \begin{equation*}
    \E_{\y} \left[
        \E_{\x \mid \y} \left[ \mathds{1}_{f(\y) \neq \x}  \right]
        -
        \min_{\cc} \E_{\x \mid \y} [ \mathds{1}_{\cc \neq \x}  ]
        \right] \geq 0.
    \end{equation*}
    Because $p(\y) > 0$, it follows that every global minimizer of the
    objective $f^*$ satisfies
    \begin{equation*}
        \E_{\x \mid \y} \left[ \mathds{1}_{f^*(\y) \neq \x}  \right]
        =
        \min_{\cc} \E_{\x \mid \y} [ \mathds{1}_{\cc \neq \x}  ]
        \enspace \text{a.s.}
    \end{equation*}
    Hence, for almost all $\y$, 
    \begin{align*}
        f^*(\y) \in \argmin_{\cc} \E_{\x \mid \y} [ \mathds{1}_{\cc \neq \x}].
    \end{align*}
    
    Finally, we show that $\argmin_{\cc} \E_{\x \mid \y} [ \mathds{1}_{\cc \neq
    \x}] = \MAP [\x \mid \y]$. The claim then follows from our preceding work showing that $\MAP[\x \mid \y]$ is almost surely unique. Consider
    \begin{align*}
        \E_{\x \mid \y} [ \mathds{1}_{\cc \neq \x}  ] 
        &= 
        \sum_i P(\x_i \mid \y) \mathds{1}_{\cc \neq  \x_i} \\
        &= 
        \sum_i P(\x_i \mid \y) (1-\mathds{1}_{\cc =  \x_i}) \\
        &=
        \sum_{i} P(\x_i \mid \y) - \sum_{\x_i = \cc} P(\x_i \mid \y) \\
        &=
        1 - P(\x=\cc \mid \y).
    \end{align*}
    Hence,
    \begin{align*}
        \argmin_{\cc} \E_{\x \mid \y} [ \mathds{1}_{\cc \neq \x}] &=
        \argmax_{\cc} P(\x=\cc \mid \y) \\
        & = \MAP [\x \mid \y].
    \end{align*}
    
\end{proof}

\subsection{Proofs of PnP Optimization Results}\label{sec:pnp-proofs}

In this section, we restate and provide proofs of \Cref{thm:pnp-pgd-lpn-stationary} and
\Cref{thm:pnp-admm-lpn-stationary}. We prove \Cref{thm:pnp-pgd-lpn-stationary}
under slightly more
general assumptions, and state the conclusions of both \Cref{thm:pnp-pgd-lpn-stationary,thm:pnp-admm-lpn-stationary} with more precision. The restated results are given below, as \Cref{thm:pnp-pgd-lpn-stationary-appendix} and
\Cref{thm:pnp-admm-lpn-stationary-appendix}.

Before proceeding to proofs, let us briefly describe the common high-level
`recipe' underlying each plug-and-play algorithm's proof. The recipe separates
into two distinct steps:
\begin{enumerate}
    \item \textbf{Leverage general, black-box convergence analyses from the
        optimization literature.} A plug-and-play algorithm is derived from
        a `baseline' optimization algorithm; we therefore appeal to convergence
        analyses from the literature of the relevant baseline algorithm. Because
        the regularization function associated to a LPN is implicitly defined by
        the LPN architecture and need not be convex, it is necessary to appeal
        to general, `black-box' convergence analyses which do not leverage
        special properties of the regularization function.
        We make use of
        convergence results on nonconvex proximal gradient descent of
        \citet{Bot2016-mw},\footnote{The form these results are stated in makes
        them most convenient for purposes of our presentation, although the
        result we need is originally due to \cite{Attouch2013-vc}.} and on nonconvex ADMM of \citet{Themelis2020-jj}. To
        make the presentation self-contained, we reproduce key results from
        these works in context. The principal technical activity is therefore to
        translate the iterate sequence generated by the relevant PnP algorithm
        into a form that allows these convergence analyses to be applied to it.
        Echoes of the same approach appear in prior work on convergent
        plug-and-play, for example work of \citet{hurault2022proximal}.
    \item \textbf{Establish general regularity properties of the regularization
        function associated to LPNs.} To appeal to the aforementioned
        convergence analyses, it is necessary to ascertain a minimum level of
        regularity of the regularization function associated to an LPN, in order
        to establish that it possesses the Kurdyka-\L{}ojasiewicz (KL) property (and, say, coercivity).
        We give a self-contained overview of the KL property and how we
        establish it in \Cref{sec:kl-omin} for clarity of presentation.
        We provide in \Cref{sec:lpn-regularity} technical lemmas that establish that LPNs of the architecture
        specified in \Cref{prop:lpn} satisfy these properties,
        \textit{regardless of the exact values of their parameters}. These results are
        essentially consequences of differentiability and surjectivity of the
        LPN when $0 < \alpha < 1$ is used as the strong convexity weight, and
        they enable us to assert convergence guarantees for LPNs \textit{without
        any extra assumptions about the trained network}.
\end{enumerate}

We anticipate that this recipe will be applicable to virtually any PnP scheme
for which there exists a convergence analysis under the KL property of the
corresponding baseline optimization algorithm. Because our technical work in
\Cref{sec:lpn-regularity} establishes the KL property and coercivity for the regularization
function associated to LPNs with the architecture of \Cref{prop:lpn}, obtaining
a convergence analysis for such a PnP scheme with LPNs of this architecture
only requires the first step of the above recipe. We expect our approach in
\Cref{sec:lpn-regularity} to extend straightforwardly to LPNs with novel
architectures---for instance, different computational graphs or weight-sharing
schemes---as long as the nonlinear activation functions do not grow too rapidly
(see the proofs for more precise statements).

\subsubsection{Proof of \Cref{thm:pnp-pgd-lpn-stationary} (PnP-PGD)}\label{sec:proof-pnp-pgd-lpn-stationary}

\begin{theorem}[Convergence guarantee for running PnP-PGD with LPNs]
    Consider the sequence of iterates $\x_{k}$, $k \in \{0, 1, \dots\}$,
    defined by \Cref{alg:pgd} run with a continuously differentiable measurement operator $A$ and an 
    LPN $f_{\theta}$ with softplus activations, trained with $0 < \alpha < 1$.
    Assume further that the
    data fidelity term $h(\x) = \tfrac{1}{2} \| \y - A(\x)
    \|_2^2$ is definable \rebuttal{in the o-minimal structure of \Cref{prop:o-minimal},
    Property 2}\footnote{This
    mild technical assumption is
    satisfied by an extremely broad array of nonlinear operators $A$: for
    example, any $A$ which is a polynomial in the input $\x$ (in particular,
    linear $A$) is definable, and compositions and inverses of definable
    functions are definable, so that definability of $A$ implies definability
    of $h$. \rebuttal{See an extensive overview of these ideas in
    \Cref{sec:kl-omin}}} and has
    $L$-Lipschitz gradient\footnote{This is a very mild
        assumption. For example, when $A$ is linear, the gradient of the
        data fidelity term $\nabla h$ has a Lipschitz constant no larger
        than $\| A^* A \|$, where $\| \,\cdot\, \|$ denotes the
        operator norm of a linear operator and $A^*$ is the adjoint of
    $A$.}, and that the step size satisfies $0 < \eta < 1/L$.
    Then, the iterates $\x_k$ converge to a fixed point $\x^*$ of
    \Cref{alg:pgd}: that is, there exist $\x^* \in \R^n$ such that
    \begin{equation}
        f_{\theta}\left(
            \x^* - \eta \nabla h(\x^*)
        \right)
        =
        \x^*,
        \label{eq:opti-theorem-fxp-appendix}
    \end{equation}
    and $\lim_{k \to \infty} \x_k = \x^*$.
    Furthermore, $\x^*$ is a critical point\footnote{In this work, the set of
    critical points of a function $f$ is defined by $\crit(f) := \{\x: 0 \in
    \partial f(\x)\}$, where $\partial f$ is the limiting (Mordukhovich)
    Fr\'echet subdifferential mapping of $f$ (see definition in \cite[Section 2]{Bot2016-mw}).} of $h+\tfrac{1}{\eta}R_{\theta}$, where $R_{\theta}$ is the
    regularization function associated to the LPN $f_{\theta}$ (i.e., $f_\theta = \prox_{R_\theta}$).
    \label{thm:pnp-pgd-lpn-stationary-appendix}
\end{theorem}

Before proceeding to the proof, we state a few settings and results from \textcite{Bot2016-mw} that are useful for proving \Cref{thm:pnp-pgd-lpn-stationary-appendix}, for better readability.

\begin{problem}[{\cite[Problem 1]{Bot2016-mw}}]\label{prob:bot}
Let $f: \R^m \to (-\infty, +\infty]$ be a proper, lower semicontinuous function which is bounded below and let $h: \R^m \to \R$ be a Fr\'echet differentiable function with Lipschitz continuous gradient, i.e. there exists $L_{\nabla h} \geq 0$ such that $\|\nabla h(\x) - \nabla h(\x')\| \leq L_{\nabla h}\|\x - \x'\|$ for all $\x,\x' \in \R^m$. Consider the optimization problem 
\[(P)\ \ \underset{\x\in \R^m}{\inf} [f(\x) + h(\x)].\]
\end{problem}

\begin{algorithm2}[{\cite[Algorithm 1]{Bot2016-mw}}]\label{alg:bot}
    Choose $\x_0, \x_1 \in \R^m, \ualpha, \oalpha >0, \beta \geq 0$ and the sequences $(\alpha_n)_{n\geq 1}, (\beta_n)_{n\geq 1}$ fulfilling 
    \[0<\ualpha\leq \alpha_n \leq \oalpha \ \ \forall n\geq 1\]
    and
    \[0\leq \beta_n \leq \beta \ \ \forall n\geq 1.\]
    Consider the iterative scheme
    \begin{equation}
    (\forall n\geq 1)\ \x_{n+1} \in \argmin_{\U \in \R^m} \{ D_F(\U,\x_n) + \alpha_n \langle \U, \nabla h(\x_n) \rangle + \beta_n \langle \U, \x_{n-1} - \x_n \rangle + \alpha_n f(\U) \}.\label{eq:alg:bot}
    \end{equation}
    Here, $F:\R^m \to \R$ is $\sigma$-strongly convex, Fr\'echet differentiable and $\nabla F$ is $L_{\nabla F}$-Lipschitz continuous, with $\sigma, L_{\nabla F}>0$; $D_F$ is the Bregman distance to $F$.
\end{algorithm2}

\begin{theorem}[{\cite[Theorem 13]{Bot2016-mw}}]\label{thm:bot-13}
    In the setting of \Cref{prob:bot}, choose $\ualpha$, $\oalpha$, $\beta$ satisfying 
    \begin{equation}
    \sigma > \oalpha L_{\nabla_h} + 2\beta \frac{\oalpha}{\ualpha}.
    \label{eq:bot-eq-11}
    \end{equation}
    Assume that $f+h$ is coercive and that 
    \[H: \R^m \times \R^m \to (-\infty, +\infty], \ H(\x,\x') = (f+h)(\x)+\frac{\beta}{2\ualpha}\|\x-\x'\|^2,\ 
    \forall (\x,\x') \in \R^m \times \R^m\]
    is a KL function\footnote{In this work, a function being KL means it satisfies the Kurdyka-\L{}ojasiewicz property \cite{lojasiewicz1963propriete}, see
    \Cref{sec:kl-omin}, \Cref{def:KL-property}.}. Let $(\x_n)_{n\in \mathbb{N}}$ be a sequence generated by \Cref{alg:bot}. Then the following statements are true:
    \begin{enumerate}
        \item $\sum_{n \in \mathbb{N}} \|\x_{n+1} - \x_n\| < +\infty$
        \item there exists $\x \in \crit (f+h)$ such that $\lim_{n\to +\infty} \x_n = \x$.
    \end{enumerate}
\end{theorem}

Now, we prove \Cref{thm:pnp-pgd-lpn-stationary-appendix}.
\begin{proof}[Proof of \Cref{thm:pnp-pgd-lpn-stationary-appendix}.]
    By \Cref{lem:prior-coercive}, there is a coercive function $R_{\theta} :
    \R^n \to \R \cup \{+\infty \}$ such that $f_{\theta} =
    \prox_{R_{\theta}}$.
    The idea of the proof is to apply \Cref{thm:bot-13} to our
    setting; this requires us to check that \Cref{alg:pgd} maps onto \Cref{alg:bot}, and that our
    (implicitly-defined) objective function and parameter choices satisfy the
    requirements of this theorem.
    To this end, note that the application of $f_{\theta}$ in \Cref{alg:pgd} can
    be written as
    \begin{align*}
        \x_{k+1} 
        &= f_{\theta} \left(\x_k - \eta \nabla h(\x_k) \right) \\
        &=
        \argmin_{\x' \in \R^n}\, 
        \frac{1}{2} \left\| \x' - \left( \x_k - \eta \nabla h(\x_k) \right)
        \right\|_2^2
        + R_{\theta}(\x')
        \\
        &=
        \argmin_{\x' \in \R^n}\, 
        \frac{1}{2} \left\| \x' - \x_k \right\|_2^2 + \langle \x' - \x_k,  \eta \nabla h(\x_k)  \rangle
        + R_{\theta}(\x')
        \\
        &=
        \argmin_{\x' \in \R^n}\, 
        \frac{1}{2} \left\| \x' - \x_k \right\|_2^2 + \eta \langle \x',  \nabla h(\x_k)  \rangle
        + \eta \cdot \frac{1}{\eta} R_{\theta}(\x')
    \end{align*}
    showing that \Cref{alg:pgd} corresponds to \Cref{alg:bot} %
    with the Bregman distance $D_F(\x,\y)=\tfrac{1}{2} \| \x-\y \|_2^2$ (and correspondingly $F(\x) = \tfrac{1}{2} \| \x \|_2^2$, which satisfies $\sigma = L_{\nabla F} = 1$),
    the momentum parameter $\beta = \beta_n = 0$, the step size $\alpha_n
    = \oalpha = \ualpha = \eta$, and $f=\tfrac{1}{\eta}R_\theta$. In the framework of
    \textcite{Bot2016-mw}, \Cref{alg:pgd} minimizes the implicitly-defined
    objective $h + \eta^{-1} R_{\theta}$.
    Moreover, one checks that our choice of
    constant step size $0 < \eta < 1/L$ verifies the necessary condition
    \Cref{eq:bot-eq-11}, and because $h \geq 0$, coercivity of
    $R_{\theta}$ implies that $h + \eta^{-1} R_{\theta}$ is coercive.
    \rebuttal{%
    Using \Cref{lem:phi-KL}, we obtain that $R_{\theta}$ is definable,
    and by assumption, $h$ is also definable, so that by \Cref{prop:o-minimal},
    Properties 2 and 5, it follows that the objective $h + \eta^{-1}
    R_{\theta}$ is definable.}
    Thus $h + \eta^{-1} R_{\theta}$ is definable, continuously differentiable (by
    \Cref{lem:prior-coercive}), and proper (as a sum of real-valued functions,
    again by \Cref{lem:prior-coercive}), and therefore has the KL property, by
    \Cref{prop:o-minimal}, Property 1. We
    can therefore apply \Cref{thm:bot-13} to conclude
    convergence to a critical point of $h + \eta^{-1} R_{\theta}$.
    Finally, by \Cref{lem:fixed-point} and the continuity of $f_\theta$ and $\nabla h$, we conclude convergence to a fixed point, $\x = f_{\theta}(\x -\eta \nabla h(\x))$,  which is identical to \Cref{eq:opti-theorem-fxp-appendix}.
\end{proof}

\begin{lemma}[Convergence Implies Fixed Point Convergence]\label{lem:fixed-point}
Suppose $\cF:\R^n \to \R^n$ is a \emph{continuous} map that defines an iterative process, $\x_{k+1} = \cF(\x_k)$. Assume $\x_k$ converges, i.e., $\exists\ \x^*$ such that $\lim_{k\to \infty} \x_k = \x^*$. Then, $\x^*$ is a fixed point of $\cF$, i.e., $\x^* = \cF(\x^*)$.
\end{lemma}
\begin{proof}
    \begin{align*}
        \x^* 
        = \lim_{k \to \infty} \x_k 
        = \lim_{k \to \infty} \x_{k+1} 
        = \lim_{k \to \infty} \cF(\x_{k}) 
        = \cF \left( \lim_{k \to \infty} \x_{k} \right) 
        = \cF( \x^* ).
    \end{align*}
    The fourth equality follows from continuity of $\cF$. 
\end{proof}

\subsubsection{Proof of \Cref{thm:pnp-admm-lpn-stationary} (PnP-ADMM)}\label{sec:proof-pnp-admm-lpn-stationary}

We present in this section a proof of convergence for PnP-ADMM schemes
which incorporate an LPN for the regularizer (\Cref{alg:admm}), following the
recipe we have described in \Cref{sec:pnp-proofs}.
These guarantees are analogous to those we have proved in
\Cref{thm:pnp-pgd-lpn-stationary-appendix} for the PnP-PGD scheme
\Cref{alg:pgd} with LPNs.
For simplicity, we will assume in this section (in contrast to the more general setting of
\Cref{thm:pnp-pgd-lpn-stationary-appendix}, and in agreement with the result
stated in \Cref{thm:pnp-admm-lpn-stationary}) that the measurement operator $A$ in
the underlying inverse problem \Cref{eq:main_inverse_problem} is linear and acts
in the standard basis, and accordingly we identify it with its matrix representation $\A$.
This means (among other things) that the data fidelity term $\x \mapsto
\tfrac{1}{2} \| \y - \A \x \|_2^2$ is convex.

Before proceeding to the proof, we note that \Cref{alg:admm} adopts an update ordering which is nonstandard in the signal processing literature (c.f.\ \cite{venkatakrishnan2013plug,Kamilov2023-uq}) for technical reasons. We employ this update order due to its prevalence in the optimization literature, notably in the analysis of \citet{Themelis2020-jj}, and we emphasize that all of our experiments are done following \Cref{alg:admm}. Although both the typical PnP-ADMM update order and the update order in \Cref{alg:admm} correspond to an ADMM algorithm for the \textit{same objective function}, the analysis of these two iterative optimization procedures is different, and seems to require different technical assumptions (c.f.\ \cite{Themelis2020-jj,Yan2016-fy}). Prior work on convergent PnP seems to have also run into this barrier, suggesting it is not an artifact of our analysis: for example, \citet{hurault2022proximal} study a PnP variant of Douglas-Rachford splitting rather than ADMM, which is roughly analogous to the reversed-order of updates in \Cref{alg:admm} by a reduction of \citet{Themelis2020-jj}, and \citet{Sun2021-ll} prove convergence of a sequence of associated residuals in the standard-order PnP-ADMM rather than of the sequence of iterates itself.

Our proof will be based on the work of \citet{Themelis2020-jj}, which
provides guarantees for ADMM in the nonconvex setting. We restate some of their
results for convenience after stating our convergence result, then proceed to
the proof.

\begin{theorem}[Convergence guarantee for running PnP-ADMM with LPNs]\label{thm:pnp-admm-lpn-stationary-appendix}
    Consider the sequence of iterates $(\x_{k}, \uu_k, \z_k)$, $k \in \{0, 1, \dots\}$,
    defined by \Cref{alg:admm} run with a linear measurement operator $\A$ and a
    LPN $f_{\theta}$ with softplus activations, trained with $0 < \alpha < 1$.
    Assume further that the penalty parameter $\rho$ satisfies 
    $\rho > \| \A\adj \A \|$.
    Then the sequence of iterates $(\x_{k}, \uu_k, \z_k)$ converges to a limit
    point $(\x^*, \uu^*, \z^*)$ which satisfies the KKT conditions (of the augmented problem):
    \begin{equation}\label{eq:pnp-admm-lpn-stationary-guarantee}
        \begin{split}
            \x^* &= \z^*, \\
            \uu^* &= -\frac{1}{\rho} \A\adj(\A \x^* - \y), \\
            \uu^* &= \nabla R_{\theta}(\z^*),
        \end{split}
    \end{equation}
    where $R_{\theta}$ is the regularization function associated to the LPN
    $f_{\theta}$ (i.e., $f_\theta = \prox_{R_\theta}$), which is continuously
    differentiable.
    In particular, the primal limit $\x^*$ is a critical point of the regularized
    reconstruction cost $\x \mapsto \tfrac{1}{2} \| \y - \A \x \|_2^2
    + \rho R_{\theta}(\x)$, and the full limit iterate $(\x^*, \uu^*, \z^*)$ is
    a fixed point of the PnP-ADMM iteration (\Cref{alg:admm}).
\end{theorem}

We restate convergence results of \citet{Themelis2020-jj} in lesser generality,
given the additional regularity properties present in our setting of interest.

\begin{problem}\label{prob:admm-themelis}
    Let $h_1 : \R^n \to \R$ and $h_2 : \R^n \to \R$ be continuously
    differentiable. We consider the minimization problem
    \begin{equation}\label{eq:admm-themelis-objective}
        \min_{\x \in \R^n}\, h_1(\x) + h_2(\x).
    \end{equation}
\end{problem}

\begin{algorithm2}[{\citet[Eqns.\ (1.2), (ADMM), (1.3)]{Themelis2020-jj}}]\label{alg:themelis-admm}
    Perform variable splitting in \Cref{eq:admm-themelis-objective} to obtain an
    equivalent problem
    \begin{equation}\label{eq:admm-themelis-objective-split}
        \min_{\x \in \R^n, \z \in \R^n}\, h_1(\x) + h_2(\z)
        \quad \mathrm{s.t.}\quad
        \x - \z = \mathbf{0}.
    \end{equation}
    Fix $\rho > 0$. Form the augmented Lagrangian for
    \Cref{eq:admm-themelis-objective-split} at level $\rho$, that is,
    the function
    \begin{equation}\label{eq:admm-augmented-lagrangian}
        \mathcal{L}_{\rho}(\x, \z, \y) =
        h_1(\x) + h_2(\z) + \langle \y, \x - \z \rangle
        + \frac{\rho}{2} \| \x - \z \|_2^2,
    \end{equation}
    and consider the following iteration,\footnote{This corresponds to
    setting the ``relaxation parameter'' $\lambda$ in \citet[Eqn.\ (ADMM)]{Themelis2020-jj} to
    $1$.}:
    \begin{equation}\label{eq:admm-themelis}
        \begin{split}
            \x^+ &\in \argmin\, \mathcal{L}_{\rho}(\,\cdot\,, \z, \y),  \\
            \y^+ &= \y + \rho ( \x^+ - \z ), \\
            \z^+ &\in \argmin\, \mathcal{L}_{\rho}(\x^+, \,\cdot\,,
            \y^{+}).
        \end{split}
    \end{equation}
    This iteration induces a set-valued map
    $\mathcal{T}_{\rho} : \R^n \times \R^n
    \rightrightarrows
    \R^n \times \R^n \times \R^n$.
    Given an initialization $(\y_0, \z_0)$, a sequence of
    ADMM iterates $(\x_k, \y_k, \z_k)$ is defined inductively by $(\x_{k},
    \y_{k}, \z_{k}) \in \mathcal{T}_{\rho}(\y_{k-1}, \z_{k-1})$, for $k \in
    \mathbb{N}$.\footnote{The variable $\y$ defined here follows the notation of
    \citet{Themelis2020-jj}, and in particular should not be confused with the
    measurements in the inverse problems framework. We hope the reader will
    forgive this conflict of notation.}

\end{algorithm2}

\citet{Themelis2020-jj} provide the following convergence guarantee for
\Cref{alg:themelis-admm} relative to the objective
\Cref{eq:admm-themelis-objective}, under weak assumptions on $h_1$ and $h_2$:

\begin{theorem}[{\citet[Theorem 5.6, Theorem 4.1, Theorem
    5.8]{Themelis2020-jj}}\label{thm:themelis-convergence}\footnote{In obtaining the result
    stated here from \citet[Theorem 5.6]{Themelis2020-jj}, we simplify the
    ``image function'' expressions \cite[Definition 5.1]{Themelis2020-jj} using
    the simple constraint structure of the ADMM problem
    \Cref{eq:admm-themelis-objective-split}: in particular, in checking
    \cite[Assumption II]{Themelis2020-jj}, we have that $\varphi_1(\s) = \inf_{\x
    \in \R^n} \{ h_1(\x) \mid \x = \s \} = h_1(\s)$ and similarly
    $\varphi_2(\s) = \inf_{\z \in \R^n} \{ h_2(\z) \mid -\z = \s \} = h_2(-\s)$.
    In particular $\varphi_1 = h_1$ and $\varphi_2 = h_2 \circ -\Id$, so that
    \cite[Assumption II.A4]{Themelis2020-jj} is implied by \cite[Assumption
    II.A1]{Themelis2020-jj}. This also allows us to translate the convergence
    guarantees of \citet[Theorems 5.6, 5.8]{Themelis2020-jj} from applying to
    the sequence of iterate images under the constraint maps to the sequence of
    iterates themselves.}]%
    Suppose the objective $h_1 + h_2$ is coercive;\footnote{This implies that
    the objective function of the equivalent penalized version of
    \Cref{eq:admm-themelis-objective-split} is level bounded and admits
    a solution
    (recall that the latter is a consequence of the Weierstrass theorem, e.g.\ \cite[Proposition
    A.8(2)]{Bertsekas2016-db}).}
    that $h_1$ is continuously differentiable, its gradient $\nabla h_1$ is
    $L$-Lipschitz, and there exists $\sigma \in \R$ such that $h_1
    + \tfrac{\sigma}{2} \| \,\cdot\, \|_2^2$ is convex;\footnote{Such a $\sigma$
    always exists, and satisfies $| \sigma | \leq L$. Roughly speaking, the
    smaller a value of $\sigma$ can be chosen, the better---this is possible
    when $h_1$ is `more convex'.} and $h_2$ is proper and lower semicontinuous.
    Moreover, suppose the penalty parameter $\rho$ is chosen so that %
    $ \rho > \max \{ 2 \max \{ -\sigma, 0 \}, L \}$, and that the augmented
    Lagrangian \Cref{eq:admm-augmented-lagrangian} is a KL function (see
    \Cref{sec:kl-omin}, \Cref{def:KL-property}).\footnote{Although \citet{Themelis2020-jj} state their global
    convergence result, Theorem 5.8, only in the semialgebraic setting,
    inspection of their arguments (notably \cite[p.\ 163 top]{Themelis2020-jj},
    and the connecting discussion in \cite[Theorem 2, Remark 2(ii)]{Li2016-ty},
    together with the equivalence between DRS and ADMM in \cite[Theorem
    5.5]{Themelis2020-jj}) reveals that it is only necessary that the augmented
    Lagrangian $\mathcal{L}_{\rho}$ is a KL function.}
    Then the sequence $(\x_k, \y_k, \z_k)_{k \in \{ 0, 1, \dots \}}$
    converges to a limit point $(\x^*, \y^*, \z^*)$ which satisfies
    the KKT conditions
    \begin{align*}
        -\y^* &= \nabla h_1(\x^*) \\
        \y^* &\in \partial h_2(\z^*) \\
        \x^* - \z^* &= \mathbf{0}.
    \end{align*}
    Here, $\partial h_2$ is the (limiting) subdifferential mapping of
    $h_2$.\footnote{This is the same notion of subdifferential introduced in
    \Cref{sec:proof-pnp-pgd-lpn-stationary} in order to state the results of
    \citet{Bot2016-mw}. We use the fact that the limiting subdifferential
    coincides (up to converting a single-valued set-valued map into a function)
    with the gradient for a $C^1$ function \cite[Theorem 9.18, Corollary
    9.19]{Tyrrell_Rockafellar1998-ol}.}
    More concisely, the limit satisfies 
    $\mathbf{0} \in \nabla h_1(\x^*) + \partial h_2(\x^*)$,
    and in particular $\x^*$ is a critical point for the
    objective \Cref{eq:admm-themelis-objective}.

\end{theorem}

It is standard to argue that \Cref{alg:admm} corresponds to
\Cref{alg:themelis-admm} up to a simple reparameterization (i.e., relabeling of
variables), given that the LPN $f_{\theta}$ in \Cref{alg:admm} is a proximal
operator.

\begin{lemma}\label{lem:admm-scaled-dual-version}
    An ADMM sequence $(\x_k, \y_k, \z_k)$ generated by the update
    \Cref{eq:admm-themelis} is linearly isomorphic to a sequence generated by
    the update rule
    \begin{equation}\label{eq:themelis-scaled}
        \begin{split}
            \x^+ &\in \prox_{\tfrac{1}{\rho} h_1 }\left(
            \z - \uu
            \right) \\
            \uu^+ &= \uu + ( \x^+ - \z ) \\
            \z^+ &\in \prox_{\tfrac{1}{\rho} h_2 } \left(
            \x^+ + \uu^+
            \right),
        \end{split}
    \end{equation}
    in the sense that if $(\x_k', \uu_k', \z_k')$ is the corresponding sequence
    of iterates generated by this update rule with initialization
    $\z_0 = \z_0'$ and $\uu_0 = \tfrac{1}{\rho} \y_0$, then
    we have $\x_k = \x'_k$, $\z_k = \z'_k$, and $\uu_k = \tfrac{1}{\rho} \y_k$ for
    every $k \in \mathbb{N}_0$.
\end{lemma}

\begin{proof}
    Notice that we can write in \Cref{eq:admm-augmented-lagrangian}
    by completing the square
    \begin{equation}\label{eq:admm-augmented-lagrangian-scaled}
        \mathcal{L}_{\rho}(\x, \z, \y) =
        h_1(\x) + h_2(\z)
        + \frac{\rho}{2} \left\| \frac{1}{\rho} \y + (\x - \z) \right\|_2^2
        - \frac{\rho}{2} \left\| \frac{1}{\rho} \y \right\|_2^2,
    \end{equation}
    in order to simplify the minimization operations in \Cref{eq:admm-themelis}.
    Indeed, the iteration \Cref{eq:admm-themelis} then becomes equivalent to,
    with \Cref{eq:admm-augmented-lagrangian-scaled}, the iteration
    \begin{equation}
        \begin{split}
            \x^+ &\in \prox_{\tfrac{1}{\rho}h_1}\left(
            \z - \tfrac{1}{\rho} \y
            \right) \\
            \y^+ &= \y + \rho ( \x^+ - \z ) \\
            \z^+ &\in \prox_{\tfrac{1}{\rho} h_2} \left(
            \x^+ + \tfrac{1}{\rho} \y^+
            \right).
        \end{split}
    \end{equation}
    Introducing now the ``scaled dual variable'' $\uu = \tfrac{1}{\rho} \y$,
    we have the equivalent update 
    \begin{equation}\label{eq:admm-scaled-dual-update}
        \begin{split}
            \x^+ &\in \prox_{\tfrac{1}{\rho} h_1 }\left(
            \z - \uu
            \right) \\
            \uu^+ &= \uu + ( \x^+ - \z ) \\
            \z^+ &\in \prox_{\tfrac{1}{\rho} h_2 } \left(
            \x^+ + \uu^+
            \right).
        \end{split}
    \end{equation}
    The equivalence claims in the statement of the lemma follow from this chain
    of reasoning.
\end{proof}

With this preparation completed, we are now ready to give the proof of
\Cref{thm:pnp-admm-lpn-stationary-appendix}.

\begin{proof}[Proof of {\Cref{thm:pnp-admm-lpn-stationary-appendix}}]
    Below, to avoid a notational conflict with the dual variables in
    \Cref{alg:themelis-admm}, we will write $\y_{\mathrm{meas}}$ for the
    measurements that define the data fidelity term in the inverse problem cost.

    We have by assumption and \Cref{prop:lpn} and \Cref{lem:prior-coercive} that there exists
    a coercive $C^1$ function $R_{\theta}: \R^n \to \R$ such that $f_{\theta}
    = \prox_{R_{\theta}}$.
    Given the initialization $\uu_0 = \mathbf{0}$ in \Cref{alg:admm}, applying
    \Cref{lem:admm-scaled-dual-version} implies that \Cref{alg:admm} corresponds
    to a sequence of ADMM iterates $(\x_k, \y_k, \z_k)$ generated via
    \Cref{alg:themelis-admm} with the initialization $\y_0 = \mathbf{0}$, the
    objectives $h_1(\x) = \tfrac{1}{2} \| \y_{\mathrm{meas}} - \A \x \|_2^2$
    and $h_2(\z) = \rho R_{\theta}(\z)$, and the correspondence $\y_k
    = \rho \uu_k$.
    In addition, we observe that $h_1$ is smooth, nonnegative, convex, and
    satisfies $\| \nabla^2 h_1 \| = \| \A\adj \A \|$, where $\| \,\cdot\, \|$
    denotes the operator norm. In turn, we know that $\z \mapsto h_2(\z) + \tfrac{\rho}{2} \| \z
    \|_2^2$ is (strongly) convex: \Cref{lem:lip-nabla-psi} implies that $f_{\theta}$
    is Lipschitz, and \cite[Proposition 2(1)]{gribonval2020characterization} implies that an $L$-Lipschitz proximal operator's associated prox-primitive is $(1-1/L)$-weakly convex, from which it follows that
    the sum
    $\z \mapsto h_2(\z) + \tfrac{\rho}{2} \| \z
    \|_2^2$ is strongly convex.
    As a result of these facts, every minimization operation in
    \Cref{alg:themelis-admm} has a unique minimizer, and we can interchange
    set inclusion operations with equalities when describing the ADMM sequence
    corresponding to \Cref{alg:admm} without any concern in the sequel.

    Now, these instantiations of $h_1$ and $h_2$ verify the elementary
    hypotheses of \Cref{thm:themelis-convergence}:
    \begin{enumerate}
        \item Nonnegativity of $h_1$ and coercivity of $h_2$ imply that $h_1
            + h_2$ is coercive;
        \item $h_1$ is smooth and we can take $L = \| \A\adj \A \|$ and $\sigma = 0$,
            and therefore the hypothesis that $\rho > \| \A\adj \A \|$
            verifies the conditions on the penalty parameter;
        \item By \Cref{lem:prior-coercive}, $h_2$ is real-valued and $C^1$, as
            above.
    \end{enumerate}
    Finally, to check the KL property of the augmented Lagrangian
    $\mathcal{L}_{\rho}$ defined by \Cref{eq:admm-augmented-lagrangian},
    we will follow \Cref{sec:kl-omin} and verify that $\mathcal{L}_{\rho}$ is
    definable in an o-minimal structure, then apply \Cref{prop:o-minimal},
    Property 1, since $\mathcal{L}_{\rho}$ is $C^1$ as a sum of $C^1$ functions
    (both $h_1$ and $h_2$ are so).
    To this end, note by \Cref{corollary:sum-definable} that
    $\mathcal{L}_{\rho}$ is definable in the o-minimal structure asserted by
    Property 2 if both $h_1$ and $h_2$ are definable in that o-minimal
    structure.
    \Cref{prop:o-minimal}, Property 2 implies that $h_1$ is definable, since it
    is a degree two polynomial. Then \Cref{lem:phi-KL} implies that $h_2$ is
    definable in the same o-minimal structure, and it thus follows from the
    preceding reasoning that
    $\mathcal{L}_{\rho}$ has the KL property.

    We can therefore apply \Cref{thm:themelis-convergence} to obtain that the
    sequence of iterates of \Cref{alg:admm} converges, and its limit point
    $(\x^*, \uu^*, \z^*)$ satisfies the KKT conditions \cref{eq:pnp-admm-lpn-stationary-guarantee}
    \begin{align*}
        \uu^* &= -\tfrac{1}{\rho}\A\adj(\A\x^* - \y_{\mathrm{meas}}), \\
        \uu^* &= \nabla R_{\theta}(\z^*), \\
        \x^* - \z^* &= \mathbf{0},
    \end{align*}
    where we have used fact that the
    limiting subdifferential coincides (up to converting a single-valued
    set-valued map into a function) with the gradient for a $C^1$ function
    \cite[Theorem 9.18, Corollary 9.19]{Tyrrell_Rockafellar1998-ol}.
    In particular, simplifying gives
    \begin{equation}
        \A\adj(\A\x^* - \y_{\mathrm{meas}})
        =
        - \rho \nabla R_{\theta}(\x^*),
    \end{equation}
    which is equivalent to the claimed critical point property. To see that this
    is also equivalent to being a fixed point of \Cref{alg:admm},
    it is convenient to use the expression \Cref{eq:admm-scaled-dual-update} for
    the ADMM update expression, which appeared in the proof of
    \Cref{lem:admm-scaled-dual-version}.
    The KKT conditions \cref{eq:pnp-admm-lpn-stationary-guarantee} imply that $\nabla h_1(\x^*) = -\nabla h_2(\x^*)$,
    and since both $\tfrac{1}{\rho} h_1$ and $\tfrac{1}{\rho} h_2$ are differentiable and yield a strongly
    convex function when summed with a quadratic $\tfrac{1}{2} \| \,\cdot\,
    \|_2^2$,
    we can express the action of their proximal operators as
    \begin{equation}
        \prox_{\tfrac{1}{\rho} h_i} = (\Id + \tfrac{1}{\rho}\nabla h_i)^{-1}, \quad i = 1, 2.
    \end{equation}
    From \Cref{eq:admm-scaled-dual-update}, 
    we get that $\x^+ = (\Id + \tfrac{1}{\rho}\nabla h_1)^{-1}(\z^* - \uu^*)$.
    The KKT conditions \cref{eq:pnp-admm-lpn-stationary-guarantee} imply that $\uu^* = -\tfrac{1}{\rho} \nabla h_1(\x^*)$
    and $\z^* = \x^*$, so that $\z^* - \uu^* = (\Id + \tfrac{1}{\rho} \nabla
    h_1)(\x^*)$, and therefore indeed $\x^+ = \x^*$.
    Proceeding, we then have that $\uu^+ = \uu^* - \z^* + \x^*$, which gives
    that $\uu^+ = \uu^*$, since $\x^* = \z^*$.
    Finally, we obtain from the previous two steps and the definition of $h_2$ that
    \begin{equation}
        \z^+ = (\Id + \nabla R_{\theta})^{-1}(\z^* + \uu^*),
    \end{equation}
    and the final KKT condition \cref{eq:pnp-admm-lpn-stationary-guarantee} implies that $\uu^* = \nabla R_{\theta}(\z^*)$.
    Hence
    \begin{equation}
        \z^+ = (\Id + \nabla R_{\theta})^{-1}(\Id + \nabla R_{\theta})(\z^*)
        = \z^*,
    \end{equation}
    and we indeed conclude that $(\x^*, \uu^*, \z^*)$ is a fixed point of
    \Cref{alg:admm}.
\end{proof}

\subsubsection{Regularity of the Regularization Function of LPNs}\label{sec:lpn-regularity}

As discussed in the ``recipe'' of \Cref{sec:pnp-proofs}, in this section we
prove basic regularity properties of the regularization function associated to
any LPN with the architecture of \Cref{prop:lpn}.

\begin{lemma}[Regularity Properties of LPNs]\label{lem:prior-coercive}
    Suppose $f_{\theta}$ is an LPN constructed following the recipe in
    \Cref{prop:lpn}, with softplus activations $\sigma(x) = (1/\beta)\log(1 +
    \exp(\beta x))$, where $\beta>0$ is an arbitrary constant, and with strong
    convexity weight $0 < \alpha < 1$. 
    Let $f_{\theta}(\y) = \nabla \psi_{\theta}(\y) + \alpha \y$ be the
            defining equation of the LPN.
    Then there is a real-valued function
    $R_{\theta} : \R^n \to \R$ such that $f_{\theta} =
    \prox_{R_{\theta}}$.
    Moreover, we have the following regularity properties:
    \begin{enumerate}
        \item $R_{\theta}$ is coercive, i.e., we have $R_{\theta}(\x) \to +\infty$
    as $\| \x \|_2 \to +\infty$.
        \item $f_{\theta} : \R^n \to \R^n$ is surjective and invertible, with an inverse
            mapping $f_{\theta}^{-1} : \R^n \to \R^n$ which is continuous.
        \item $R_{\theta}$ is continuously differentiable. In particular, it holds
            \begin{equation}
                \begin{split}
                    R_{\theta}(\x)
                    &=
                    (1-\alpha)\langle f^{-1}_{\theta}(\x), \nabla
                    \psi_{\theta}(f^{-1}_{\theta}(\x)) \rangle
                    \\
                    &\quad+ \frac{\alpha(1-\alpha)}{2}\| f^{-1}_{\theta}(\x) \|_2^2
                    - \tfrac{1}{2} \| \nabla \psi_{\theta}(f^{-1}_{\theta}(\x))\|_2^2
                    - \psi_{\theta}(f^{-1}_{\theta}(\x)).
                    \label{eq:prior-defining-equation-with-alpha}
                \end{split}
            \end{equation}
    \end{enumerate}
\end{lemma}

\begin{remark}
    \Cref{lem:prior-coercive} does not, strictly speaking, require the softplus
    activation: the proof shows that any Lipschitz activation function with
    enough differentiability and slow growth at infinity, such as another
    smoothed verison of the ReLU activation, the GeLU, or the Swish activation,
    would also work.  
\end{remark}

\begin{proof}[Proof of \Cref{lem:prior-coercive}.]
    The main technical challenge will be to establish coercivity of $R_{\theta}$, which always exists as necessary, by \Cref{prop:characterization-continuous-prox,prop:lpn}. We will therefore pursue this estimate as the main line of the proof, establishing the remaining assertions in the result statement along the way.
    
    By \Cref{prop:lpn}, there exists $R_{\theta}$ such that $f_{\theta} =
    \prox_{R_{\theta}}$. Now, using \cite[Theorem
    4(a)]{gribonval2020characterization}, %
    for every $\y \in \R^n$,
    \begin{equation*}
        R_{\theta}(f_{\theta}(\y))
        =
        \langle \y, f_{\theta}(\y) \rangle
        - \tfrac{1}{2} \| f_{\theta}(\y) \|_2^2
        - \left(
            \psi_{\theta}(\y) + \tfrac{\alpha}{2} \| \y \|_2^2
        \right).%
    \end{equation*} 
    Using the definition of $f_{\theta}$ and minor algebra, we rewrite this as
    \begin{align*}
        R_{\theta}(f_{\theta}(\y))
        &=
        \langle \y, \nabla \psi_{\theta}(\y) + \alpha \y \rangle
        - \tfrac{1}{2} \| \nabla \psi_{\theta}(\y) + \alpha \y \|_2^2
        - \left(
            \psi_{\theta}(\y) + \tfrac{\alpha}{2} \| \y \|_2^2
        \right)
        \\
        &=
        (1-\alpha)\langle \y, \nabla \psi_{\theta}(\y) \rangle
        + \frac{\alpha(1-\alpha)}{2}\| \y \|_2^2
        - \tfrac{1}{2} \| \nabla \psi_{\theta}(\y)\|_2^2
        - \psi_{\theta}(\y).
        \labelthis \label{eq:phi-coercivity-equality}
    \end{align*}
    At this point, we observe that %
    by \Cref{lem:f-inv-C0}, the map $f_{\theta} : \R^n \to \R^n$ is invertible and surjective, with a continuous inverse mapping.
    This establishes the second assertion that we have claimed. In addition, 
    taking inverses in \Cref{eq:phi-coercivity-equality} implies
    \Cref{eq:prior-defining-equation-with-alpha} and as a consequence the fact
    that $R_{\theta}$ is real-valued, and the fact that it is continuously differentiable on $\R^n$ is then an immediate consequence of \cite[Corollary 6(b)]{gribonval2020characterization}.
    To conclude, it only remains to show that $R_{\theta}$ is coercive, which we will accomplish by lower bounding the RHS of \Cref{eq:phi-coercivity-equality}. By \Cref{lem:lip-psi}, $\psi_\theta$ is $L$-Lipschitz for a constant $L>0$. Thus, we have for every $\y$ (by the triangle
    inequality)
    \begin{equation*}
        | \psi_{\theta}(\y) | \leq L \| \y \|_2 + K
    \end{equation*}
    for a (finite) constant $K\in \R$, depending only on $\theta$.
    Now, the Cauchy-Schwarz inequality implies from the previous
    two statements (and $\|\nabla \psi_{\theta}\|_2 \leq L$ by the Lipschitz property of $\psi_{\theta}$)
    \begin{align*}
        R_{\theta}(f_{\theta}(\y))
        & \geq
        -(1-\alpha)\|\y\|_2 \|\nabla \psi_{\theta}(\y) \|_2
        + \frac{\alpha(1-\alpha)}{2}\| \y \|_2^2
        - \tfrac{1}{2}\| \nabla \psi_{\theta}(\y)\|_2^2
        - L \| \y \|_2
        - K,
        \\
        & \geq
        -L(1-\alpha)\|\y\|_2 %
        + \frac{\alpha(1-\alpha)}{2}\| \y \|_2^2
        - \frac{L^2}{2} %
        - L \| \y \|_2
        - K.
    \end{align*}
    We rewrite this
    estimate with some algebra as
    \begin{equation*}
        R_{\theta}(f_{\theta}(\y))
        \geq
        \| \y \|_2 \left(
            \frac{\alpha(1-\alpha)}{2} \| \y \|_2
            - L(1-\alpha)
            - L
        \right)
        - \frac{L^2}{2}
        - K.
    \end{equation*}
    Next, we notice that when $0 < \alpha < 1$, the coefficient $\alpha(1-\alpha) > 0$; hence there is a constant $M > 0$ depending only on $\alpha$ and $L$ such that for every $\y$ with $\| \y \|_2 \geq M$, one has
    \begin{equation*}
            \frac{\alpha(1-\alpha)}{2} \| \y \|_2
            - L(1-\alpha)
            - L
            \geq \frac{\alpha(1-\alpha)}{4} \| \y \|_2.
    \end{equation*}
    In turn, iterating this exact argument implies that there is another constant $M' > 0$ (depending only on $\alpha$, $L$, and $K$) such that whenever $\| \y \|_2 \geq M'$, one has
    \begin{equation*}
        R_{\theta}(f_{\theta}(\y))
            \geq \frac{\alpha(1-\alpha)}{8} \| \y \|_2^2.
    \end{equation*}
    We can therefore rewrite the previous 
    inequality as
    \begin{equation}\label{eq:phi_coercivity_almost}
        R_{\theta}(\x)
            \geq \frac{\alpha(1-\alpha)}{8} \| f^{-1}_{\theta}(\x) \|_2^2,
    \end{equation}
    for every $\x$ such that $\| f^{-1}(\x) \|_2 \geq M'$. To conclude, we will
    show that whenever $\| \x \|_2 \to +\infty$, we also have
    $\|f_{\theta}^{-1}(\x) \|_2 \to +\infty$, which together with
    \Cref{eq:phi_coercivity_almost} will imply coercivity of $R_{\theta}$.
    To this end, write $\| \,\cdot\, \|_{\mathrm{Lip}}$ for the Lipschitz
    seminorm:
    \begin{equation*}
        \| f \|_{\Lip} = \sup_{\y \neq \y'} \frac{\| f(\y) - f(\y') \|_2}{\| \y -
        \y' \|_2},
    \end{equation*}
    and note that $\| f_{\theta} \|_{\Lip} \leq \| \nabla \psi_{\theta}
    \|_{\Lip} + \alpha$. 
    By \Cref{lem:lip-nabla-psi}, $\nabla \psi_\theta$ is $L_{\nabla \psi_\theta}$-Lipschitz continuous, thus $f_\theta$ is $(L_{\nabla \psi_\theta} + \alpha)$-Lipschitz continuous,
    \begin{equation*}
        \| f_{\theta}(\y) - f_{\theta}(\y') \|_2
        \leq
        (L_{\nabla \psi_\theta} + \alpha) \| \y - \y' \|_2.
    \end{equation*}
    Thus, taking inverses, we have
    \begin{equation*}
        \|f_{\theta}^{-1}(\x) - f_{\theta}^{-1}(\mathbf{0})\|_2
        \geq
        \frac{1}{L_{\nabla \psi_{\theta}} + \alpha} \| \x \|_2,
    \end{equation*}
    and it then follows from the triangle inequality that whenever $\x$ is such that $\| \x \|_2 \geq 2(L_{\nabla \psi_{\theta}} + \alpha) \| f_{\theta}^{-1}(\mathbf{0}) \|_2 $, we have in fact
    \begin{equation*}
        \|f_{\theta}^{-1}(\x)\|_2
        \geq
        \frac{1}{2(L_{\nabla \psi_{\theta}} + \alpha)} \| \x \|_2.
    \end{equation*}
    Combining this estimate with \Cref{eq:phi_coercivity_almost},
    we obtain that for every $\x$ such that $\| \x \|_2 \geq 2(L_{\nabla \psi_{\theta}} + \alpha) \| f_{\theta}^{-1}(\mathbf{0}) \|_2 $
    and $\| \x \|_2 \geq 2M' (L_{\nabla \psi_{\theta}} + \alpha)$, it holds
    \begin{equation*}
        R_{\theta}(\x)
            \geq \frac{\alpha(1-\alpha)}{32(L_{\nabla \psi_{\theta}} + \alpha)^2} \| \x \|_2^2.
    \end{equation*}
    Taking limits in this last bound yields coercivity of $R_{\theta}$, and hence the claim.
\end{proof}

\begin{lemma}[Invertibility of $f_\theta$ and Continuity of $f_\theta^{-1}$]\label{lem:f-inv-C0}
    Suppose $f_{\theta}$ is an LPN constructed following the recipe in
    \Cref{prop:lpn}, with softplus activations $\sigma(x) = (1/\beta)\log(1 +
    \exp(\beta x))$, where $\beta>0$ is an arbitrary constant, and with strong
    convexity weight $0 < \alpha < 1$. 
    Then $f_\theta : \R^n \to \R^n$ is invertible and surjective, 
and $f_\theta^{-1} : \R^n \to \R^n$ is $C^0$.
\end{lemma}
\begin{proof}
    The proof uses the invertibility construction that we describe
    informally in \Cref{sec:method-prior}.
    By construction, we have $f_{\theta} = \nabla \psi_{\theta} + \alpha \Id$,
    where $\Id$ denotes the identity operator on $\R^n$ (i.e., $\Id(\x) = \x$
    for every $\x\in \R^n$).

    For a fixed $\x \in \R^n$, consider the strongly convex minimization problem $\min_{\y} \psi_\theta(\y) + \tfrac{\alpha}{2}\|\y\|_2^2 - \langle \x , \y \rangle$. By first-order optimality condition, the minimizers are exactly $\{\y \mid \nabla \psi_\theta(\y) + \alpha \y = \x\}$. Furthermore, since the problem is strongly convex, it has a unique minimizer for each $\x \in \R^n$ \cite{boyd2004convex}. Therefore, for each $\x \in \R^n$, there exists a unique $\y$ such that $\x = \nabla \psi_\theta(\y) + \alpha \y = f_{\theta}(\y)$.

    The argument above establishes that $f_{\theta} : \R^n \to \R^n$ is injective and surjective; hence there exists an inverse $f_{\theta}^{-1} : \R^n \to \R^n$. 
    To conclude the proof, we will argue that $f_{\theta}^{-1}$ is continuous. 
    To this end, we use the characterization of continuity which states that a function $g : \R^n \to \R^n$ is continuous if and only if for every open set $U \subset \R^n$, we have that $g^{-1}(U)$ is open, where $g^{-1}(U) = \{ \x \in \R^n \mid g(\x) \in U \}$ (e.g., \cite[Theorem 4.8]{Rudin1976-ij}).
    To show that $f_{\theta}^{-1}$ is continuous, it is therefore equivalent to show that for every open set $U \subset \R^n$, one has that $f_{\theta}(U)$ is open.
    But this follows from invariance of domain, a standard result in algebraic topology (e.g., 
    \cite[Proposition 7.4]{Dold2012-nj}), since $f_{\theta}$ is injective and continuous.
    We have thus shown that $f_{\theta}$ is invertible, and that its inverse is continuous, as claimed.
\end{proof}

\begin{lemma}[Lipschitzness of $\psi_\theta$]\label{lem:lip-psi}
    Suppose $f_{\theta}$ is an LPN constructed following the recipe in
    \Cref{prop:lpn}, with softplus activations $\sigma(x) = (1/\beta)\log(1 +
    \exp(\beta x))$, where $\beta>0$ is an arbitrary constant, and let
    $\psi_{\theta}$ denote the convex potential function for the LPN.
    Then $\psi_{\theta}$ is
    $L_{\psi_\theta}$-Lipschitz continuous for a constant $L_{\psi_\theta}>0$, i.e., $|\psi_\theta(\y) - \psi_\theta(\y')| \leq L_{\psi_\theta} \|\y - \y'\|_2$, for all $\y, \y' \in \R^n$.
\end{lemma}
\begin{proof}
    Note that
    the derivative $\sigma'$ of the softplus activation satisfies $\sigma'(x) =
    1 / (1 + \exp(-\beta x))$, which is no larger than $1$, since $\exp(x) > 0$
    for $x \in \R$. 
    If $F$ is a map between Euclidean spaces we will write $DF$ for its
    differential (a map from the domain of $F$ to the space of linear operators
    from the domain of $F$ to the range of $F$). 
    Hence the activation function $g$ in \Cref{prop:lpn} is
    $1$-Lipschitz with respect to the $\ell_2$ norm, since the induced (by
    elementwise application) map $g : \R^n \to \R^n$ defined by $g(\y) =
    [\sigma(x_1), \dots, \sigma(x_n)]^T$ satisfies
    \begin{equation*}
        Dg(\y) =
        \begin{bmatrix}
            \sigma'(x_1) && \\
            &\ddots& \\
            && \sigma'(x_n)
        \end{bmatrix},
    \end{equation*}
    which is bounded in operator norm by $\sup_{x} | \sigma'(x) | \leq 1$.
    First, notice that
    \begin{align*}
        \| \psi_{\theta}(\y) - \psi_{\theta}(\y') \|_2
        &=
        \| \w^T(\z_K(\y) - \z_K(\y') ) \|_2
        \\
        &\leq
        \| \w \|_2
        \| \z_K(\y) - \z_K(\y') \|_2
    \end{align*}
    by Cauchy-Schwarz.
    Meanwhile, we have similarly
    \begin{equation*}
        \| \z_1(\y) - \z_1(\y') \|_2
        \leq
        \| \HHH_1 \|
        \| \y - \y' \|_2,
    \end{equation*}
    where $\| \,\cdot\, \|$ denotes the operator norm of a matrix,
    and for integer $0 < k < K+1$
    \begin{equation*}
        \| \z_k(\y) - \z_k(\y') \|_2
        \leq
        \| \W_{k} \|
        \| \z_{k-1}(\y) - \z_{k-1}(\y') \|_2
        +
        \| \HHH_{k} \|
        \| \y - \y' \|_2.
    \end{equation*}
    By a straightforward induction, it follows that $\psi_{\theta}$ is
    $L$-Lipschitz for a constant $L>0$ (depending only on $\theta$).
\end{proof}

\begin{lemma}[Lipschitzness of $\nabla \psi_\theta$ and LPNs $f_\theta$]\label{lem:lip-nabla-psi}
    Suppose $f_{\theta}$ is an LPN constructed following the recipe in
    \Cref{prop:lpn}, with softplus activations $\sigma(x) = (1/\beta)\log(1 +
    \exp(\beta x))$, where $\beta>0$ is an arbitrary constant, and with strong
    convexity weight $0 < \alpha < 1$. 
    Let $f_{\theta}(\y) = \nabla \psi_{\theta}(\y) + \alpha \y$ be the
    defining equation of the LPN. Then
    $\nabla \psi_\theta$ is $L_{\nabla \psi_\theta}$-Lipschitz continuous, for a constant $L_{\nabla \psi_\theta}>0$.
    In particular, $f_{\theta}$ is $(\alpha + L_{\nabla
    \psi_{\theta}})$-Lipschitz continuous.
\end{lemma}
\begin{proof}
    The claimed expression for Lipschitzness of $f_{\theta}$ follows from the
    claimed expression for Lipschitzness of $\nabla \psi_{\theta}$, by
    differentiating and using the triangle inequality.
    We recall basic notions of Lipschitz continuity of differentiable mappings
    at the beginning of the proof of \Cref{lem:lip-psi}, which we will make use
    of below.
    We will upper bound $\| \nabla \psi_{\theta} \|_{\Lip}$
    by deriving an explicit expression for the gradient. By the defining
    formulas in \Cref{prop:lpn}, we have
    \begin{equation*}
        \psi_{\theta}(\y) = \w^T \z_K(\y) + \bb.
    \end{equation*}
    The chain rule gives
    \begin{equation*}
        \nabla \psi_{\theta}(\y)
        = D\z_K(\y)\adj \w,
    \end{equation*}
    where $\adj$ denotes the adjoint of a linear operator, so for any $\y, \y'$
    we have
    \begin{align*}
        \| \nabla \psi_{\theta}(\y) - \nabla \psi_{\theta}(\y') \|_2
        &=
        \left\| \left(
            D\z_K(\y) - D\z_K(\y')
        \right)\adj \w
        \right\|_2
        \\
        &\leq
        \left\| \left(
            D\z_K(\y) - D\z_K(\y')
        \right)\adj
        \right\|
        \| \w \|_2
        \\
        &=
        \left\|
            D\z_K(\y) - D\z_K(\y')
        \right\|
        \| \w \|_2
        \\
        &\leq
        \left\|
            D\z_K(\y) - D\z_K(\y')
            \right\|_{\mathrm{F}}
        \| \w \|_2,
    \end{align*}
    where the first inequality uses Cauchy-Schwarz, the third line uses that the
    operator norm of a linear operator is equal to that of its adjoint,
    and the third line uses that the operator norm is upper-bounded by the
    Frobenius norm. This shows that we obtain a Lipschitz property in $\ell_2$
    for $\nabla \psi_{\theta}$ by obtaining one for the differential $D\z_{K}$
    of the LPN's last-layer features.
    To this end, we can use the chain rule to compute for any integer $1 < k <
    K+1$ and any $\boldsymbol \delta \in \R^n$
    \begin{equation*}
        D\z_k(\y)(\boldsymbol \delta)
        =
        g'\left(
            \W_k \z_{k-1}(\y) + \HHH_k \y + \bb_k
        \right)
        \mathbin{\odot}
        \left[
            \W_k D\z_{k-1}(\y)( \boldsymbol \delta ) + \HHH_k \boldsymbol \delta
        \right],
    \end{equation*}
    where $g'$ is the derivative of the softplus activation function $g$,
    applied elementwise, and $\odot$ denotes elementwise multiplication,
    and similarly
    \begin{equation*}
        D\z_1(\y)(\boldsymbol \delta)
        =
        g'\left(
            \HHH_1 \y + \bb_1
        \right)
        \mathbin{\odot}
        \left[
            \HHH_1 \boldsymbol \delta
        \right].
    \end{equation*}
    Now notice that for any vectors $\vv$ and $\y$ and any matrix $\A$ such that
    the sizes are compatible, we have $\vv \mathbin{\odot} (\A \y) = \diag(\vv)
    \A \y$. Hence we can rewrite the above recursion in matrix form as
    \begin{equation*}
        D\z_k(\y)
        =
        \underbrace{
            \diag \left(
                g'\left(
                    \W_k \z_{k-1}(\y) + \HHH_k \y + \bb_k
                \right)
            \right)
        }_{\D_k(\y)}
        \left[
            \W_k D\z_{k-1}(\y) + \HHH_k
        \right],
    \end{equation*}
    and similarly
    \begin{equation*}
        D\z_1(\y)
        =
        \underbrace{
            \diag \left(
                g'\left(
                    \HHH_1 \y + \bb_1
                \right)
            \right)
        }_{\D_1(\y)}
        \HHH_1.
    \end{equation*}
    We will proceed with an inductive argument. First, by the submultiplicative
    property of the Frobenius norm and the triangle inequality for the Frobenius
    norm, note that we have if $1 < k < K+1$
    \begin{align*}
        \|
        D\z_k(\y) - D\z_k(\y')
        \|_{\frob}
        &\leq
        \|
        \D_k(\y) - \D_k(\y')
        \|_{\frob}
        \\
        &\quad+
        \|
        \D_k(\y)\W_k D\z_{k-1}(\y) - \D_k(\y')\W_k D\z_{k-1}(\y')
        \|_{\frob}
        \\
        &\leq
        \|
        \D_k(\y) - \D_k(\y')
        \|_{\frob}
        \\
        &\quad+
        \|
        \D_k(\y)\W_k D\z_{k-1}(\y) - \D_k(\y)\W_k D\z_{k-1}(\y')
        \|_{\frob}
        \\
        &\quad+
        \|
        \D_k(\y)\W_k D\z_{k-1}(\y') - \D_k(\y')\W_k D\z_{k-1}(\y')
        \|_{\frob}
        \\
        &\leq
        \|
        \D_k(\y) - \D_k(\y')
        \|_{\frob}
        \\
        &\quad+
        \|
        \D_k(\y)\W_k
        \|_{\frob}
        \| 
        D\z_{k-1}(\y)  - D\z_{k-1}(\y')
        \|_{\frob}
        \\
        &\quad+
        \|
        D\z_{k-1}(\y')
        \|_{\frob}
        \| 
        \D_k(\y)\W_k - \D_k(\y') \W_k
        \|_{\frob}
        \\
        &\leq
        \left(
            1 + \| \W_k \|_{\frob}
        \right)
        \|
        \D_k(\y) - \D_k(\y')
        \|_{\frob}
        \\
        &\quad+
        \|
        \D_k(\y)
        \|_{\frob}
        \| \W_k \|_{\frob}
        \| 
        D\z_{k-1}(\y)  - D\z_{k-1}(\y')
        \|_{\frob}.
    \end{align*}
    Now, as we have shown above, $g'(x) = (1 + \exp(-\beta x))^{-1} \leq 1$
    for every $x \in \R$. 
    This implies
    \begin{equation*}
        \|
        \D_k(\y)
        \|_{\frob}
        \leq
        \sqrt{n_k},
    \end{equation*}
    where $n_k$ is the output dimension of $k$-th layer.
    Moreover, we calculate with the chain rule
    \begin{equation*}
        g''(x) = \frac{\beta e^{-\beta x}}{(1 + e^{-\beta x})^2},
    \end{equation*}
    and by L'H\^{o}pital's rule, we have that $\lim_{x \to +\infty} \tfrac{x}{(1
    + x)^2} = 0$, so that by continuity, $g''$ is bounded for $x \in \R$.
    It follows that $g'$ is Lipschitz.  Notice now that
    \begin{align*}
        \|
        \D_k(\y) - \D_k(\y')
        \|_{\frob}
        &=
        \left\|
        g'\left(
            \W_k \z_{k-1}(\y) + \HHH_k \y + \bb_k
        \right)
        -
        g'\left(
            \W_k \z_{k-1}(\y') + \HHH_k \y' + \bb_k
        \right)
        \right\|_{2}
        \\
        &\leq
        \| g' \|_{\Lip}
        \left(
            \| \W_k \|_{\frob} \| \z_{k-1}(\y) - \z_{k-1}(\y') \|_2
            +
            \| \HHH_k \|_{\frob} \| \y - \y' \|_2,
        \right)
    \end{align*}
    where in the second line we used the fact that the derivative of an
    elementwise function is a diagonal matrix together with the triangle
    inequality and Cauchy-Schwarz. However, we have already argued previously by
    induction that $\psi_{\theta}$ is Lipschitz, and in particular each of its
    feature maps $\z_{k}$ is Lipschitz. We conclude that $\D_{k}$ is Lipschitz,
    and the Lipschitz constant depends only on $\theta$.
    This means that there are constants $L_k, L_k'$ depending only on $n$ and
    $\theta$ such that
    \begin{align*}
        \|
        D\z_k(\y) - D\z_k(\y')
        \|_{\frob}
        &\leq
        L_k
        \|
        \y - \y'
        \|_{2}
        +
        L_k'
        \| 
        D\z_{k-1}(\y)  - D\z_{k-1}(\y')
        \|_{\frob}.
    \end{align*}
    Meanwhile, following the same arguments as above, but in a slightly
    simplified setting, we obtain
    \begin{align*}
        \|
        D\z_1(\y) - D\z_1(\y')
        \|_{\frob}
        &=
        \| \D_1(\y)\HHH_1 - \D_1(\y') \HHH_1 \|_{\frob}
        \\
        &\leq
        \| \HHH_1 \|_{\frob}
        \| \D_1(\y) - \D_1(\y') \|_{\frob}
        \\
        &\leq
        \| g' \|_{\Lip}
        \| \HHH_1 \|_{\frob}^2
        \| \y - \y' \|_2,
    \end{align*}
    which demonstrates that $D\z_1$ is also Lipschitz, with the Lipschitz
    constant depending only on $\theta$. By induction, we therefore conclude
    that there is $L_{\nabla \psi_\theta} > 0$ such that
    \begin{equation*}
        \| \nabla \psi_{\theta}(\y) - \nabla \psi_{\theta}(\y') \|_2
        \leq
        L_{\nabla \psi_\theta} \| \y - \y' \|_2,
    \end{equation*}
    with $L_{\nabla \psi_\theta}$ depending only on $\theta$ and $n_k$.
\end{proof}

\subsubsection{\rebuttal{Nondegeneracy of the Prior}}\label{sec:kl-omin}

For our convergence results for PnP with LPNs, we make use of general
convergence analyses for nonsmooth nonconvex optimization from the literature.
To invoke these results, we need to show that the regularizer $R_{\theta}$
associated to the LPN $f_{\theta}$ satisfies a certain nondegeneracy property
called the Kurdyka-\L ojaciewicz (KL) inequality. We provide a
self-contained proof of this fact in this section. We will instantiate various
concepts from the literature in our specific setting (i.e., sacrificing
generality for clarity), and include appropriate references to the literature
for more technical aspects that are not essential to our setting.

\begin{definition}[KL property; {\cite[Definition
    3.1]{Attouch2010-vs}, \cite[Definition 1]{Bot2016-mw}}]\label{def:KL-property}
    A $C^1$ function $f : \R^n \to \R$ is said to have the KL property (or to be
    a KL function) at
    a point $\bar{x} \in \R^n$ if there exists $0 < \eta \leq +\infty$, a
    neighborhood $U \subset \R^n$ of $\bar{x}$, and a continuous concave
    function $\varphi : [0, \eta) \to \R$ such that
    \begin{enumerate}
        \item $\varphi(0) = 0$;
        \item $\varphi \geq 0$;
        \item $\varphi$ is $C^1$ on $(0, \eta)$;
        \item for all $s \in (0, \eta)$, $\varphi'(s) > 0$;
        \item for all $x \in U \cap \{ x \mid f(\bar{x}) < f(x) < f(\bar{x}) +
            \eta \}$, the Kurdyka-\L ojaciewicz inequality holds:
            \begin{equation}\label{eq:KL-inequality}
                \| \nabla f(x) \|_2 \geq 
                \frac{1}{\varphi'(f(x) - f(\bar{x}))}.
            \end{equation}
    \end{enumerate}
\end{definition}

It can be shown that any $f \in C^1(\R^n)$ has the KL property at every point
$\bar{x}$ that is not a critical point of $f$ \cite[Remark
3.2(b)]{Attouch2010-vs}.
At critical points $\bar{x}$ of such a function $f$, the KL inequality is a kind
of nondegeneracy condition on $f$. %
This can be intuited from the (common) example of functions $f$ witnessing the KL
property via the function $\varphi(s) = c s^{1/2}$, where $c>0$ is a constant: 
in this case, \Cref{eq:KL-inequality} gives a lower bound on the squared magnitude of
the gradient near to a critical point in terms of the values of the function
$f$. Every Morse function satisfies such an instantiation of the KL property: see \cite[beginning of \S 4]{Attouch2010-vs}.
In general, the KL property need not require excessive regularity of $f$; we only state it
as such because of the special structure in our setting (c.f.\
\Cref{lem:prior-coercive}).

The KL property is useful in spite of the unwieldy \Cref{def:KL-property}
because broad classes of functions can be shown to have the KL property, and
these classes of functions satisfy a convenient calculus that allows one to
construct new KL functions from simpler primitives (somewhat analogous to the
situation with convex functions in convex analysis).
The class of such functions we will use in this work are \textit{functions
definable in an o-minimal structure}. We refer to \cite[Definition
4.1]{Attouch2010-vs} for the precise definition of this concept from model
theory---for our purposes, it will suffice to understand that an o-minimal
structure is a collection of subsets of $\R^n$, for each
$n \in \mathbb{N}$, which satisfy certain algebraic properties and are called \textit{definable sets}, or \textit{definable} for short, and that a definable function is
one whose graph is contained in this
collection (i.e., its graph is a definable set)\footnote{Here and below, the use of the
article ``an'' is intentional---there may be multiple o-minimal structures in
which a function is definable, and in some applications it is important to
distinguish between them. For our purposes in this work, this distinction is
immaterial: we will be content to simply work in a `maximal' o-minimal structure
that contains all functions we will need.}---and instead state several useful
properties of functions definable in an o-minimal structure that will suffice to
complete our proof. We note that \citet[\S 4]{Attouch2010-vs} give an excellent
readable optimization-motivated overview of these properties in greater
generality and depth, as do \citet{Ji2020-mg}, the latter moreover in a deep
learning context.

\begin{proposition}\label{prop:o-minimal}
    The following properties of o-minimal structures and functions definable in
    an o-minimal structure hold.
    \begin{enumerate}
        \item \cite[Theorem 4.1]{Attouch2010-vs} If $f : \R^n \to \R$ is
            continuously differentiable and
            definable in an o-minimal structure,
            then it has the KL property (\Cref{def:KL-property}) at every $x \in
            \R^n$. 
        \item (Wilkie) There is an o-minimal structure in which the following
            functions are definable: the exponential function $x \mapsto
            \exp(x)$ (for $x \in \mathbb{R}$), and
            polynomials of arbitrary degree in $n$ real variables $x_1, \dots, x_n$
            \cite{Van_den_Dries1994-yz}, \cite[Corollary
            2.11]{Van_den_Dries1998-ru}.
        \item \cite[\S B, B.7(3)]{Van_den_Dries1996-oz} If $f : \R^n \to \R$ is
            $C^1$ and definable in an o-minimal structure, then its gradient $\x
            \mapsto \nabla f(\x)$ is definable.
        \item (\cite[\S B, B.4]{Van_den_Dries1996-oz}) A function $f : \R^n \to
            \R^m$ with $f = (f_1, \dots, f_m)$ is
            definable in an o-minimal structure if and only if each $f_i$ is
            definable (in the same structure).
        \item \cite[Theorem 2.3(iii)]{Loi2010-dx} The composition of functions
            definable in an o-minimal structure is definable.
        \item \cite[Theorem 2.3(ii)]{Loi2010-dx} If $f : \R^n \to \R^n$ is
            definable in an o-minimal structure, then its image $\mathrm{im}(f)$
            is definable; if moreover $f$ is invertible on its image,
            then its inverse $f^{-1}$ (defined on $\mathrm{im}(f)$) is
            definable.\footnote{The proof of this property follows readily from
            the cited result and the preceding properties by
            noticing that if one defines a map $F : \R^n \times \R^n \to \R^n
            \times \R^n$ by $F(x, y) = (f(y), f(x))$, then
            $\mathrm{gr}(f^{-1}) = F^{-1}(\mathrm{gr}(f))$, which
            expresses the graph of $f^{-1}$ as the inverse image of a definable
            set by a definable function.}
    \end{enumerate}
\end{proposition}

These properties represent a rather minimal set that are sufficient for our purposes. To illustrate how to use them to deduce slightly more accessible (and useful) corollaries, we provide the following result
on definable functions that will be used repeatedly in our proofs.
\begin{corollary}
\label{corollary:sum-definable}
    If functions $f_1 : \R^n \to \R$ and $f_2 : \R^n \to \R$ are definable in the o-minimal structure asserted by \Cref{prop:o-minimal}, Property 2, then $f_1 + f_2$ is definable in the same structure.
\end{corollary}
\begin{proof}
    By \Cref{prop:o-minimal}, Property 2, the function $g = x_1 + x_2$, for $x_1, x_2 \in \R$, is definable since it is a polynomial. By \Cref{prop:o-minimal}, Property 4, $(f_1, f_2)$ is definable. So, by \Cref{prop:o-minimal}, Property 5, $f_1 + f_2$ is definable since it is the composition of $(f_1, f_2)$ and $g$.
\end{proof}

Using these properties, we prove that the regularizer associated to an LPN with $0 <
\alpha < 1$ is a KL function. 
For concision, we call a function $f$ ``definable'' if it is definable in
the o-minimal structure whose existence is asserted by \Cref{prop:o-minimal},
Property 2.

\begin{lemma}[Definability and KL property of the prior $R_\theta$]\label{lem:phi-KL}
    Suppose $f_{\theta}$ is an LPN constructed following the recipe in
    \Cref{prop:lpn}, with softplus activations $\sigma(x) = (1/\beta)\log(1 +
    \exp(\beta x))$, where $\beta>0$ is an arbitrary constant, and with strong
    convexity weight $0 < \alpha < 1$. 
    Then there is a function
    $R_{\theta} : \R^n \to \R$ such that $f_{\theta} =
    \prox_{R_{\theta}}$,
    and $R_{\theta}$ is definable and has the KL property (\Cref{def:KL-property}) at every
    point of $\R^n$.
    Moreover, this $R_{\theta}$ can be chosen to simultaneously satisfy the
    conclusions of \Cref{lem:prior-coercive}.
\end{lemma}

\begin{proof}
    We appeal to \Cref{lem:prior-coercive}, given that $0 < \alpha < 1$, to
    obtain that a function 
    $R_{\theta} : \R^n \to \R$ such that $f_{\theta} = \prox_{R_{\theta}}$
    exists,
    and moreover that this function $R_{\theta}$ is $C^1$ and satisfies the 
    relation \Cref{eq:prior-defining-equation-with-alpha}.
    The idea of the proof is to use \Cref{eq:prior-defining-equation-with-alpha}
    to show that $R_{\theta}$ is definable in an o-minimal structure, then
    appeal to \Cref{prop:o-minimal}, Property 1 to obtain that $R_{\theta}$
    has the KL property at every point of $\R^n$.
    By \Cref{corollary:sum-definable}, to show that $R_{\theta}$
    is definable, it suffices to show that four functions appearing in the
    representation \Cref{eq:prior-defining-equation-with-alpha}
    are definable: $\x \mapsto \langle f_{\theta}^{-1}(\x), \nabla
    \psi_{\theta}(f^{-1}_{\theta}(\x)) \rangle $,
    $\x \mapsto \| f^{-1}_{\theta}(\x) \|_2^2$, 
    $\x \mapsto \| \nabla \psi_{\theta}( f^{-1}_{\theta}(\x) ) \|_2^2$,
    and $\x \mapsto \psi_{\theta}( f_{\theta}^{-1}(\x) )$.
    We can reduce further. By \Cref{prop:o-minimal}, Property 2, the squared norm $\x
    \mapsto \| \x \|_2^2$ is definable, and \Cref{prop:o-minimal}, Property 5, namely
    that finite compositions of definable functions are definable, it suffices
    to show that the arguments of the norms appearing amongst these four
    functions are definable. Similarly, by \Cref{prop:o-minimal}, Properties 2,
    4, and 5, for the inner product term appearing amongst these
    four functions, it suffices to show that the two individual arguments of the
    inner product are definable.
    It then follows by one additional application of \Cref{prop:o-minimal},
    Property 5 that
    we need only argue that the following three functions are definable:
    $\x \mapsto \psi_{\theta}(\x)$, $\x \mapsto f_{\theta}^{-1}(\x)$, and $\x
    \mapsto \nabla \psi_{\theta}(\x)$.
    Finally, by \Cref{prop:o-minimal}, Property 6 and \Cref{lem:prior-coercive}, which
    asserts that $f_{\theta}$ is invertible on $\R^n$ with range equal to $\R^n$,
    it suffices merely to show that $\x \mapsto f_{\theta}(\x)$ is definable to
    obtain that its inverse is definable.

    To complete the proof, we argue in sequence below that each of these three
    functions are definable.

    \paragraph{ICNN $\psi_{\theta}$.} The ICNN $\psi_{\theta}$ is defined
    inductively in \Cref{prop:lpn}, and by our hypotheses, the activation
    function $g(\x)$ is an elementwise application of the softplus activation
    with parameter $\beta > 0$: with a minor abuse of notation, this function is
    \begin{equation}
        g(x) = (1/\beta)\log(1 + \exp(\beta x)).
    \end{equation}
    This (scalar) activation function is definable. To see this, by
    \Cref{prop:o-minimal}, Property 2, we have that $x \mapsto \exp x$ is
    definable, and by \Cref{prop:o-minimal}, Property 6, we have that $x \mapsto
    \log x$ is definable. It then follows from
    \Cref{prop:o-minimal}, Properties 2 \& 5 that $g$ is definable, as a
    finite composition of definable functions.
    We then conclude from \Cref{prop:o-minimal}, Property 4 that the
    elementwise activation $\x \mapsto g(\x)$ is definable.
    Now, by the definition of $\psi_{\theta}$ in \Cref{prop:lpn}, we have that
    $\z_1$ is definable as a composition of definable functions
    (\Cref{prop:o-minimal}, Properties 2 and 5), and arguing inductively, we
    have by the same reasoning that $\z_k$ is definable for each $k = 2, \dots,
    K$. Because $\psi_{\theta} = \mathbf{w}^T \z_K + b$ is an affine function of
    $\z_K$, one additional application of \Cref{prop:o-minimal}, Properties 2
    and 5 establishes that $\psi_{\theta}$ is definable.

    \paragraph{ICNN gradient $\nabla \psi_{\theta}$.} 
    We conclude this immediately from \Cref{prop:o-minimal}, Property 3 and the
    definability of $\psi_\theta$, since $\psi_\theta$ is $C^2$ (because $\psi_\theta$ is composed of affine maps and $C^2$ activations).

    \paragraph{LPN $f_{\theta}$.} Recall that $f_{\theta}(\x) = \nabla
    \psi_{\theta}(\x) + \alpha \x$.
    Because we have shown above that $\nabla \psi_{\theta}$ is definable, using
    \Cref{prop:o-minimal}, Properties 2 and 5 once more, we conclude that
    $f_{\theta}$ is definable.

    \paragraph{Concluding.} By our preceding reductions, we have shown that
    $R_{\theta}$ is definable. We conclude from \Cref{prop:o-minimal}, Property 1 that
    $R_{\theta}$ has the KL property at every point of $\R^n$.

\end{proof}

\begin{remark}
    Note that a byproduct of the proof of \Cref{lem:phi-KL} is that all
    constituent functions of the LPN are also definable in an o-minimal
    structure. This fact may be of interest for future work extending
    \Cref{lem:phi-KL} to priors associated with novel LPN architectures.
\end{remark}

\section{Algorithms}

\subsection{Algorithm for Log-Prior Evaluation}
\label{sec:alg-prior}
\begin{algorithm}[H]
\caption{Log-prior evaluation for LPN}
\label{alg:prior}

\begin{algorithmic}[1]

\Require Learned proximal network $f_\theta(\cdot)$, $\psi_\theta(\cdot)$ that satisfies $f_\theta = \nabla \psi_\theta$, query point $\x$

\State Find $\y$ such that $f_\theta(\y) = \x$, by solving $\min_\y \psi_{\theta}(\y; \alpha) - \langle \x, \y \rangle$ or $\min_\y \|f_\theta(\y) - \x\|_2^2$

\State $R \gets  \langle \y, \x \rangle - \frac{1}{2} \|\x\|^2
- \psi_\theta(\y)$

\Ensure $R$ \Comment{The learned log-prior (i.e., regularizer function) at $\x$}
\end{algorithmic}
\end{algorithm}

\subsection{Algorithm for LPN Training}
\label{sec:alg-train}
\begin{algorithm}[H]
\caption{Training the LPN with proximal matching loss}
\label{alg:train}

\begin{algorithmic}[1]

\Require Training dataset $\mathcal{D}$, initial LPN parameter $\theta$, loss schedule $\gamma(\cdot)$, noise standard deviation $\sigma$, number of iterations $K$, network optimizer $\operatorname{Optm}(\cdot, \cdot)$

\State $k \gets 0$

\Repeat

\State Sample $\x \sim \mathcal{D}$, $\beps \sim \N(0, \mathbf{I})$

\State $\y \gets \x + \sigma \beps$

\State $\mathcal{L}_{PM} \gets m_{\gamma(k)}(\|f_\theta(\y) - \x\|_2)$

\State $\theta \gets \operatorname{Optm}(\theta, \nabla_\theta \mathcal{L}_{PM})$ \Comment{Update network parameters}

\State $k \gets k+1$

\Until $k = K$

\Ensure $\theta$ \Comment{Trained LPN}
\end{algorithmic}
\end{algorithm}

\subsection{Algorithm for Using LPN with PnP-PGD to Solve Inverse Problems}

\begin{algorithm}[H]
\caption{Solving inverse problems with LPN and PnP-PGD}
\label{alg:pgd}

\begin{algorithmic}[1]

\Require Trained LPN $f_\theta$,
measurement operator $A$, measurement $\y$,
data fidelity function $h(\x) = \tfrac{1}{2} \| \y - A(\x) \|_2^2$, %
initial estimation $\x_0$, step size $\eta$, number of iterations $K$

\For{$k = 0$ \textbf{to} $K-1$}
    \State $\x_{k+1} \gets f_\theta \left(\x_k - \eta \nabla h(\x_k)\right)$ 
\EndFor

\Ensure $\x_K$ 
\end{algorithmic}
\end{algorithm}

\section{Experimental Details}
\label{sec:details}
\subsection{Details of Laplacian experiment}
\label{sec:details-lap}
The LPN architecture contains four linear layers and $50$ hidden neurons at each layer, with $\beta=10$ in softplus activation. The LPN is trained by Gaussian noise with $\sigma=1$, Adam optimizer \cite{kingma2014adam} and batch size of $2000$. For either $\ell_2$ or $\ell_1$ loss, the model is trained for a total of $20k$ iterations, including $10k$ iterations with learning rate $lr=1e-3$, and another $10k$ with $lr=1e-4$. For the proximal matching loss, we initialize the model from the $\ell_1$ checkpoint and train according to the schedule in \Cref{tab:lap-schedule}. \rebuttal{To enforce nonnegative weights of LPN, weight clipping is applied during training, projecting the negative weights to zero at each iteration.}

\begin{table}[H]
    \centering
    \begin{tabular}{ccc}
    \toprule
        Number of \\ iterations & $\gamma$ in $\mathcal{L_{PM}}$ & Learning rate \\ \midrule
        $2k$ & $0.5$ & $1e-3$ \\
        $2k$ & $0.5$ & $1e-4$ \\
        $4k$ & $0.4$ & $1e-4$ \\
        $4k$ & $0.3$ & $1e-4$ \\
        $4k$ & $0.2$ & $1e-5$ \\
        $4k$ & $0.1$ & $1e-5$ \\
        $4k$ & $0.1$ & $1e-6$ \\ \bottomrule
    \end{tabular}
    \caption{The schedule for proximal matching training of LPN in the Laplacian experiment.}
    \label{tab:lap-schedule}
\end{table}

\subsection{Details of MNIST experiment}
\label{sec:details-mnist}
The LPN architecture is implemented with four convolution layers and $64$ hidden neurons at each layer, with $\alpha=0.01$ and softplus $\beta=10$. The model is trained on the MNIST training set containing $50k$ images, with Gaussian noise with standard deviation $\sigma=0.1$ and batch size of $200$. The LPN is first trained by $\ell_1$ loss for $20k$ iterations; and then by the proximal matching loss for $20k$ iterations, with $\gamma$ initialized at $0.64 * 28 = 17.92$ and halved every $5k$ iterations. The learned prior is evaluated on $100$ MNIST test images. Conjugate gradient is used to solve the convex inversion problem: $\min_\y \psi_{\theta}(\y) - \langle \x, \y \rangle$ in prior evaluation.

\subsection{Details of CelebA experiment}
\label{sec:details-celeba}
We center-crop CelebA images from $178\times 218$ to $128 \times 128$, and normalized the intensities to $[0,1]$. Since CelebA images are larger and more complex than MNIST, we use a deeper and wider network. The LPN architecture includes $7$ convolution layers and $256$ hidden neurons per layer, with $\alpha=1e-6$ and $\beta=100$.
For LPN training, we train two separate models with two levels of training noise: $\sigma=0.05$ and $0.1$. When applied for deblurring, the best model is selected for each blurring degree ($\sigma_{blur}$) and measurement noise level ($\sigma_{noise}$). We pretrain the network with $\ell_1$ loss for $20k$ iterations with $lr=1e-3$. Then, we train the LPN with proximal matching loss $\mathcal{L_{PM}}$ for $20k$ iterations using $lr=1e-4$, with the schedule of $\gamma$ similar to MNIST: initialized at $0.64 \times \sqrt{128 \times 128 \times 3} \approx 142$, and multiplied by $0.5$ every $5k$ iterations. A batch size of $64$ is used during training. \rebuttal{We observed that initializing the respective weights to be nonnegative, by initializing them according to a Gaussian distribution and then taking the exponential, helped the training converge faster. Therefore, we applied such initialization in the experiments on CelebA and Mayo-CT. The same weight clipping as in \Cref{sec:details-lap} is applied to ensure the weights stay nonnegative throughout training. The training time of LPN on the CelebA dataset is about 6 hours on a NVIDIA RTX A5000 GPU.}

\paragraph{PnP algorithm and comparison methods} We use PnP-ADMM to perform deblurring on CelebA for BM3D, DnCNN, and our LPN (see \Cref{alg:admm}). 
We implement the PnP-ADMM algorithm using the SCICO package \cite{balke2022scientific}. We implement DnCNN \cite{zhang2017beyond} using their public code \footnote{\url{https://github.com/cszn/KAIR}}. \rebuttal{We implement the GS denoiser, Prox-DRUNet \cite{hurault2022proximal}, using their public code\footnote{\url{https://github.com/samuro95/Prox-PnP}} and follow their paper to use the Douglas–Rachford splitting (DRS) algorithm when solving inverse problem, which performs the best with Prox-DRUNet based on their paper.} Both DnCNN and Prox-DRUNet are trained on CelebA.

\subsection{Details of Mayo-CT experiment}
\label{sec:details-ct}
We use the public dataset from Mayo-Clinic for the low-dose CT grand challenge (Mayo-CT) \cite{mccollough2016tu}, which contains abdominal CT scans from 10 patients and a total of 2378 images of size $512 \times 512$. Following \cite{lunz2018adversarial}, we use $128$ images for testing and leave the rest for training. 
The LPN architecture contains $7$ convolution layers with $256$ hidden neurons per layer, with $\alpha=1e-6$ and $\beta=100$. During training, we randomly crop the images to patches of size $128 \times 128$. At test time, LPN is applied to the whole image by sliding windows of the patch size with stride size of $64$.
The training procedure of LPN is the same as in CelebA, except that $\gamma$ in proximal matching loss is initialized to $0.64 \times \sqrt{128 \times 128} \approx 82$. As in the CelebA experiment, we use LPN with PnP-ADMM for solving inverse problems. 

\paragraph{Sparse-view CT}
Following \textcite{lunz2018adversarial}, we simulate CT sinograms using a parallel-beam geometry with 200 angles and 400 detectors. The angles are uniformly spaced between $-90^\circ$ and $90^\circ$. White Gaussian noise with standard deviation $\sigma=2.0$ is added to the sinogram data to simulate noise in measurement. 
We implement AR in PyTroch based on its public TensorFlow code\footnote{\url{https://github.com/lunz-s/DeepAdverserialRegulariser}.}; for UAR, we use the publicly available code and model weights \footnote{\url{https://github.com/Subhadip-1/unrolling_meets_data_driven_regularization}.}.

\paragraph{Compressed sensing}
For compressed sensing, we implement the random Gaussian sampling matrix following \textcite{jalal2021instance}, and add noise of $\sigma=0.001$ to the measurements.
The Wavelet-based sparse recovery method for compressed sensing minimizes the object $\frac{1}{2}\|\y - A\x\|_2^2 + \lambda \|W\x\|_1$, where $A$ is the sensing matrix and $W$ is a suitable Wavelet transform. We select the ``db4'' Wavelet and $\lambda=0.01$. To solve the minimization problem in Wavelet-based approach, we use proximal gradient descent with a step size of $0.5$, stopping criterion  $\|\x_{k+1} - \x_{k}\|_1 < 1e-4$, and maximum number of iterations $=1000$.

\section{Discussions}
\subsection{Other ways to parameterize gradients of convex functions via neural networks}
\label{sec:other-parameterization-of-lpn}
Input convex gradient networks (ICGN) \cite{richter2021input} provide another way to parameterize gradients of convex functions. The model performs line integral over Positive Semi-Definite (PSD) Hessian matrices, 
where the Hessians are implicitly parameterized by the Gram product of Jacobians of neural networks, hence guaranteed to be PSD. 
However, this approach only permits single-layer networks 
in order to satisfy a crucial PDE condition in its formulation 
\cite{richter2021input}, significantly limiting the representation capacity. Furthermore, the evaluation of the convex function is less straightforward than ICNN, which is an essential step in prior evaluation for LPN (see \Cref{sec:method-prior}). 
We therefore adopt the differentiation-based parameterization in this work and leave the exploration of other possibilities to future research.

\section{Additional Results}
\subsection{Learning soft-thresholding from Laplacian distribution}
\label{sec:experiments-lap-appx}
\begin{figure}[H]
    \centering
    \includegraphics[width=0.95\textwidth]{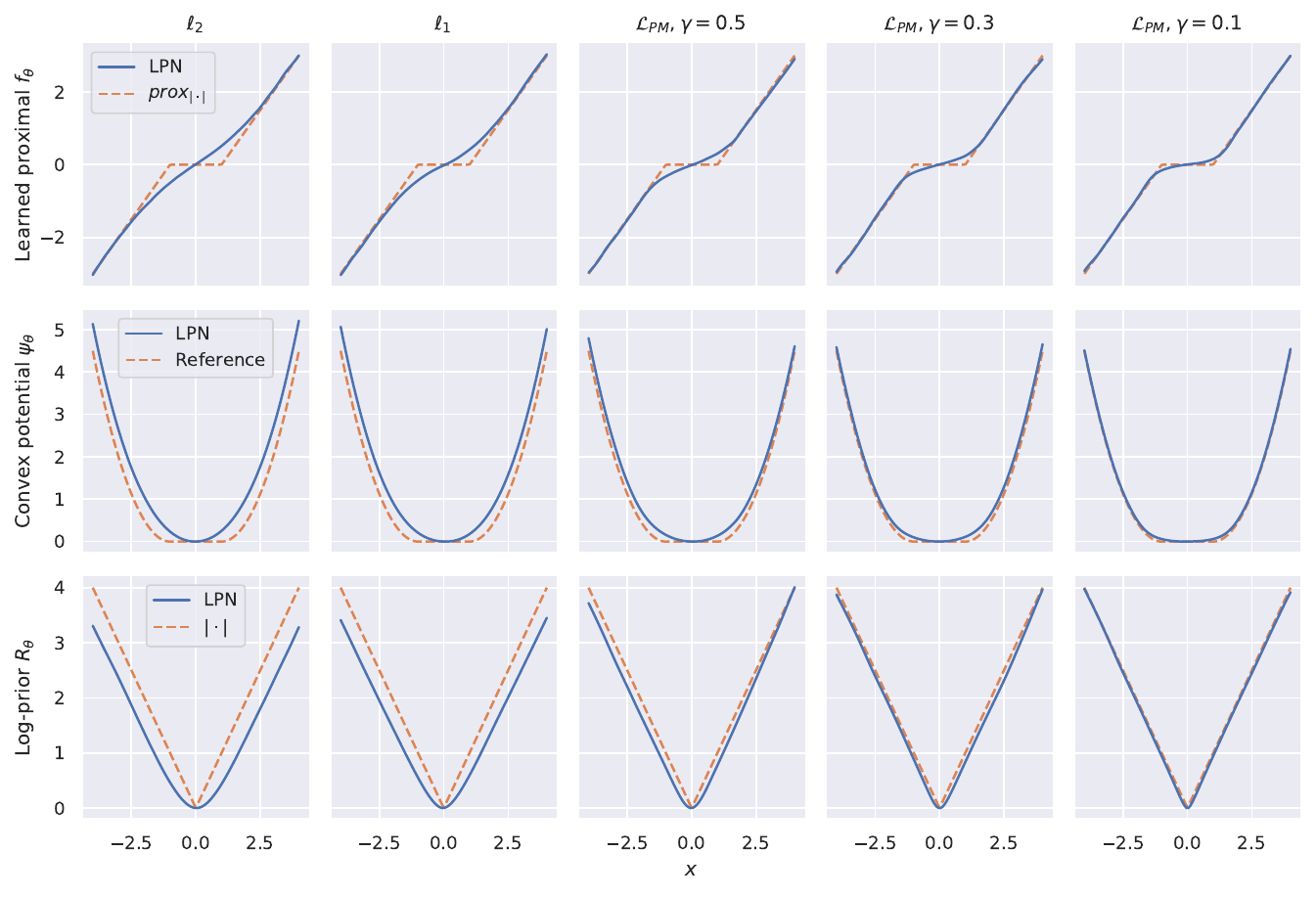}
    
    \caption{The proximal operator $f_\theta$, convex potential $\psi_\theta$, and log-prior $R_\theta$ learned by LPN via different losses: the square $\ell_2$ loss, $\ell_1$ loss, and the proposed proximal matching loss $\mathcal{L}_{PM}$ with different $\gamma \in \{0.5, 0.3, 0.1\}$. The ground-truth data distribution is the Laplacian $p(x) = \tfrac{1}{2}\exp(-|x|)$, with log-prior $-\log p(x) = |x| - \log (\tfrac{1}{2})$. With proximal matching loss, the learned proximal $f_\theta$ and log-prior $R_\theta$ progressively approach their ground-truth, $\prox_{|\cdot|}$ and $|\cdot|$ respectively, as $\gamma$ shrinks from $0.5$ to $0.1$.}
    \label{fig:lap-appx}
\end{figure}

\subsection{Learning a prior for MNIST - image blur}
\label{sec:experiments-mnist-blur}
\begin{figure}[H]
    \centering
    \includegraphics[width=0.3\textwidth]{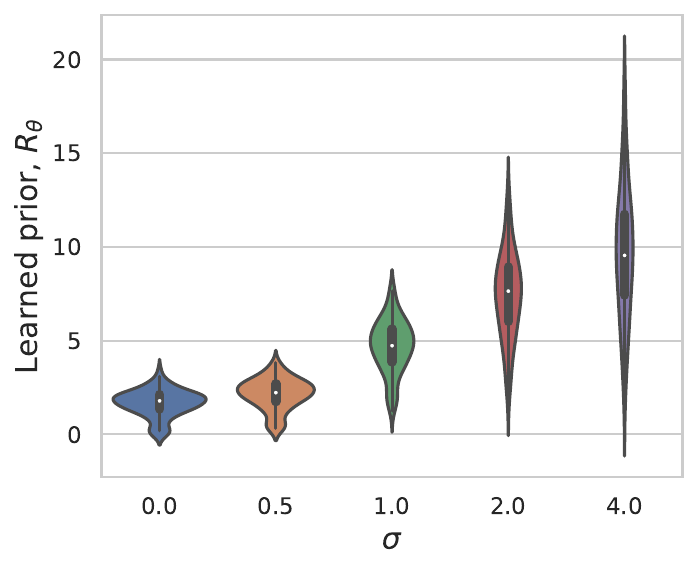}
    \includegraphics[width=0.45\textwidth]{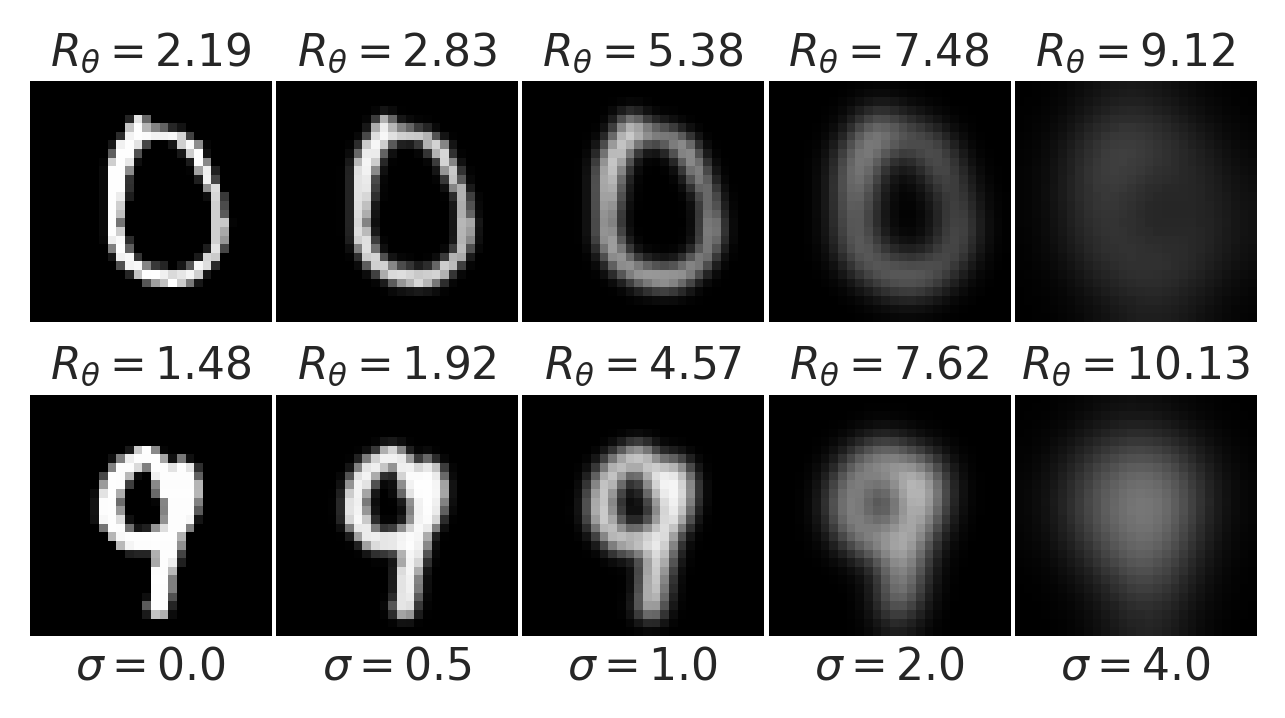}
    
    \caption{The log-prior $R_\theta$ learned by LPN on MNIST, evaluated at images blurred by Gaussian kernel with increasing standard deviation $\sigma$. Left: the prior over 100 test images. Right: the prior at individual examples.}
    \label{fig:mnist-blur}
\end{figure}

Besides perturbing the images by Gaussian noise and convex combination in \Cref{sec:experiments-mnist}, we also evaluate the prior of LPN at blurry images, with results shown in \Cref{fig:mnist-blur}. Again, the prior increases as the image becomes blurrier, coinciding with the distribution of the hand-written digits in MNIST.

\subsection{Solving inverse problems using LPN with PnP-PGD}
\rebuttal{Besides PnP-ADMM, we also test LPN's performance for solving inverse problems using PnP-PGD (proximal gradient descent). \Cref{tab:celeba-deblur-pgd} shows the numerical results for deblurring CelebA images: PGD is slightly less performant than ADMM in terms of PSNR.}
\begin{table}[H]
\rebuttal{
\centering
\caption{\rebuttal{Numerical results for CelebA deblurring using LPN with PnP-PGD and PnP-ADMM, averaged over 20 test images.}}
\label{tab:celeba-deblur-pgd}
\small
\centering
\begin{tabular}{@{}ccccc@{}}
\toprule
\multirow{2}{*}{METHOD} & \multicolumn{2}{c}{$\sigma_{blur}=1, \, \sigma_{noise}=.02$} & \multicolumn{2}{c}{$\sigma_{blur}=1, \, \sigma_{noise}=.04$} \\
 \cmidrule(lr){2-3} \cmidrule(lr){4-5} 
& PSNR($\uparrow$) & SSIM($\uparrow$) & PSNR & SSIM \\ \midrule
LPN with PnP-PGD & 32.7 $\pm$ 2.9 & .92 $\pm$ .03 & 31.2 $\pm$ 2.5 & .89 $\pm$ .04 \\
LPN with PnP-ADMM & 33.0 $\pm$ 2.9 & .92 $\pm$ .03 & 31.3 $\pm$ 2.3 & .89 $\pm$ .03 \\ 
\bottomrule
\end{tabular}%
}
\end{table}

\end{document}